\documentclass{article}

\usepackage[T1]{fontenc}
\usepackage[htt]{hyphenat}

\usepackage{microtype}
\usepackage{graphicx}
\usepackage{subfigure}
\usepackage{booktabs} 
\usepackage{natbib}
\usepackage{algorithm}
\usepackage{algorithmic}
\usepackage{subfigure}
\usepackage[subfigure]{tocloft}
\usepackage{amsmath}
\usepackage{amssymb}
\usepackage{graphicx}
\usepackage{mathtools}
\usepackage{amsfonts}
\usepackage{amsthm,bm}
\usepackage{color}
\usepackage{enumitem}
\usepackage{dsfont}
\usepackage{pgfplots}
\pgfplotsset{width=10cm,compat=1.9}
 \usepackage{pifont}
\usepackage{thmtools} 
\usepackage{thm-restate}
\usepackage{sidecap}

\newcommand{\R}{\mathbb{R}}

\renewcommand{\phi}{\varphi}

\graphicspath {{figures/}}

\usepackage{hyperref}
\usepackage{caption}
\usepackage[super]{nth}
\newtheorem{lemma}{Lemma}

\newtheorem{theorem}{Theorem}

\newtheorem{assumption}{Assumption}
\newtheorem{proposition}{Proposition}

\newtheorem{corollary}{Corollary}

\mathtoolsset{showonlyrefs}

\DeclareMathOperator*{\argmin}{argmin}
\DeclareMathOperator*{\argmax}{argmax}

\let\oldaddcontentsline\addcontentsline
\newcommand{\stoptocentries}{\renewcommand{\addcontentsline}[3]{}}
\newcommand{\starttocentries}{\let\addcontentsline\oldaddcontentsline}

\usepackage{breakcites}
\usepackage{authblk}
\usepackage[left=1.30in,right=1.30in,top=1.00in,bottom=1.00in]{geometry}
\usepackage{enumitem}

\usepackage{hyperref}

\hypersetup{
     colorlinks = true,
     linkcolor = brown,
     anchorcolor = blue,
     citecolor = teal,
     filecolor = blue,
     urlcolor = black
     }

\title{Dual Training of Energy-Based Models \\ with Overparametrized Shallow Neural Networks}

\author[a]{Carles Domingo-Enrich}
\author[b]{Alberto Bietti}
\author[b,c]{Marylou Gabri\'e}
\author[a,b]{Joan Bruna}
\author[a]{Eric Vanden-Eijnden}
\affil[a]{Courant Institute of Mathematical Sciences, New York University}
\affil[b]{Center for Data Science, New York University}
\affil[c]{Center for Computational Mathematics, Flatiron Institute}

\begin{document}

\maketitle


\begin{abstract}%
Energy-based models (EBMs) are generative models that are usually
trained via maximum likelihood estimation. This approach becomes
challenging in generic situations where the trained energy is
non-convex, due to the need to sample the Gibbs distribution
associated with this energy. Using general Fenchel duality results,
we derive variational principles dual to maximum likelihood EBMs
with shallow overparametrized neural network energies, both in the
  \emph{feature-learning} and \emph{lazy} linearised regimes. In the feature-learning regime, this dual formulation justifies using a two time-scale 
  gradient ascent-descent (GDA) training algorithm in which
  one updates concurrently the particles in the sample space and the
  neurons in the parameter space of the energy. We also consider a
  variant of this algorithm in which the particles are sometimes
  restarted at random samples drawn from the data set, and show that
  performing these restarts at every iteration step corresponds to
  score matching training.
These results are  illustrated in simple numerical experiments, which indicates that GDA performs best when features and particles are updated using similar time scales.
\end{abstract}

\stoptocentries

\section{Introduction}
Energy-based models (EBMs) are explicit generative models which consider Gibbs measures defined through an \emph{energy function}~$f$, with a probability density proportional to~$\exp(- \beta f(x))$, where $\beta$ is the inverse temperature.
Such models originate in statistical physics~\citep{gibbs2010elementary,ruelle1969statistical}, and have become a fundamental modeling tool in statistics and machine learning~\citep{wainwright2008graphical,ranzato2007efficient,lecun2006tutorial,du2019implicit,song2021train}.
Given data samples from a target distribution, the learning algorithms for EBMs attempt to estimate an energy function~$f$ to model the samples density. The resulting learned model can then be used to obtain new samples, typically through Markov Chain Monte Carlo (MCMC) techniques. 

The standard method to train EBMs is maximum likelihood estimation, i.e. the learned energy is the one maximizing the likelihood of the target samples, within a certain function class.
One generic approach for this is to use gradient descent, where gradients may be approximated using MCMC samples from the trained model. However, this is computationally difficult for highly non-convex trained energies, due to  `metastability', ie the presence of large basins in the energy landscape that trap trajectories for potentially exponential time.  
This has motivated a myriad of alternative losses to learn EBM energies, such as the popular score matching; see \citep{song2021train} for a review. All in all, such weaker losses result in a loss of statistical power, which motivates exploring computationally efficient methods for EBM maximum-likelihood estimation. 

EBMs also have structural connections with maximum entropy (maxent) models, which have been studied for decades through Fenchel duality. \cite{dai2019exponential} was the first work to leverage similar duality arguments for maximum likelihood EBM training. However, their analysis is restricted to energies lying in RKHS balls (i.e. non-parametric linear models). Despite the appealing optimization properties of RKHS, these spaces of functions typically only contain very smooth functions when the dimension is large~\citep{berlinet2004reproducing}. A recent line of work---originating in supervised learning---has considered an alternative based on shallow neural networks \citep{bach2017breaking}, which  admit a linear representation in terms of a measure over its parameters and are able to adapt to hidden low-dimensional structures in the data. 
The statistical benefits of the obtained $\mathcal{F}_1$ or \emph{Barron} spaces have  recently been studied in the context of shallow EBMs by \cite{domingoenrich2021energybased}, who show that they may outperform the RKHS models. 

\section{Problem setup and main results}
  \label{sec:setup}
Consider a measurable set $\mathcal{X} \subseteq \R^{d_1}$ with a
fixed base probability measure
$\tau_{\mathcal{X}} \in \mathcal{P}(\mathcal{X})$. If $\mathcal{F}$ is
a class of functions (or energies) mapping $\mathcal{X}$ to $\R$, for
any $f \in \mathcal{F}$ we can define the probability measure
$\nu_{\beta f}$ as a Gibbs measure with density:
\begin{equation}
  \label{eq:nu_f_def0}
    \frac{d\nu_{\beta f}}{d\tau_{\mathcal{X}}}(x) := Z_{\beta f}^{-1}e^{-\beta f(x)}
    \text{ with } Z_{\beta f} := \int_{\mathcal{X}}
    e^{-\beta f(y)}
    d\tau_{\mathcal{X}}(y)~,
\end{equation}
where $d\nu_{\beta f}/d\tau_{\mathcal{X}}$ is the Radon-Nikodym
derivative of $\nu_{\beta f}$ and $Z_{\beta f}$ is the partition
function and the parameter $\beta > 0$ is the inverse temperature.
Gibbs measures are the cornerstone of statistical physics since the
seminal works of Boltzmann and Gibbs. Beyond their widespread use
across computational sciences, they have also found their application
in machine learning, by the name of \emph{energy-based models} (EBMs),
where the energy function is parametrized using a neural
network.

Given samples $\{x_i\}_{i=1}^n$ from a target measure $\nu_p$,
training an EBM consists in selecting the best $\nu_{\beta f}$ with
energy $f\in \mathcal{F}$ according to a given criterion. The maximum
likelihood estimator (MLE) is defined as the maximizer of the
cross-entropy between $\nu_{\beta f} $ and the empirical measure
$\nu_n = \frac{1}{n} \sum_{i=1}^n \delta_{x_i}$,
$H(\nu_{\beta f},\nu_n) = \int_{\mathcal{X}} \log
(d\nu_{\beta f}/d\tau_{\mathcal{X}}) d\nu_n$. Using the expression
for~$\nu_{\beta f}$ in~\eqref{eq:nu_f_def0}, the MLE $\hat f$ is given by
\begin{align} 
  \begin{split} \label{eq:hat_mu0}
    \hat{f} = \argmax_{f \in \mathcal{F}} \Big\{-\int_{\mathcal{X}} f(x) d\nu_n(x)
    - \beta^{-1} \log Z_{\beta f}\Big\} = \argmax_{f \in \mathcal{F}} \Big\{-\frac{1}{n} \sum_{i=1}^n f(x_i)
    - \beta^{-1} \log Z_{\beta f}\Big\}.
\end{split}
\end{align}
The estimated distribution is then simply given by
$d\nu_{\beta \hat{f}} = Z^{-1}_{\beta \hat f} e^{-\beta \hat f}
d\tau_{\mathcal{X}}$.  Since
$D_{\text{KL}}(\nu || \nu') = H(\nu, \nu) - H(\nu',\nu)$, where
$D_{\text{KL}}(\nu || \nu') = \int_{\mathcal{X}} \log (d\nu/d\nu')
d\nu$ is the Kullback-Leibler (KL) divergence, observe that maximizing
the cross-entropy would be equivalent to minimizing the KL divergence if
the latter were finite.

Maximizing the objective in~\eqref{eq:hat_mu0} is challenging because it
contains the free energy functional $\log Z_{\beta f}$, which is unknown
to us. One way to go around this difficulty is to realize that the
gradient of $\log Z_{\beta f}$ over the parameters used
to parametrize $\beta f$  can be expressed as an expectation over the
the probability measure $\nu_{\beta f}$. For example, a simple
calculation shows that
\begin{equation}
  \label{eq:FEbeta}
   \partial_\beta \log Z_{\beta f}  = -\int_{\mathcal{X}} f(x)
   d\nu_{\beta f} (x).
 \end{equation}
 This offers the possibility to maximize the objective
 in~\eqref{eq:hat_mu0} by stochastic gradient ascent (SGA) by
 estimating the gradient of~\eqref{eq:hat_mu0} at every SGA step via
 sampling of $\nu_{\beta f}$, which can be done, for example, via
 Metropolis-Hastings Monte-Carlo. From a functional standpoint, this amounts to replacing the maximization problem
 in~\eqref{eq:hat_mu0} by the max-min
 \begin{equation}
   \label{eq:maxmin1}
   \max_{f \in \mathcal{F}} \min_{\nu \in \mathcal{P}(\mathcal{X})}
   \int_{\mathcal{X}} f
   (d\nu-d\nu_n) + \beta^{-1} D_{\text{KL}}(\nu||\tau_{\mathcal{X}})
 \end{equation}
 Indeed the minimization over $\nu \in \mathcal{P}(\mathcal{X}) $ can
 be carried explicitly and brings us back to~(\ref{eq:hat_mu0}).

Reformulating the problem as the max-min in~(\ref{eq:maxmin1})  also indicates why proceeding this way may not be optimal from computational standpoint. Indeed, sampling $\nu_{\beta f}$ to estimate the gradient of the objective~\eqref{eq:hat_mu0} at every step of training by SGA is typically tedious and costly, and it would be better to amortize this computation along the training. This suggests  to perform the minimization and the maximization in~\eqref{eq:maxmin1} concurrently  rather than in sequence. In practice this can be done using stochastic gradient descent-ascent (SGDA)  to simultaneously train the energy $f$ and sample  its associated Gibbs measure $\nu_{\beta f}$, using timescales for both that can be adjusted for efficiency. 
A necessary condition for the convergence of concurrent SGDA is that the min and the max in~\eqref{eq:maxmin1} commute, i.e.  the optimal value and optimal solutions of \eqref{eq:maxmin1} are equal to the ones of the problem
  \begin{equation}
   \label{eq:minmax1}
   \min_{\nu \in \mathcal{P}(\mathcal{X})} \max_{f \in \mathcal{F}} 
   \int_{\mathcal{X}} f
   (d\nu-d\nu_n) + \beta^{-1} D_{\text{KL}}(\nu||\tau_{\mathcal{X}}).
 \end{equation}
 A main theoretical contribution of the present paper is to use infinite-dimensional Fenchel duality results to show the
equivalence between \eqref{eq:maxmin1} and \eqref{eq:minmax1} for
energies $f$ that belong to metric ball of Barron space $\mathcal{F}_1$, which is a Banach space containing infinitely-wide neural networks. While the duality between~\eqref{eq:maxmin1} and \eqref{eq:minmax1} is only a necessary condition for the convergence of SGDA, we observe that it does work in practice by performing experiments with shallow neural network energies and investigating which relative timescales in SGDA lead to faster convergence.
 
 We also show
 that a simple modification of this SGDA algorithm interpolates between
 standard MLE training on~(\ref{eq:maxmin1}) and training using score
 matching (SM), which is another objective function used to train EBMs
 that has gained popularity in recent years. The SM metric or relative Fisher information between two absolutely continuous probability measures $\nu, \nu'$ is defined as $\text{SM}(\nu,\nu') = \int_{\mathcal{X}} |\nabla \log \frac{d\nu}{d\tau_{\mathcal{X}}}(x) - \nabla \log \frac{d\nu'}{d\tau_{\mathcal{X}}}(x) |^2 \ d\nu(x)$. The key insight from \cite{hyvarinen05estimation} is that via integration by parts this quantity may be rewritten without involving the log-density of $\nu$, which leads to the following loss function:
\begin{align}
\begin{split} \label{eq:loss_L0}
\min_{f\in \mathcal{F}} \frac{1}{n} \sum_{i=1}^n \beta^{-1} \Delta
f(x_i) + \frac{1}{2} |\nabla f(x_i)|^2 = \min_{f\in \mathcal{F}}
\int_{\mathcal{X}}\left(\beta^{-1} \Delta f(x) + \frac{1}{2} |\nabla
  f(x)|^2 \right) d\nu_n(x)
\end{split}
\end{align}
Score matching is computationally more tractable than maximum likelihood because it avoids estimating the partition function $Z_{\beta f}$ or its gradient altogether. However, it has the known drawback that it may fail to distinguish distributions in some instances---this is because the SM metric is weaker than the KL divergence. 
\textbf{Related work.}
Our work is based on general Fenchel duality results (\autoref{sec:general_duality}) that may be useful in applications beyond the main focus of this paper (see \autoref{sec:two_layer_nn_sampling}). These theorems are a generalization of results stated in the compact case in~\cite{domingoenrich2021energybased} in their Appendix D. Similar duality results have been studied extensively in the area of maximum entropy (maxent) models (reviewed in Ch. 12 of \cite{mohri2012foundations}). The first maxent duality principle was due to \cite{jaynes1957information}. Maxent models have been applied since the 1990s in natural language processing and in species habitat modeling among others, and studied theoretically especially since the 2000s \citep{altun2006unifying, dudik2007maximum}. 


Recently \cite{dai2019kernel} leveraged duality arguments in the context of maximum likelihood EBMs, although in a form different from ours. Their duality result works in the more restrictive setting of ``lazy'' energies lying in RKHS balls and probability measures with $L^2$ densities, and they derive it directly from a general theorem that works for reflexive Banach spaces~\citep[Ch. 6, Thm. 2.1]{ekeland1999convex}. Our Fenchel duality results, which work for Borel probability measures and feature-learning ($\mathcal{F}_1$) energies, are  more general because we must rely on measure spaces, which are non-reflexive Banach spaces. Their algorithm is also different: they do not evolve generated samples, but rather use a transport parametrization of the energy.
\cite{dai2019exponential} expand the work~\citep{dai2019kernel} combining it with Hamiltonian Monte Carlo.

A precursor of modern machine learning EBMs were restricted Boltzmann machines (RBMs),
first trained via contrastive divergence or CD \citep{hinton2002training} - which estimates the gradient of the log-likelihood via approximate MCMC samples of the trained model. It later led to maximum likelihood training of EBMs~\citep[see, e.g.,][among others]{xie2016atheory, xie2017synthesizing, du2019implicit}. A popular variant of CD is persistent contrastive divergence or PCD \citep{tieleman2008training, tieleman2009using}, in which the MCMC samples are evolved and reused over gradient computations to be progressively equilibrated. Training EBMs by updating the energy parameters and the samples simulteanously like we do resembles PCD.

A vast array of EBM losses alternative to maximum likelihood have been developed recently \citep{song2021train} with the goal of avoiding the MCMC procedure, e.g. score matching \citep{hyvarinen05estimation} and related methods such as denoising score matching \citep{pascal2011aconnection}, and score-based generative modeling~\citep{song2019generative, song2020improved, ho2020denoising, song2021scorebased}. Our work should also be contrasted with the literature on convergence for minimax problems: \cite{heusel2017gans, lin2020ongradient} among others argue for two-timescale GDA and SGDA, while our experiments show the benefits of simultaneous training, suggesting that further work is needed for clarification.

Finally, \autoref{sec:implicit_EBM} has links with maximum mean discrepancy (MMD) flows. MMDs are probability metrics that were first introduced by~\citet{gretton2007kernel, gretton2012akernel} for kernel two-sample tests, and that have been successful as discriminating metrics in generative modeling \citep{li2015generative, dziugaite2015training, li2017mmd}. 
\citet{arbel2019maximum} study theoretically the convergence of unregularized MMD gradient flow (our equation \eqref{eq:wasserstein_EBM_2} with $\tilde{\beta}^{-1} = 0$). In their experiments, they observe that noisy updates ($\tilde{\beta}^{-1} > 0$) are needed for good generalization. Our work shows that their algorithm is exactly training maximum likelihood EBMs energies in an RKHS ball of radius that depends on the noise level. 


\section{Background} \label{sec:background}

In this section, we provide preliminary background on the Barron space $\mathcal{F}_1$, and on the specific form of EBM losses for this type of energies.  

\textbf{Notation.}
If $V$ is a normed vector space, $\mathcal{B}_V(\beta)$ denotes the closed ball of $V$ of radius $\beta$, and $\mathcal{B}_V := \mathcal{B}_V(1)$. If $K$ is a subset of the Euclidean space, $\mathcal{P}(K)$ is the set of Borel probability measures, $\mathcal{M}(K)$ is the space of Radon (i.e. signed and finite) measures, and $\mathcal{M}^{+}(K)$ is the set of non-negative Radon measures. If $\gamma \in \mathcal{M}(K)$, then ${\|\gamma\|}_{\text{TV}} = \int_K d|\gamma|$ is the total variation (TV) norm of $\gamma$, which turns $\mathcal{M}(K)$ into a Banach space. Unless otherwise specified, $\sigma : \R \to \R$ is a generic non-linear activation function. The ReLU activation is denoted by $(z)_{+} = \max \{z, 0\}$. $\tau$ denotes a fixed base probability measure; a subindex specifies the space it is defined over. $\mathbb{S}^d \subseteq \R^{d+1}$ is the $d$-dimensional hypersphere; $\log$ is the natural logarithm; $\lambda$ is the Lebesgue measure. Given $\nu,\nu'\in \mathcal{P}(K)$, $D_{\text{KL}}(\nu\| \nu') = \int_{K} \log \frac{d\nu}{d\nu'} \ d\nu$ is the KL divergence and $H(\nu, \nu') = -\int_K \log(\frac{d\nu'}{d\tau}) d\nu$ is the cross-entropy. 


\textbf{The Barron space $\mathcal{F}_1$.} 
Let $\mathcal{X} \subseteq \R^{d_1}$,  $\Theta \subseteq \R^{d_2}$, $\phi : \mathcal{X} \times \Theta \to \R$, and $\tau_{\Theta}$ be a fixed base probability measure over $\Theta$.
$\mathcal{F}_1$ is defined as the Banach space of functions $f : \mathcal{X} \rightarrow \R$ such that, for some  Radon measure $\gamma \in \mathcal{M}(\Theta)$, for all $x \in \mathcal{X}$ we have $f(x) = \int_{\Theta} \phi(x, \theta) \ d\gamma(\theta)$. We define 
the norm of $\mathcal{F}_1$ as
    $\|f\|_{\mathcal{F}_1} = \inf \left\{ {\|\gamma\|}_{\text{TV}} \ | f(\cdot) = \int_{\Theta} \phi(\cdot, \theta) \ d\gamma(\theta) \right\}.$
This construction was introduced by \cite{bach2017breaking}, who first used the notation $\mathcal{F}_1$ and focused in particular on the case $\mathcal{X} \subseteq \R^d$, $\Theta = \mathbb{S}^{d}$ and $\phi(x, \theta) = \text{ReLu}^{k}(\langle (x,1), \theta \rangle)$ for some $k \in \mathbb{Z}_{+}$. This space is also known by the name of Barron space \citep{e2019apriori,e2020banach} in reference to the classic work \cite{barron1993universal}.

We denote by $\mathcal{F}_1$-EBMs the energy-based models for which the energy class $\mathcal{F}$ is the unit ball $\mathcal{B}_{\mathcal{F}_1}(1)$ of $\mathcal{F}_1$. Notice that the class $\{ \beta f | f \in \mathcal{F} \}$ is equal to the ball $\mathcal{B}_{\mathcal{F}_1}(\beta)$. Such models may be regarded as abstractions of more complex deep EBMs, in that they incorporate feature learning, and they were first studied by \cite{domingoenrich2021energybased}, which provide statistical guarantees. They are to be contrasted with $\mathcal{F}_2$-EBMs, for which $\mathcal{F}$ is the unit ball $\mathcal{B}_{\mathcal{F}_2}(1)$. $\mathcal{F}_2$-EBMs, which we study in \autoref{sec:implicit_EBM}, have fixed features and showed worse statistical performance in experiments \citep{domingoenrich2021energybased}.

\textbf{Maximum likelihood for $\mathcal{F}_1$-EBMs.} We rewrite the maximum likelihood problem \eqref{eq:hat_mu0} for the case in which $\mathcal{F} = \mathcal{B}_{\mathcal{F}_1}(1)$. Since an arbitrary element $f$ of $\mathcal{F}_1$ can be expressed as $f(x) = \int_{\Theta} \phi(x, \theta) \ d\gamma(\theta)$, with $\|f\|_{\mathcal{F}_1}$ equal to the infimum of $\|\gamma\|_{\text{TV}}$ for all such $\gamma$, the maximum likelihood energy is
$f_\text{MLE} = \int_\Omega \varphi(\cdot,\theta) d\gamma_\text{MLE}(\theta)$,
where
\begin{align} 
\begin{split} \label{eq:EBM_dual_f1}
    \gamma_\text{MLE} = \argmin_{\substack{\gamma \in \mathcal{M}(\Theta) \\ \|\gamma\|_{\text{TV}} \leq 1}} \bigg\{\frac{1}{n} \sum_{i=1}^n \int_{\Theta} \phi(x_i, \theta) \ d\gamma(\theta) + \frac{1}{\beta} \log \bigg( \int_{\mathcal{X}} \exp \left(- \beta \int_{\Theta} \phi(x, \theta) \ d\gamma(\theta) \right) d\tau_{\mathcal{X}}(x) \bigg) \bigg\}.
\end{split}
\end{align}

\textbf{Score matching for $\mathcal{F}_1$-EBMs.} 
Suppose that $\mathcal{X} \subseteq \R^{d_1}$ is a manifold without boundaries. 
Assume that $\int_{\mathcal{X}} |\nabla_x \phi(x, \theta) \cdot \nabla \frac{d\nu_p}{d\tau_{\mathcal{X}}}(x)| \ d\tau_{\mathcal{X}}(x)$ is upper-bounded by some constant $K$ for all $\theta \in \Theta$. Assume also that $\sup_{\theta \in \Theta} \|\nabla_x \phi(x, \theta)\| < \eta(x)$ and that $\int_{\mathcal{X}} |\eta(x)|^2 \ d\nu_p(x) < \infty$. \autoref{prop:SM_expression} in \autoref{sec:proofs_SM} shows that the optimal energy under the score matching loss with energy class $\mathcal{B}_{\mathcal{F}_1}(1)$ is $f_{\rm SM} = \int_\Omega \varphi(\cdot,\theta) 
d\gamma_{\rm SM} (\theta)$ where
\begin{align} 
\begin{split} \label{eq:SM_objective_2}
    \gamma_{\rm SM}= \argmin_{\substack{\gamma \in \mathcal{M}(\Theta) \\ \|\gamma\|_{\text{TV}} \leq 1}}
    \int_{\Theta} \int_{\mathcal{X}}\bigg(\frac{\nabla_x \phi(x, \theta) }{2} \cdot \nabla_x \int_{\Theta} \phi(x, \theta') \ d\gamma(\theta') - \beta^{-1} \Delta_x \phi(x, \theta) \bigg) d\nu_n(x) d\gamma(\theta).
\end{split}
\end{align}

\section{Duality for $\mathcal{F}_1$-EBMs} \label{sec:implicit_mle_f1}


As a corollary of the duality result from \autoref{sec:kl_regularized_Linfty},
we derive an alternative objective for $\mathcal{F}_1$-EBMs trained via maximum likelihood, the original objective being \eqref{eq:EBM_dual_f1} and we develop an algorithm to solve this alternative problem. To this end, we make:
\vspace{-4pt}
\begin{assumption} \label{ass:phi_new}
Let $\phi : \mathcal{X} \times \Theta \to \R$ be a continuous function such that either $\mathcal{X}$ is compact or (i) for any fixed $\theta \in \Theta$, $\phi(x, \theta) \leq \xi(x)$ for some strictly positive $\xi : \mathcal{X} \to \R$, and (ii) 
$\xi(x) + \log(\xi(x)) = o\left(-\log\left(\frac{d\tau_{\mathcal{X}}}{d\lambda}(x) \right) - (d_1+\epsilon) \log \|x\|_2 \right) \text{ as } \|x\|_2 \to \infty$ for some $\epsilon > 0$.
\end{assumption}
\vspace{-4pt}
In particular, this assumption holds for ReLU network energies when setting $\mathcal{X} = \R^{d_1}$, $\Theta = \R^{d_1 + 1}$, $\phi(x,\theta) = \sigma(\langle (x,1), \theta \rangle)/\|\theta\|$ and $\tau_{\mathcal{X}}$ Gaussian (and in many other settings).
\begin{theorem}
\label{cor:EBM_f1}
Under Assumption \ref{ass:phi_new}, the maximum likelihood problem \eqref{eq:EBM_dual_f1} is the Fenchel dual of
\begin{align}
    \begin{split} \label{eq:EBM_primal_3}
    \min_{\nu \in \mathcal{P}(\mathcal{X})} \max_{\substack{\gamma \in \mathcal{M}(\Theta), \\ \|\gamma\|_{\text{TV}} \leq 1}} \bigg\{\beta^{-1} D_{\text{KL}}(\nu||\tau_{\mathcal{X}})
    + \int_{\Theta} \int_{\mathcal{X}} \phi(x, \theta) \ d(\nu - \nu_n)(x) \ d\gamma(\theta) \bigg\}.
\end{split}
\end{align}
Moreover, the solution $\nu^{\star}$ of \eqref{eq:EBM_primal_2} is precisely the Gibbs measure for the optimal $\gamma^{\star}$ in \eqref{eq:EBM_dual_f1}, that is, $\frac{d\nu^{\star}}{d\tau_{\mathcal{X}}}(x) = \frac{1}{Z_{\beta}}\exp \left( - \beta \int_{\Theta} \phi(x, \theta) \ d\gamma^{\star}(x) \right)$.
\end{theorem}
\vspace{-4pt}
As mentioned in \autoref{sec:setup}, \autoref{cor:EBM_f1} shows that the min with respect to the probability measure $\nu$ and the max with respect to the parameter measure $\mu$ can be exchanged; the optimal values and the minimax points of both problems coincide. This puts in more solid footing the tuning of timescales that we propose in the next section and analyze experimentally in \autoref{sec:experiments}.
In \autoref{cor:EBM_f1}, if we replace the $\mathcal{F}_1$ ball by the unit ball of the related space $\mathcal{F}_2$, an analogous duality result links the maximum likelihood problem with the entropy regularized MMD flow from \cite{arbel2019maximum} (see \autoref{sec:implicit_EBM}).

\section{Algorithm for $\mathcal{F}_1$-EBMs}
In this section we introduce measure dynamics to solve the minimax problems \eqref{eq:EBM_dual_f1}-\eqref{eq:EBM_primal_3}.
We consider the triple $(\gamma^+,\gamma^-,\nu)$ where the nonegative measures $\gamma^\pm$ are defined through the Hahn decomposition of $\gamma = \gamma_+-\gamma_-$.  
Then we introduce  coupled gradient flows for this triple, in which $\gamma^+_t$ and $\gamma^-_t$ evolve via a  Wasserstein-Fisher-Rao gradient flow \citep{chizat2018unbalanced} and  $\nu_t$ evolves via a Wasserstein gradient flow \citep{santambrogio2017euclidean}:
\begin{align}
\begin{split} \label{eq:coupled_dynamics_f1}
    \partial_{t} \gamma^\sigma_t &= -\alpha \sigma \nabla_\theta \cdot \left( \gamma^\sigma_t  \nabla_{\theta} F_t(\theta) \right) + \alpha \gamma^\sigma_t \left( \sigma F_t(\theta) - K_t \right), \quad \text{where} \ \ \sigma = \pm 1, \ \  \gamma_t^\sigma = \gamma_t^\pm  \\ \partial_{t} \nu_t &= \nabla_x \cdot \bigg( \nu_t \left(\nabla_x f_t(x)- \beta^{-1} \nabla \log \frac{d\tau_{\mathcal{X}}}{d\lambda} \right) \bigg) + \beta^{-1} \Delta_x \nu_t,
\end{split}
\end{align}
where $\alpha$ is a tunable parameter and we defined
\begin{align} 
\begin{split} \label{eq:def_F_f_mu}
F_t(\theta) &= \int_{\mathcal{X}}  \phi(x, \theta) \ d(\nu_t \!-\! \nu_n)(x), \quad
f_t(x) = \int_{\Theta} {\phi}(x, \theta) \ (d\gamma_t^+-d\gamma^-_t)(\theta), \\ K_t &= \mathds{1}_{ \|\gamma^+_t\|_{\text{TV}}+\|\gamma^-_t\|_{\text{TV}} \geq 1} \int_\Theta  F_t(\theta) (d\gamma_t^+-d\gamma^-_t)(\theta).
\end{split}
\end{align}
The initialization of \eqref{eq:coupled_dynamics_f1} is $\nu_0 = \nu_n$ and $\gamma_0^\pm=0$ (such that the initial energy is null). The term $K_t$ keeps the total variation of $\gamma_t$ below one.
The parameter $\alpha$ acts as a relative timescale. Notice that different values of $\alpha$ lead to different behaviors of the dynamics. Setting $\alpha  \ll 1$ corresponds heuristically to solving the primal formulation of maximum likelihood with persistent MCMC samples (equation \eqref{eq:EBM_dual_f1}), as the measures $\gamma^\pm_t$ on neurons evolve slower than the measure $\nu_t$ on particles. In contrast if $\alpha\gg1$, $\gamma^\pm_t$ evolves faster than $\nu_t$ and if the optimization is well behaved, at all times $\gamma_t=\gamma^+_t-\gamma^-_t$ remains close to minimizing the inner maximization problem of \eqref{eq:EBM_primal_3} with $\gamma = \gamma_t$. Thus, $\alpha\gg1$ is heuristically solving \eqref{eq:EBM_primal_3}. The experiments in the next section suggest that $\alpha \approx 1$ yields the fastest convergence computationally. 

\autoref{prop:particle_dynamics_f1} below states that the solution $(\mu_t, \nu_t)$ may be approximated using coupled particle systems (see proof in \autoref{sec:proofs_implicit_f1}) and is the basis for Alg.~\ref{alg:implicit_ebm_f1}. The link between particle systems and measure PDEs is through a classical technique known as propagation of chaos \citep{sznitmantopics1991} and it has been used previously for similar coupled systems in the machine learning literature \citep{domingoenrich2020amean}, as well as to analyze the convergence of gradient descent for infinite-width neural networks \citep{rotskoff2018neural, mei2018mean, chizat2018global}. 
\begin{proposition}
\label{prop:particle_dynamics_f1}
Let $\{\theta_0^{(j)}\}_{j=1}^{m}$ be initial features sampled uniformly over $\Theta$, let $\{\sigma_j\}_{j=1}^{m}$ be uniform samples over $\{\pm 1\}$ and let $\{w_0^{(j)} = 1\}_{j=1}^{m}$ be the initial weight values, which are set to 1. Let $\{X_0^{(i)}\}_{i=1}^{N}$ be the initial ``generated'' samples, which are chosen i.i.d. uniformly from the target sample set $\{x_i\}_{i=1}^{n}$. Consider the system of ODEs/SDEs:
\begin{align}
\begin{split} \label{eq:coupled_particles_MLE}
    \frac{d\theta_t^{(j)}}{dt} &= \alpha \sigma_j \nabla \tilde F_t(\theta_t^{(j)}), 
    \quad \frac{dw_t^{(j)}}{dt} = \alpha w_t^{(j)} (\sigma_j \tilde F_t(\theta_t^{(j)}) - \tilde{K}_t)
    \\
    dX_t^{(i)} &= \left( - \nabla \tilde f_t(X_t^{(i)}) + \beta^{-1} \nabla \log \frac{d\tau_{\mathcal{X}}}{d\lambda}(X_t^{(i)}) \right) \ dt + \sqrt{2 \beta^{-1}} \ dW_{t}^{(i)}
\end{split}    
\end{align}
where 
\begin{align}
\begin{split} \label{eq:def_tilde_F_f_K}
\tilde F_t(\theta) &= \frac{1}{N} \sum_{i=1}^{N} \phi(X_t^{(i)}, \theta) - \frac{1}{n} \sum_{i=1}^{n} \phi(x_i, \theta), \quad \tilde f_t(x) = \frac{1}{m} \sum_{j=1}^m \sigma_j w^{(j)}_t \phi(x,\theta^{(j)}_t), \\ \tilde{K_t} &= \mathds{1}_{\sum_{j=1}^m w^{(j)}_t \geq m} \  \frac1m \sum_{j=1}^m \sigma_j w^{(j)}_t \tilde F_t(\theta_t^{(j)}).
\end{split}
\end{align}
are the empirical counterparts of the functions in \eqref{eq:def_F_f_mu}.
Then the system~\eqref{eq:coupled_particles_MLE} approximates the measure dynamics. 
Namely, as $m, N \rightarrow \infty$:
\begin{itemize}[leftmargin=0.3cm,itemsep=0pt,topsep=0.1pt]
\item the empirical measure 
$\hat{\gamma}_t = \frac{1}{m} \sum_{j=1}^m \sigma_j w_t^{(j)} \delta_{\theta_t^{(j)}}$
converges weakly to the solution $\gamma_t = \gamma_t^{+} - \gamma_t^{-}$ of \eqref{eq:coupled_dynamics_f1} with uniform initialization for any finite time interval $[0,T]$, and
\item the empirical measure $\hat{\nu}_t = \frac{1}{N} \sum_{i=1}^N \delta_{X_t^{(i)}}$ converges weakly to the solution $\nu_t$ of \eqref{eq:coupled_dynamics_f1} for any finite time interval $[0,T]$.
\end{itemize}
\end{proposition}
Importantly, the system of ODEs/SDEs in~\eqref{eq:coupled_particles_MLE} may be solved via forward Euler steps on $\{\theta_j\}_{j=1}^m$ and $\{w_j\}_{j=1}^m$ (or rather, $\{\log w_j\}_{j=1}^m$), and Euler-Maruyama updates on $\{X_0^{(i)}\}_{i=1}^{N}$. Such a discretization yields Alg.~\autoref{alg:implicit_ebm_f1}. We reemphasize that when $\alpha \ll 1$, Alg.~\autoref{alg:implicit_ebm_f1} is simply the classical maximum likehilood algorithm with persistent particles (up to the minor detail that gradient descent is applied to $\log w^{(j)}$ instead of $w^{(j)}$).

\begin{algorithm}
\caption{Dual $\mathcal{F}_1$-EBM training} 
\begin{algorithmic}
\STATE \textbf{Input:} $n$ samples $\{x_i\}_{i=1}^n$ of the target distribution, stepsize $s$, stepsize ratio $\alpha$.
\STATE Initialize features $(\theta_0^{(j)})_{j=1}^m$ uniformly over $\Theta$, weights $(w_0^{(j)})_{j=1}^m$  in $[0,1)$, signs $(\sigma_j)_{j=1}^m$  over $\{\pm 1\}$. Initialize generated samples $\{X_0^{(i)}\}_{i=1}^N$ unif. i.i.d. from $\{x_i\}_{i=1}^n$.
\FOR {$t=0,\dots,T-1$}
    \FOR{$i=1,\dots,N$}
        \STATE Sample $\zeta_{t}^{(i)}$ from the $d_1$-variate standard Gaussian.
        \STATE Perform Euler-Maruyama update:
        $X_{t+1}^{(i)} = X_{t}^{(i)} - s (\nabla \tilde f_t(X_t^{(i)}) + \beta^{-1} \nabla \log \frac{d\tau_{\mathcal{X}}}{d\lambda}(X_t^{(i)}))
        + \sqrt{2 \beta^{-1} s}\, \zeta_{t}^{(i)}$, where $\tilde f_t$ is defined in \eqref{eq:def_tilde_F_f_K}.
    \ENDFOR
    \FOR{$j=1,\dots,m$}
        \STATE Update
        $\theta_{t+1}^{(j)} = \theta_{t}^{(j)} + s \alpha \sigma_j \nabla \tilde F_t(\theta_t^{(j)})$, where $\tilde F_t$ is defined in \eqref{eq:def_tilde_F_f_K}.
        \STATE Update $\tilde{w}_{t+1}^{(j)} = w_{t+1}^{(j)} \exp( s \alpha \sigma_j \tilde F_t(\theta_t^{(j)}) )$. 
        \STATE Normalize, when it is needed: $w_{t+1}^{(j)} = \tilde{w}_{t+1}^{(j)}/\max\left(\frac{1}{m}\sum_{j'=1}^{m} \tilde{w}_{t+1}^{(j')},1\right)$.
    \ENDFOR
\ENDFOR
\STATE {\bfseries Output:} samples $\{X_T^{(i)}\}_{i=1}^N$, energy $f_T(x) := \frac{\beta}{m} \sum_{j=1}^{m} \sigma_j w_j \phi(x, \theta_j)$.
\end{algorithmic}
\label{alg:implicit_ebm_f1}
\end{algorithm}

\section{Links between maximum likelihood and score matching $\mathcal{F}_1$-EBMs} \label{sec:links_MLE_SM}
In this section we uncover how the score matching loss fits seamlessly as a variant of Alg.~\ref{alg:implicit_ebm_f1}, in the form of particle restarts. Interestingly, we can modify the PDE \eqref{eq:coupled_dynamics_f1} in a way that allows us to make a connection with score matching. To this end, let us introduce the following coupled measure PDE:
\begin{align}
\begin{split} \label{eq:coupled_dynamics_main}
    \partial_{t} \gamma^\sigma_t &= -\alpha \sigma \nabla_\theta \cdot \left( \gamma^\sigma_t  \nabla_{\theta} F_t(\theta) \right) + \alpha \gamma^\sigma_t \left( \sigma F_t(\theta) - K_t \right), 
    \\ \partial_{t} \nu_t &= \nabla_x \cdot \left( \nu_t \left(\nabla_x f_t(x)- \beta^{-1} \nabla \log \frac{d\tau_{\mathcal{X}}}{d\lambda} \right) \right) + \beta^{-1} \Delta_x \nu_t - \alpha \left( \nu_t - \nu_{n} \right).
\end{split}
\end{align}
Remark that the only difference between this equation and the PDE \eqref{eq:coupled_dynamics_f1} for dual maximum likelihood training is the term $-\alpha (\nu_t - \nu_{n})$, which draws $\nu_t$ closer to the empirical target measure $\nu_n$. We have:
\begin{proposition} \label{prop:limitSM}
In the limit $\alpha \to\infty$, the equations for $\gamma^\sigma_t$ in \eqref{eq:coupled_dynamics_main} reduce to
\begin{align} \label{eq:WFR_SM}
    \partial_{t} \gamma^\sigma_t = \sigma \nabla_\theta \cdot \left( \gamma^\sigma_t \nabla_{\theta} V(\gamma_t)(\theta) \right) - \gamma^\sigma_t \left(\sigma V(\gamma_t)(\theta) - \bar V(\gamma_t) \right), 
\end{align}
where $\gamma_t=\gamma^+_t-\gamma^-_t$, $\bar V(\gamma) = \int_\Theta  V(\gamma) \, d\gamma$, and $V(\gamma)(\theta)$ is the Frechet derivative of the score matching loss $L : \mathcal{M}(\Theta) \rightarrow \R$ defined in \eqref{eq:SM_objective_2}. 
\end{proposition}
That is, in the large $\alpha$ limit, equation \eqref{eq:coupled_dynamics_main} is equivalent to the Wasserstein-Fisher-Rao gradient flow of a loss $L$ which, remarkably,
is the score matching loss for $\mathcal{F}_1$-EBMs. This means that adding the term $- \alpha \left( \nu_t - \nu_{n} \right)$ to the dual maximum likelihood measure dynamics and letting $\alpha\to\infty$ we recover the score matching dynamics. This additional term can be  easily  implemented at particle level by replacing each training sample $X_t^{(i)}$ by some random target sample in $\{x_i\}_{i=1}^n$ with probability $p_R=1-e^{-\alpha t} = \alpha t + o(t)$ for every time interval of length $t$ (proof in~\autoref{sec:proofs_SM}). Similar birth-death processes were used in \citep{rotskoff2019global} in the context of neural network regression. Hence, the score matching scheme corresponds to modifying Alg.~\ref{alg:implicit_ebm_f1} by restarting each sample $X^{(i)}_t$ as a uniformly chosen sample in $\{x_i\}_{i=1}^n$ with probability $p_R = s \alpha$, right before the Euler Maruyama update. The restart probability acts as a knob that allows us to interpolate between score matching and maximum likelihood. In the experimental simulations we use a parameter $\alpha'$ different from $\alpha$ for the reinjection term, to study a wider range of behaviors. 

In summary, score matching differs from dual maximum likelihood in that the trained measure is being ``pulled'' towards the target measure at all times via particle restarting. Such constant pulling should be useful to alleviate sampling problems due to metastability issues which may arise with dual maximum likelihood. However, dual maximum likelihood  has the upside of providing samples of the learned EBM as a byproduct of training, which score matching does not. 
It is also interesting to contrast our approach to score matching with the works \cite{sutherland2018efficient, arbel2018kernel}, which using different techniques propose algorithms to train EBMs with RKHS energies via score matching; the connection that we identify between maximum likelihood and score matching is novel.
Finally, notice that a particle discretization of the flow \eqref{eq:WFR_SM} yields an alternative straightforward algorithm to train $\mathcal{F}_1$-EBMs via score matching (see \autoref{subsec:direct_optim}), which can be linked directly to Alg.~\ref{alg:implicit_ebm_f1} with particle restarts.

\section{Experiments} \label{sec:experiments}

We perform two sets of experiments. Our first set of experiments involves simulating the measure PDEs in dimension 1 to understand which timescale $\alpha$ provides better convergence properties, and whether particle restarts help. Since the dimension is low, we solve the PDEs exactly by gridding the space and do not need to resort to particle dynamics. In our second set of experiments 
we apply Alg.~\ref{alg:implicit_ebm_f1} to high-dimensional spheres. 

\subsection{PDE simulations}
\textbf{Setup.} In the notation of \autoref{sec:background}, we set both the base space $\mathcal{X}$ and the parameter space $\Theta$ to be the 1-dimensional torus $[0,1]$ (i.e. with periodic boundary conditions). 
We set $\phi : \mathcal{X} \times \Theta \to \R$ as a Gaussian with fixed variance $\delta^2$: $\phi(x,\theta) = \sum_{k\in \mathbb{Z}} \exp(2i\pi k(x-\theta)-\frac{1}{2} \delta^2k^2)$. 
We define the target measure $\nu_p$ with energy $E_p(x) = \int_{\Theta} \phi(x,\theta) g(\theta) \, d\theta$, where $g(\theta) = - \log(p\exp(a_1\cos(2\pi (\theta-\bar{\theta}_1)) - a_1) + q\exp(a_2\cos(2\pi (\theta-\bar{\theta}_2) - a_2))$ with $q=1-p$.
The measure $\nu_p$ and the energy $E_p$ are shown in \autoref{fig:1d_energy_density}. $\nu_p$ is bimodal: $p, q=1-p > 0$ control the relative size of the modes and $a_1>0$, $a_2>0$ control their width. Note that $E_p$ belongs to the space $\mathcal{F}_1$, which means that the target energy may be recovered exactly. We work with the population loss so that there is no statistical error. 
Consistently, we use no regularization term in the equation for $\gamma_t^\sigma$ in \ref{eq:coupled_dynamics_main} ($K_t=0$), and we only consider $\sigma=-1$ since the target $g_p$ is negative. For simplicity we also neglect the transport term in the equation for $\gamma_t^\sigma$ in \ref{eq:coupled_dynamics_main} and we set $\tau_{\mathcal{X}} = \lambda$ in the equation for $\nu_t$. These PDE are solved using a pseudo-spectral code with an exponential integrator in time.

\begin{figure*}
    \begin{minipage}[c]{10cm}
    \centering
    \includegraphics[width=0.9\textwidth]{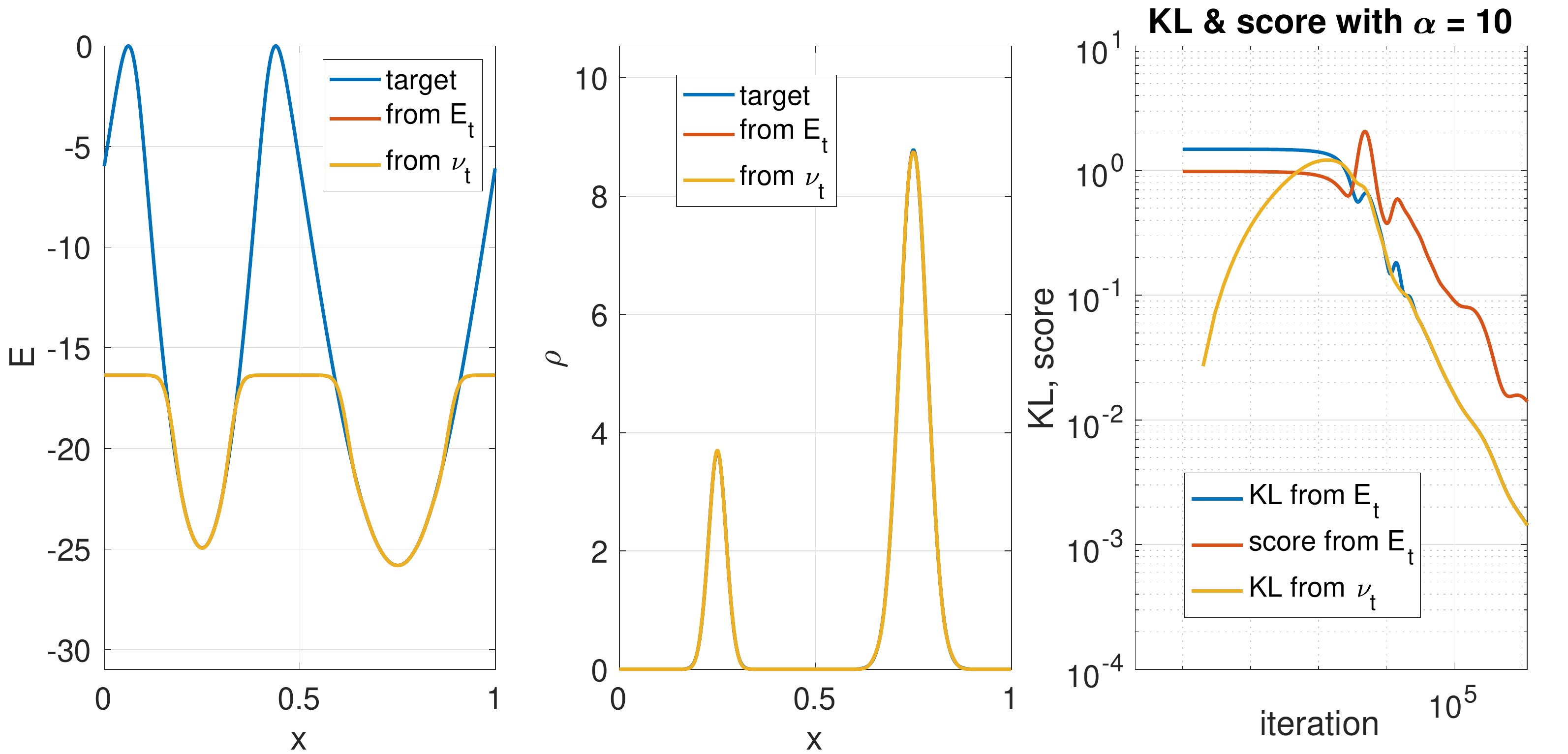}
    \end{minipage}
    \begin{minipage}[c]{5.5cm}
    \caption{(Left) Energy $E_p$ of the target distribution in blue, learned energy in red and minus log-density of $\nu_t$ in yellow. (Center) Target density in blue, density computed from learned energy in red, $\nu_t$ in yellow. (Right) Evolution of the KL computed with respect to the learned EBM in blue, of the KL with respect to $\nu_t$ in yellow, and of the SM metric in red.}
    \label{fig:1d_energy_density}
    \end{minipage}
\end{figure*}

\begin{figure*}
    \centering
    \includegraphics[width=0.495\textwidth]{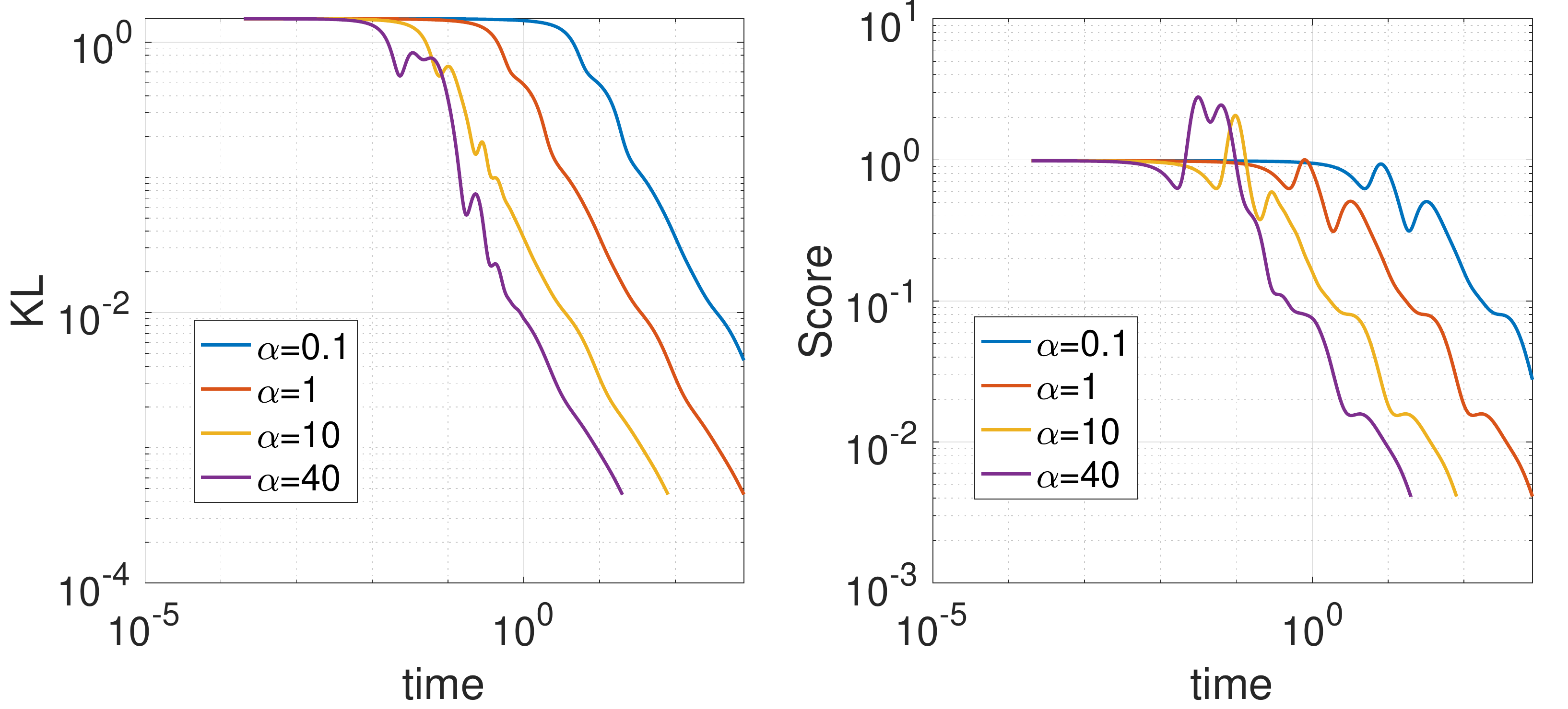} 
    \includegraphics[width=0.495\textwidth]{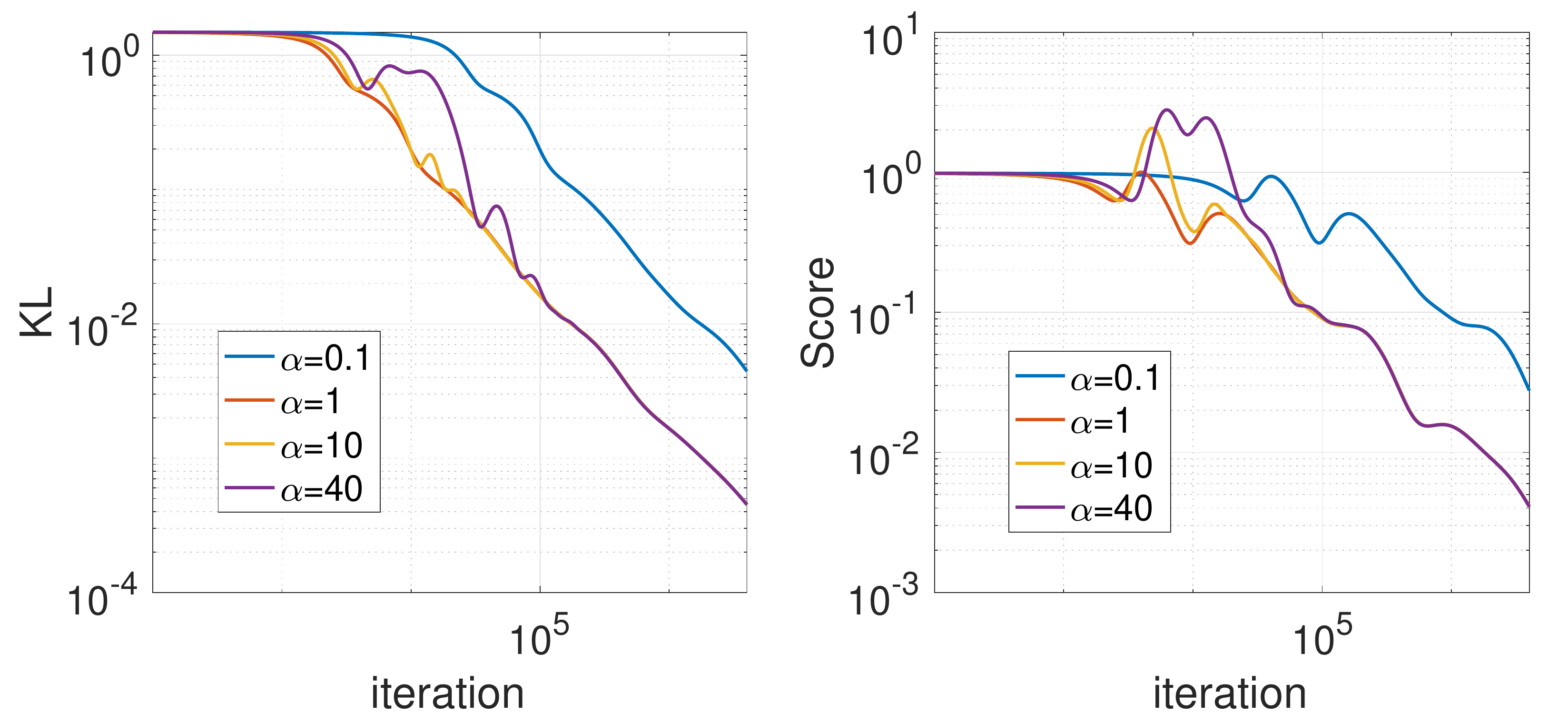}
    \caption{(Left figures) Evolution of the KL divergence and SM metric between the target and learned distributions in the physical (real) time, for different values of $\alpha$. (Right figures) Same curves with time rescaled by $\max\{\alpha,1\}$.
    }
    \label{fig:1d_evolution}
\end{figure*}

\begin{figure*}
    \begin{minipage}[c]{10cm}
    \centering
    \includegraphics[width=0.9\textwidth]{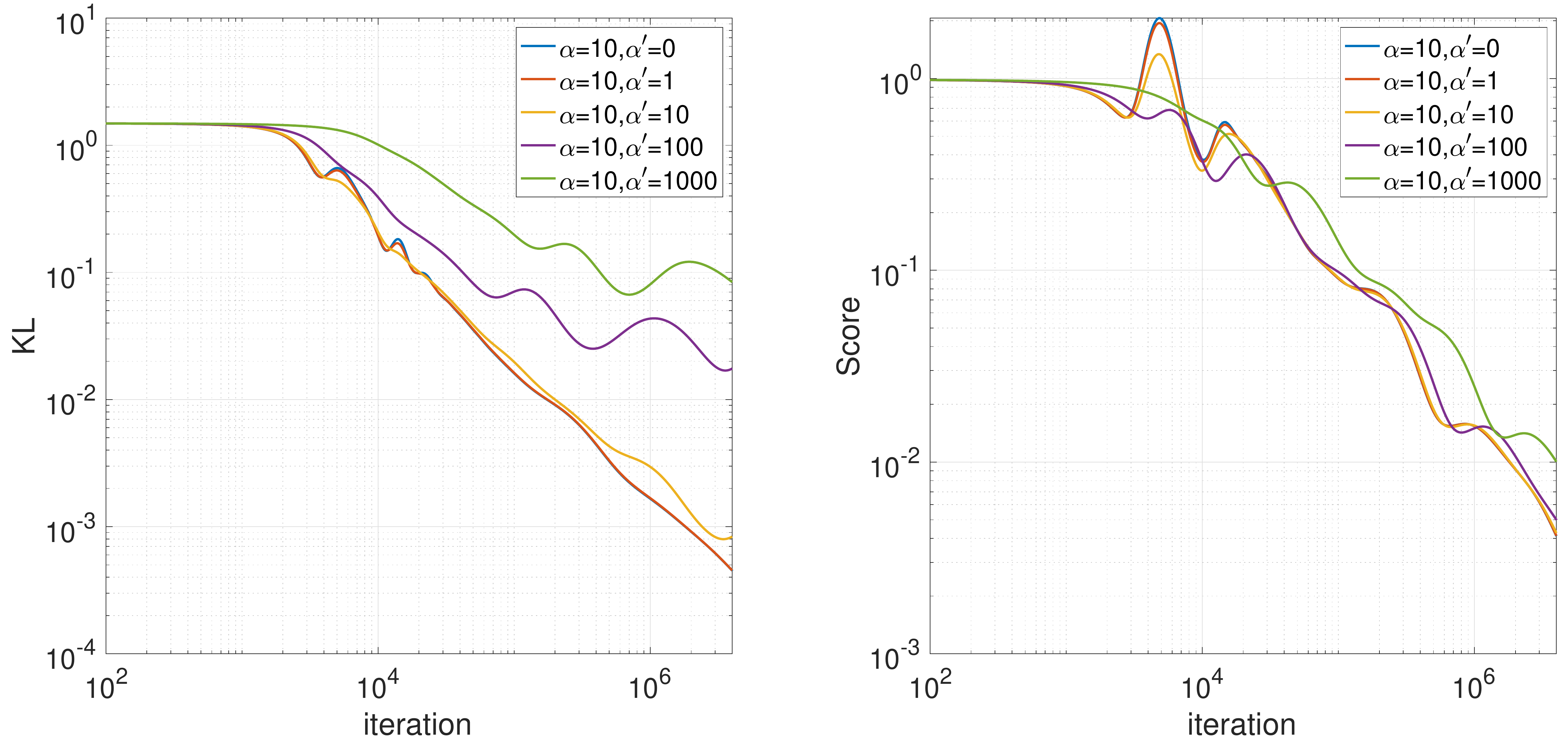}
    \end{minipage}
    \begin{minipage}[c]{5.5cm}
    \caption{Evolution of the KL and the SM for the measure PDE with different intensities of particle reinjection. For all the curves, $\alpha = 10$ and $\alpha'$ is the parameter in the reinjection term $-\alpha'(\nu_t - \nu_n)$. Times are rescaled by $\max\{\alpha,1\}$.}
    \label{fig:1d_restarts}
    \end{minipage}
\end{figure*}

\begin{figure*}
    \centering
    \includegraphics[width=0.7\textwidth]{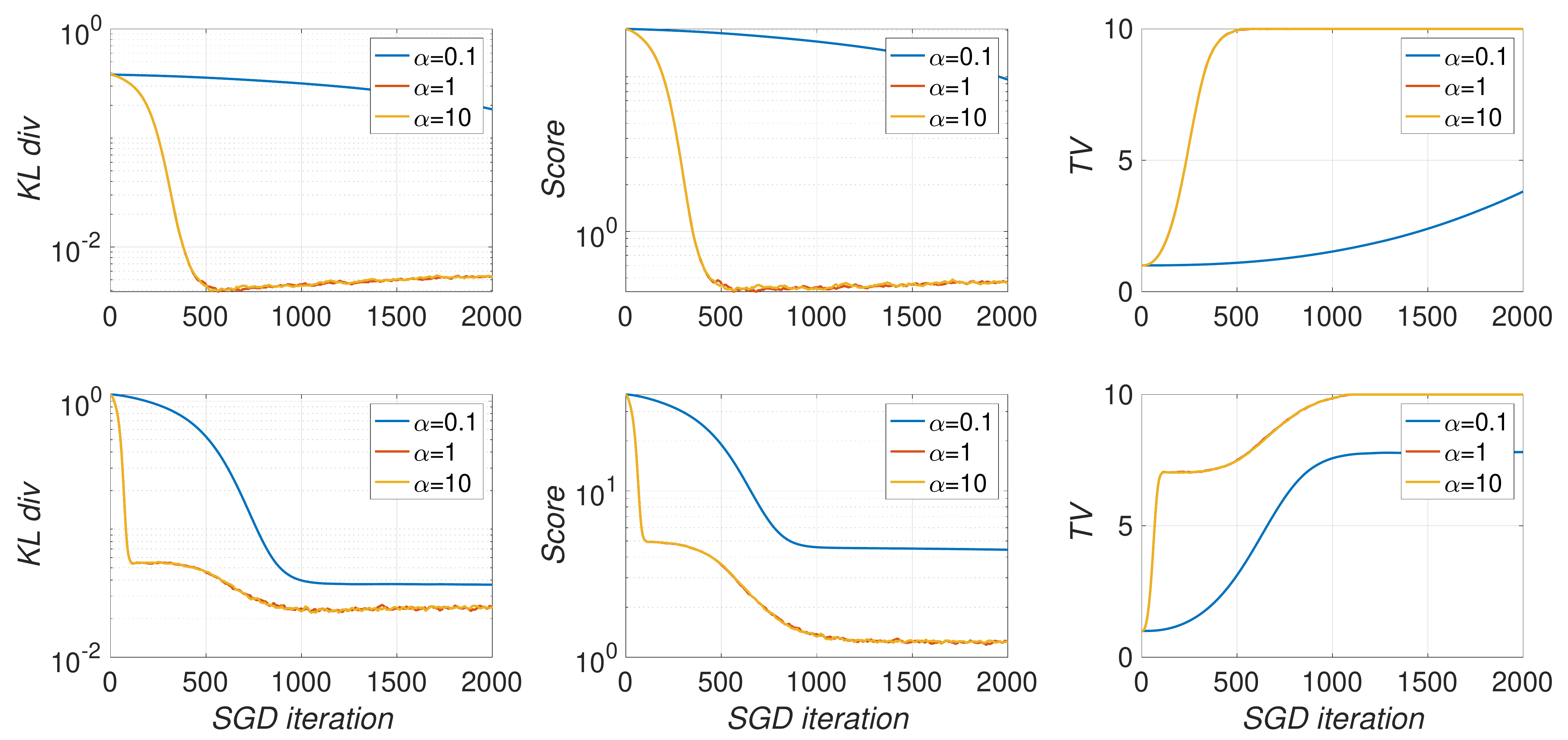}
    \caption{(Top) The evolution of the KL divergence, the score matching metric and the TV norm of the trained measure (i.e., the $\mathcal{F}_1$ norm) during training for Algorithm~\ref{alg:implicit_ebm_f1} with $\mathcal{X}=\mathbb{S}^{14}$, $m = 64$, $p_R=0$, $s = 0.02$, $n = 10^5$, $N = 2 \cdot 10^5$ and $\alpha = 0.1,1,10$, showing a speedup by a factor about 10-20 of the latter over the former. The angle between the two teacher neurons is 1.37 rad (monomodal distribution). (Bottom) Same experiments with an angle of 2.87 rad (bimodal distribution). 
    }
    \label{fig:primal_vs_dual}
\end{figure*}

\textbf{Results.} \autoref{fig:1d_evolution} and \autoref{fig:1d_restarts} show the evolution of the KL divergence and SM metric between the target measure $\nu_p$ and learned measure over training. Two time parametrizations are shown: the small plots are in the time of the PDEs \eqref{eq:coupled_dynamics_f1}-\eqref{eq:coupled_dynamics_main}, while in the large plots the time of each curve is rescaled by multiplying it by $\max\{\alpha, 1\}$. The reason behind the rescaling is that computationally the timestep needs to be proportional to $\min\{\alpha, 1\}$ to avoid numerical instabilities at the faster timescale. Thus, for algorithmic purposes the appropriate comparison between convergence speeds for different values of $\alpha$ is through the curves with rescaled time. Two main observations arise from \autoref{fig:1d_evolution} and \autoref{fig:1d_restarts}:
\vspace{-5pt}
\begin{itemize}
    \item \textbf{The best choice is $\alpha \approx 1$}: Looking at \autoref{fig:1d_evolution}, when $\alpha \ll 1$ the rescaled time curves for both KL and SM decrease slower. When $\alpha \gg 1$, the rescaled time curves decrease roughly at the same rate regardless of the specific value of $\alpha$, but larger oscillations appear the larger $\alpha$ is taken; the dynamics is more unstable.
    \vspace{-2pt}
    \item \textbf{Particle restarts hurt performance (in KL):} \autoref{fig:1d_restarts} shows that when the term $-\alpha'(\nu_t - \nu_n)$ is included in the dynamics for $\nu_t$, the convergence in KL is slower the larger $\alpha'$ is: the best choice is $\alpha' = 0$. Recall that $\alpha=0$ corresponds to maximum likelihood, while $\alpha, \alpha' \gg 1$ is equivalent to score matching.
    Noticeably, the SM curves decrease at roughly the same rate for all values of $\alpha$. This phenomenon is explained because the SM metric generally is weaker than the KL divergence.
\end{itemize}

\subsection{High-dimensional experiments for Alg.~\ref{alg:implicit_ebm_f1}}
\textbf{Setup.} To illustrate Alg.~\ref{alg:implicit_ebm_f1} in a higher-dimensional setting, we perform numerical experiments  on simple synthetic datasets generated by teacher models with energy $f^*(x) = \frac{1}{J} \sum_{j=1}^{J} w^*_j \sigma(\langle \theta^*_j, x \rangle)$,
with~$\theta^*_j \in \mathbb{S}^d$ for all~$j$.
The training is performed using Alg.~\ref{alg:implicit_ebm_f1} with the added detail that both the features $\theta^{(j)}_t$ and the particles $X^{(i)}_t$ are constrained to remain on the sphere by adding a projection step in the update of their positions. The code, figures, and videos on the dynamics can be found in the supplementary material. In the main text we consider two planted teacher neurons ($J = 2$) with negative output weights $w_1^* = w_2^* = -10$ in dimension~$d = 14$ and $m=64$ neurons for the student model, but we include additional experiments and videos in \autoref{sec:additional_experiments} and supplementary material.
We study setups with two different choices of angles between the teacher neurons, which showcase different behaviors:
\vspace{-6pt}
\begin{itemize}[leftmargin=3mm]
\item Teacher neurons $\theta^*_1, \theta^*_2$ forming an angle of 2.87 rad ($\approx$ 164 degrees),
and output weights $w_1^* = w_2^* = -10$. The teacher neurons are almost in opposite directions, and the resulting target distribution is bimodal, as the energy has two local minimizers around $\theta^*_1$ and $\theta^*_2$ (see \autoref{fig:projections_14d}).
\vspace{-3pt}
\item Teacher neurons $\theta^*_1, \theta^*_2$ forming an angle of 1.37 rad ($\approx$ 78 degrees). The teacher neurons are almost orthogonal and the resulting target distribution is monomodal; indeed, when the angle is less than~$\pi/2$, the target energy has a unique minimizer at the geodesic average between $\theta^*_1$ and $\theta^*_2$ (see \autoref{fig:projections_14d} in \autoref{sec:additional_experiments}).
\end{itemize} 
\vspace{-5pt}
\textbf{Monitoring convergence.} To monitor convergence we use a testing set of $n_*$ data points $\{x^*_i\}_{i=1}^{n_*}$ sampled from the teacher distribution: we estimate the KL divergence from the student to the teacher via $\log (\frac{1}{n^*} \sum_{i=1}^{n^*} \exp(-\beta f_t(x^*_i)+\beta f^*(x^*_i)))+\frac{1}{n^*} \sum_{i=1}^{n^*}(f_t(x^*_i)+\beta f^*(x^*_i))$
where $f_t(x) = \frac1m\sum_{j=1}^m w_t^{(j)} \sigma(\langle \theta^{(j)}_t, x \rangle)$. Similarly, for the score matching objective we use the estimate $\frac{1}{n^*} \sum_{i=1}^{n^*}\left|\nabla_x f_t(x^*_i)-\nabla_x f^*(x^*_i)\right|^2$. 

\textbf{Results.} We defer the empirical study of tuning the restart probability to \autoref{sec:additional_experiments}, and in this section focus on 
the convergence properties of Alg.~\ref{alg:implicit_ebm_f1} in the regimes $\alpha \ll 1, \alpha \gg 1$, $\alpha = 1$.
To obtain a principled comparison of the three settings where numerical errors do not blow up, we set $s$ to be the stepsize for the fastest process (particle evolution for the primal, neuron evolution for the dual), and $\min(\alpha,1) s$ the stepsize for the slow process. The results are shown in 
\autoref{fig:primal_vs_dual} for the two angle configurations between teacher neurons. We observe that for $\alpha = 1$ and $\alpha \gg 1$, the KL and SM metrics decrease much faster than for $\alpha \ll 1$. Interestingly, the decrease of the performance metrics seems to stall as soon as the hard $\mathcal{F}_1$-norm threshold is reached. 

\section{Discussion and outlook} \label{sec:discussion}
In this work we leverage a Fenchel duality result to recast the maximum likelihood loss for $\mathcal{F}_1$-EBMs into a min-max problem on probability measures over the sample space. 
This duality result paves the way for learning EBMs by training the energy parameters and the samples on simultaneous timescales via stochastic GDA.
We perform PDE simulations for a low dimensional example which suggests that similar timescales have the fastest convergence, and perform higher dimensional experiments as well.


It would be interesting to 
test Alg.~\ref{alg:implicit_ebm_f1} using deeper neural architectures and see if simultaneous timescales is the best choice as well: while the duality analysis is more complicated in this case, the scheme itself can be straightforwardly generalized to deep networks.

\section*{Acknowledgements}

CDE acknowledges partial support by the ``la Caixa'' Foundation (ID 100010434), under the agreement LCF/BQ/AA18/11680094.
EVE acknowledges partial support from the National Science Foundation (NSF) Materials Research Science and Engineering Center Program grant DMR-1420073, NSF DMS- 1522767, and the Vannevar Bush Faculty Fellowship. JB acknowledges partial support from the Alfred P. Sloan Foundation, NSF RI-1816753, NSF CAREER CIF 1845360, NSF CHS-1901091 and Samsung Electronics.

\bibliography{biblio}

\newpage

\appendix

\tableofcontents

\starttocentries

\section{Preliminaries on Fenchel duality and maxent models}
\label{sec:fenchel}


The basic theoretical tool of the present paper is Fenchel duality, whose main applications in machine learning are maximum entropy or \emph{maxent} models. We provide a brief description of such models to put in context the results in \autoref{sec:general_duality}, which are related. 

\paragraph{Fenchel duality.} If $X$ is a Banach space and $X^{*}$ is its dual space, the convex or Fenchel conjugate of $f : X \to \R$ is the function $f^{*} : X^{*} \to \R$ defined as $f^{*}(x^{*}) = \sup_{x \in X} \langle x^{*}, x \rangle - f(x)$. The Fenchel strong duality theorem (see \autoref{thm:fenchel}) states that under certain conditions, if $X,Y$ are Banach spaces, $f : X \to \R \cup \{+\infty\}, \ g : Y \to \R \cup \{+\infty\}$ are convex functions, and $A : X \to Y$ is a bounded linear map, then 
\begin{align} \label{eq:Fenchel_strong}
    \inf_{x\in X}\{f(x)+g(Ax)\} = \sup_{y^{*}\in Y^{*}}\{-f^{*}(A^{*}y^{*})-g^{*}(-y^{*})\}.
\end{align}

\paragraph{Entropy and log-partition as convex conjugates.} For simplicity, let $\mathcal{Y}$  be a finite set and let $\tau_{\mathcal{Y}} \in \mathcal{P}(\mathcal{Y}) \subseteq \R^{|\mathcal{Y}|}$ be a base distribution on $\mathcal{Y}$. Crucially, the KL divergence or relative entropy $D_{\text{KL}}(\nu \| \tau_{\mathcal{Y}})$ is a convex function of $\nu$ and its convex conjugate is the log-partition function $v \in \R^{|\mathcal{Y}|} \mapsto \log ( \sum_{y \in \mathcal{Y}} \tau_{\mathcal{Y}}(y) \exp(v(y)) )$.
The functional equivalent of this convex conjugate pair is key both for maximum entropy models, introduced below, as well as for the results of \autoref{sec:general_duality}.

\paragraph{Maximum entropy (maxent) models.} 

Let $\Phi : \mathcal{Y} \to \R^q$ be a feature mapping, and $\nu_n = \frac{1}{n} \sum_{i=1}^n \delta_{y_i} \in \mathcal{P}(\mathcal{Y})$ an empirical measure. One is interested in a statistical model that is `maximally non committal', i.e. as close as possible to the base measure in KL divergence, given that its feature moments are not too far from those of $\nu_n$.
This rationale leads to the $l^{\infty}$ maxent problem (Ch. 12, \cite{mohri2012foundations}), which is
\begin{align}
\begin{split} \label{eq:maxent_primal}
    \min_{\nu \in \mathcal{P}(\mathcal{Y})} \ D_{\text{KL}}(\nu \| \tau_{\mathcal{Y}}) \quad
    \text{such that} \quad \|\mathbb{E}_{\nu}[\Phi(y)]-\mathbb{E}_{\nu_n}[\Phi(y)]\|_{\infty} \leq \lambda.
\end{split}    
\end{align}
Let $\nu_{w} \in \mathcal{P}(\mathcal{Y})$ be the distribution with density $\frac{d\nu_{w}}{d\tau_{\mathcal{Y}}} \propto \exp(-\langle w, \Phi(y) \rangle)$.
One can apply Fenchel strong duality (equation \eqref{eq:Fenchel_strong}) on the problem \eqref{eq:maxent_primal}, by taking the KL divergence as the function $f$ and the indicator function of the constraint set as $g \circ A$. Using that the log-partition is the convex conjugate, the dual of \eqref{eq:maxent_primal} is
\begin{align} \label{eq:maxent_dual}
    \max_{w \in \R^q} -\frac{1}{n} \sum_{i=1}^n \langle w, \Phi(y_i) \rangle - \log \left(\sum_{y \in \mathcal{Y}} \exp (-\langle w, \Phi(y) \rangle) \right) - \lambda \|w\|_{1} = \frac{1}{n} \sum_{i=1}^n \log\left( \frac{d\nu_{w}}{d\tau_{\mathcal{Y}}}(y_i) \right) - \lambda \|w\|_{1},
\end{align}
Strong duality holds and $\nu_{w^{\star}}$ is a solution of \eqref{eq:maxent_primal} when $w^{\star}$ is a solution of \eqref{eq:maxent_dual}.
That is, solving an entropy maximization problem with an $\ell^{\infty}$ constraint on some generalized moments is equivalent to solving a maximum likelihood for the exponential family problem under $\ell^1$ regularization. If in \eqref{eq:maxent_dual} we replace the norm in the constraint by the $\ell^2$ norm, the corresponding dual problem involves $\ell^2$ norm of $w$ instead. The $\ell^{\infty}$-$\ell^1$ maxent problems \eqref{eq:maxent_primal}-\eqref{eq:maxent_dual} are to be compared with problems \eqref{eq:general_primal_problem_f1}-\eqref{eq:general_dual_problem_f1} in the next section, while the $\ell^2$ maxent problems should be contrasted with problems \eqref{eq:general_primal_problem}-\eqref{eq:general_dual_problem}.

\section{General duality results} \label{sec:general_duality}

In this section we state Fenchel duality results between KL regularized regression problems over probability measures with problems that are formally equivalent maximum likelihood estimation. On the one hand, in \autoref{thm:duality_general} the metric used for regression is the $L^2$ distance (not squared, unlike in least squares regression) and the corresponding dual problem is over a properly defined $L^2$ space. \autoref{thm:duality_general} is the basis for the dual formulation of $\mathcal{F}_2$-EBMs (\autoref{sec:implicit_EBM}) and also for a formulation of neural network regression via sampling (\autoref{sec:two_layer_nn_sampling}). These two topics are deferred to the appendices.
On the other hand, in \autoref{thm:duality_general_f1} the regression metric is the $L^\infty$ distance, and the corresponding dual problem is over a space of Radon measures. 
\autoref{thm:duality_general_f1} is the theoretical foundation for the dual formulation of $\mathcal{F}_1$-EBMs in \autoref{sec:implicit_mle_f1}. The proofs are in \autoref{sec:proofs_duality}.

Let $\mathcal{Y} \subseteq \R^{d_1}$ and let $\tau_{\mathcal{Y}} \in \mathcal{P}(\mathcal{Y})$ be a fixed base probability measure over $\mathcal{Y}$ with full support. Let $\mathcal{Z} \subseteq \R^{d_2}$ and let $\tau_{\mathcal{Z}} \in \mathcal{P}(\mathcal{Z})$ be a fixed base probability measure over $\mathcal{Z}$ with full support. 
Denote $L^2(\mathcal{Z}) = \{ f : \mathcal{Z} \to \R \ | \ \int_{\mathcal{Z}} f(z)^2 
\ d\tau_{\mathcal{Z}}(z) < +\infty \}$. Let $g \in L^2(\mathcal{Z})$ be a fixed function. 
\begin{assumption} \label{ass:phi}
Let $\phi : \mathcal{Y} \times \mathcal{Z} \to \R$ be a continuous function such that either $\mathcal{Y}$ is compact or (i) for any fixed $z \in \mathcal{Z}$, $\phi(y, z) = O(\xi_1(y))$ for some strictly positive $\xi_1 : \mathcal{Y} \to \R$, and (ii) 
$\xi_1(y) + \log(\xi_1(y)) = o\left(-\log\left(\frac{d\tau_{\mathcal{Y}}}{d\lambda}(y) \right) - (d_1+\epsilon) \log \|y\|_2 \right) \text{ as } \|y\|_2 \to +\infty$ for some $\epsilon > 0$,
and
(iii) the function $\xi_2(y) := (\int_{\mathcal{Z}} \phi(y, z)^2 
\ d\tau_{\mathcal{Z}}(z))^{1/2}$ fulfills $\sup_{y \in \mathcal{Y}} |\xi_2(y)|/\xi_1(y) < \infty$.
\end{assumption}

\autoref{ass:phi} imposes that either $\mathcal{Y}$ is compact, or the map $\phi$ has a well-behaved growth in a certain sense, not very stringently. Note that (ii) is merely to ensure that $\xi_1(y)$ has finite expectation under the base measure $\tau_{\mathcal{Y}}$.   

\subsection{KL-regularized $L^2$ regression} Consider the two problems 
\begin{align}
    \begin{split} \label{eq:general_primal_problem}
    \min_{\nu \in \mathcal{P}(\mathcal{Y})} &\beta^{-1} D_{\text{KL}}(\nu||\tau_{\mathcal{Y}}) + \left( \int_{\mathcal{Z}} \left( \int_{\mathcal{Y}} \phi(y, z) \ d\nu(y) - g(z) \right)^2 
    \ d\tau_{\mathcal{Z}}(z) \right)^{1/2},
\end{split}
\end{align}
and
\begin{align}
\begin{split} \label{eq:general_dual_problem}
    &\max_{\substack{h \in L^2(\mathcal{Z}) \\ \|h\|_{L^2} \leq 1}} - \int_{\mathcal{Z}} g(z) h(z)\ d\tau_{\mathcal{Z}}(z) - \frac{1}{\beta} \log\left(\int_{\mathcal{Y}} \exp \left(-\beta \int_{\mathcal{Z}} \phi(y, z) h(z)\ d\tau_{\mathcal{Z}}(z) \right) d\tau_{\mathcal{Y}}(y) \right).
\end{split}
\end{align}
\begin{theorem}
\label{thm:duality_general}
The problems \eqref{eq:general_primal_problem} and \eqref{eq:general_dual_problem} are convex. Suppose that \autoref{ass:phi} holds. Then problem \eqref{eq:general_dual_problem} is the Fenchel dual of problem \eqref{eq:general_primal_problem}, and strong duality holds. Moreover, the solution $\nu^{\star}$ of \eqref{eq:general_primal_problem} is unique and its density satisfies
\begin{align}
    \frac{d\nu^{\star}}{d\tau_{\mathcal{Y}}}(y) = \frac{1}{Z_{\beta}}\exp \left( - \beta \int_{\mathcal{Z}} \phi(y, z) \ h^{\star}(z) d\tau_{\mathcal{Z}}(z) \right),
\end{align}
where $h^{\star}$ is a solution of \eqref{eq:general_dual_problem} and $Z_{\beta}$ is a normalization constant.
\end{theorem}
The proof of this result is in \autoref{sec:proofs_duality}. Note that in the problem \eqref{eq:general_dual_problem} we are implicitly optimizing over an RKHS ball, which makes \autoref{thm:duality_general} close to the results from \cite{dai2019kernel}.

A relevant problem that is very similar to \eqref{eq:general_primal_problem} is:
\begin{align}
\begin{split} \label{eq:general_primal_problem2}
    \min_{\nu \in \mathcal{P}(\mathcal{Y})} &\tilde{\beta}^{-1} D_{KL}(\nu||\tau_{\mathcal{Y}}) + \int_{\mathcal{Z}} \left( \int_{\mathcal{Y}} \phi(y, z) \ d\nu(y) - g(z) \right)^2 
    \ d\tau_{\mathcal{Z}}(z).
\end{split}
\end{align}
The following result links this problem with problem \eqref{eq:general_primal_problem}.
\begin{proposition} \label{lem:primal}
Problems \eqref{eq:general_primal_problem} and \eqref{eq:general_primal_problem2} are equivalent in the following sense: if $\nu_1^{\star}$ is a solution of \eqref{eq:general_primal_problem} for $\beta$, then it is also a solution of \eqref{eq:general_primal_problem2} for 
\begin{align}
\tilde{\beta} = \beta \left( 4 \int_{\mathcal{Z}} \left( \int_{\mathcal{Y}} \phi(y, z) \ d\nu_1^{\star}(y) - g(z) \right)^2 d\tau_{\mathcal{Z}}(z) \right)^{-1/2}
\end{align}
provided that the left-most factor is non-zero.
Conversely, if $\nu_2^{\star}$ is a solution of \eqref{eq:general_primal_problem2} for $\tilde{\beta}$, then it is also a solution of \eqref{eq:general_primal_problem} for 
\begin{align}
\beta = \tilde{\beta} \left( 4 \int_{\mathcal{Z}} \left( \int_{\mathcal{Y}} \phi(y, z) \ d\nu_2^{\star}(y) - g(z) \right)^2 d\tau_{\mathcal{Z}}(z) \right)^{1/2}.
\end{align}
\end{proposition}
And the next lemma provides additional insights into how the problems \eqref{eq:general_primal_problem} and \eqref{eq:general_primal_problem2} differ in the planted case.
\begin{proposition} \label{prop:planted}
    Suppose $g : \mathcal{Z} \to \R$ is of the form $g(z) = \int_{\mathcal{Y}} \phi(y, z) \ d\nu_p(y)$ for some $\nu_p \in \mathcal{P}(\mathcal{Z})$, and assume that the (negated) log-density $E(y) = -\log(\frac{d\nu_p}{d\tau_{\mathcal{Y}}}(y))$ belongs to the RKHS ball $B_{\mathcal{F}_2}(\beta_0)$.
    
    (a) On the one hand, when $\beta \geq \beta_0$ the solution $\nu^{\star}_1$ of \eqref{eq:general_primal_problem} is equal to $\nu_p$. That is, there is recovery of the planted target measure and consequently $\int_{\mathcal{Z}} \left( \int_{\mathcal{Y}} \phi(y, z)^2 \ d\nu_1^{\star}(y) - g(z) \right)^2 
    \ d\tau_{\mathcal{Z}}(z) = 0$.
    
    (b) On the other hand, for all choices of $\tilde{\beta}$ finite if $\nu^{\star}_2$ is the solution of \eqref{eq:general_primal_problem2}, the unregularized regression loss at $\nu^{\star}_2$ is \textit{not} zero: $\int_{\mathcal{Z}} \left( \int_{\mathcal{Y}} \phi(y, z)^2 \ d\nu_2^{\star}(y) - g(z) \right)^2 
    \ d\tau_{\mathcal{Z}}(z) > 0$. Hence, $\nu^{\star}_2 \neq \nu_p$ and there is no recovery.
\end{proposition}

\subsection{KL-regularized $L^{\infty}$ regression} \label{sec:kl_regularized_Linfty}
Consider the two problems
\begin{align}
    \begin{split} \label{eq:general_primal_problem_f1}
    \min_{\nu \in \mathcal{P}(\mathcal{Y})} &\beta^{-1} D_{\text{KL}}(\nu||\tau_{\mathcal{Y}}) + \left\| \int_{\mathcal{Y}} \phi(y, \cdot) \ d\nu(y) - g(\cdot) \right\|_{L^{\infty}},
\end{split}
\end{align}
and
\begin{align}
\begin{split} \label{eq:general_dual_problem_f1}
    &\max_{\substack{\gamma \in \mathcal{M}(\mathcal{Z}) \\ \|\gamma\|_{\text{TV}} \leq 1}} - \int_{\mathcal{Z}} g(z) \ d\gamma(z) - \frac{1}{\beta} \log\left(\int_{\mathcal{Y}} \exp \left(-\beta \int_{\mathcal{Z}} \phi(y, z) \ d\gamma(z) \right) d\tau_{\mathcal{Y}}(y) \right).
\end{split}
\end{align}
\begin{theorem}
\label{thm:duality_general_f1}
The problems \eqref{eq:general_primal_problem_f1} and \eqref{eq:general_dual_problem_f1} are convex. Suppose that \autoref{ass:phi} holds and also that (i) there exists $K > 0$ such that $\sup_{z \in \mathcal{Z}} \sup_{y \in \mathcal{Y}} \phi(y,z)/\xi_1(y) < K$, and (ii) $g(\cdot) \in C_b(\mathcal{Z})$. Then problem \eqref{eq:general_dual_problem_f1} is the Fenchel dual of  problem \eqref{eq:general_primal_problem_f1}, and strong duality holds. Moreover, the solution $\nu^{\star}$ of \eqref{eq:general_primal_problem_f1} is unique and its density satisfies
\begin{align}
    \frac{d\nu^{\star}}{d\tau_{\mathcal{Y}}}(y) = \frac{1}{Z_{\beta}}\exp \left( - \beta \int_{\mathcal{Z}} \phi(y, z) \ d\gamma^{\star}(z) \right),
\end{align}
where $\gamma^{\star}$ is a solution of \eqref{eq:general_dual_problem_f1} and $Z_{\beta}$ is a normalization constant.
\end{theorem}

At this point, we remark the similarity between the maxent problems \eqref{eq:maxent_primal}-\eqref{eq:maxent_dual}  and problems \eqref{eq:general_primal_problem_f1}-\eqref{eq:general_dual_problem_f1}. The former are stated in finite dimension and involve a constraint in the minimization problem and a penalization term in the maximization problem; the latter hold in infinite-dimensional settings and involve a a penalization term in the minimization problem and a constraint in the maximization problem.

In \autoref{thm:duality_general} and \autoref{thm:duality_general_f1}, the improvement over \cite{domingoenrich2021energybased} is that $\mathcal{Y}$, $\mathcal{Z}$ may be taken unbounded, which makes the results more general and closer to practice. This adds certain technical difficulties in constructing the Banach spaces needed to apply Fenchel duality (proofs in \autoref{sec:proofs_duality}). 


\section{Dual $\mathcal{F}_2$-EBM training as KL-regularized MMD optimization}
\label{sec:implicit_EBM}

\paragraph{Kernel regime: the space $\mathcal{F}_2$.} Let $\mathcal{X} \subseteq \R^{d_1}$,  $\Theta \subseteq \R^{d_2}$, $\phi : \mathcal{X} \times \Theta \to \R$, and $\tau_{\Theta}$ be a fixed base probability measure over $\Theta$. We define $\mathcal{F}_2$ as the reproducing kernel Hilbert space (RKHS) of functions $f : \mathcal{X} \rightarrow \R$ such that for some $h \in L^2(\Theta,\tau_\Theta)$,
we have that, for all $x \in \mathcal{X}$,  $f(x) = \int_{\Theta} \phi(x, \theta)  h(\theta) d\tau_\Theta(\theta)$. The RKHS norm of $\mathcal{F}_2$ is defined as $\|f\|_{\mathcal{F}_2} = \inf \left\{ \|h\|_{L^2(\Theta)} \ | \ f(\cdot) = \int_{\Theta} \phi(\cdot, \theta) \ h(\theta) \ d\tau_{\Theta}(\theta) \right\}$ where $\|h\|^2_{L^2(\Theta)} := \int_{\Theta} |h(\theta)|^2 
\ d\tau_{\Theta}(\theta)$ (c.f. \cite{bach2017breaking}). As an RKHS, the kernel of $\mathcal{F}_2$ is 
\begin{align} \label{eq:kernel_random}
k(x,y) = \int_{\Theta} \phi(x, \theta) \phi(x, \theta) 
\ d\tau_{\Theta}(\theta).
\end{align}
Kernels of this form where popularized in machine learning under the name of random feature kernels~\citep{rahimi2008random}, and they admit closed form expressions in the case $\phi(x, \theta) = \sigma(\langle x, \theta \rangle)$ for several choices of the activation $\sigma$ and base measure $\tau_{\Theta}$ \citep{leroux2007continuous, cho2009kernel,bach2017breaking}. Remark that since $\|h\|_{L^1(\Theta)} = \int_{\Theta} |h(\theta)| 
\ d\tau_{\Theta}(\theta) \leq (\int_{\Theta} |h(\theta)|^2 
\ d\tau_{\Theta}(\theta))^{1/2} = \|h\|_{L^2(\Theta)}$ by the Cauchy-Schwarz inequality, we have $\mathcal{F}_2 \subset \mathcal{F}_1$: in particular finite-width neural networks belong to $\mathcal{F}_1$ but not to $\mathcal{F}_2$ \citep{bach2017breaking}. 

\paragraph{$\mathcal{F}_2$-EBMs.} Let $\mathcal{X} \subseteq \R^{d_1}$ and let $\tau_{\mathcal{X}} \in \mathcal{P}(\mathcal{X})$ be a fixed base measure.
Assume that we have access to i.i.d. samples $\{x_i\}_{i=1}^n$ from an arbitrary target $\nu_p \in \mathcal{P}(\mathcal{X})$, and let $\nu_n = \frac{1}{n} \sum_{i=1}^n \delta_{x_i}$ be the empirical distribution. For any $f : \mathcal{X} \to \R$, denote by $\nu_f$ the Gibbs measure of energy $f$ and base measure $\tau_{\mathcal{X}}$, i.e. $\frac{d\nu_f}{d\tau_{\mathcal{X}}}(x) = \exp(-f(x))/Z$. 
Let $k$ be the corresponding kernel defined in \eqref{eq:kernel_random}, and denote $L^2(\Theta) = \{ f : \Theta \to \R \ | \ \int_{\Theta} f(\theta)^2 
\ d\tau_{\Theta}(\theta) < + \infty\}$.
We consider the problem of training an energy-based model with energies in the RKHS ball $\mathcal{B}_{\mathcal{F}_2}(\beta)$ of radius $\beta$ via maximum likelihood, i.e.
\begin{align}
\begin{split} \label{eq:hat_mu2}
    f^{\star} &= \argmin_{f \in \mathcal{B}_{\mathcal{F}_2}(\beta)} H(\nu_n, \nu_{f}) = \argmin_{f \in \mathcal{B}_{\mathcal{F}_2}(\beta)} -\frac{1}{n} \sum_{i=1}^n \log \left(\frac{d\nu_{f}}{d\tau_{\mathcal{X}}}(x_i) \right)\\ &= \argmin_{f \in \mathcal{B}_{\mathcal{F}_2}(\beta)} \frac{1}{n} \sum_{i=1}^n f(x_i) + \log \left( \int_{\mathcal{X}} e^{-f(x)} d\tau_{\mathcal{X}}(x) \right),
\end{split}
\end{align}
where $H(\nu, \nu') = -\int \log(\frac{d\nu'}{d\tau_{\mathcal{X}}}) d\nu$ denotes the cross-entropy between two measures. Remark that an arbitrary element $f$ of the RKHS $\mathcal{F}_2$ admits a representation as \citep{bach2017breaking}
\begin{align} \label{eq:f_h_change}
    f(x) = \int_{\Theta} \phi(x, \theta) h(\theta) \ d\tau_{\Theta}(\theta), \quad \text{where } h \in L^2(\Theta). 
\end{align}
Thus, the problem \eqref{eq:hat_mu2} can be restated as
\begin{align} \label{eq:EBM_dual}
    \argmin_{\substack{h \in L^2(\Theta) \\ \|h\|_{L^2} \leq 1}} \frac{1}{n} \sum_{i=1}^n \int_{\Theta} \phi(x_i, \theta) h(\theta) \ d\tau_{\Theta}(\theta) + \frac{1}{\beta} \log \left( \int_{\mathcal{X}} \exp \left(- \beta \int_{\Theta} \phi(x, \theta) h(\theta) \ d\tau_{\Theta}(\theta) \right) d\tau_{\mathcal{X}}(x) \right)
\end{align}
Problem \eqref{eq:EBM_dual} can be identified with problem \eqref{eq:general_dual_problem} up to a sign flip by setting $\mathcal{Y} = \mathcal{X}$, $\mathcal{Z} = \Theta$ and $g(\theta) = \frac{1}{n} \sum_{i=1}^n \int_{\Theta} \phi(x_i, \theta) = \int_{\Theta} \phi(x, \theta) \ d\nu_n(x)$. Hence, if \autoref{ass:phi} holds, by \autoref{thm:duality_general} the problem \eqref{eq:EBM_dual} is the Fenchel dual of 
\begin{align}
    \begin{split} \label{eq:EBM_primal}
    \min_{\nu \in \mathcal{P}(\mathcal{X})} &\beta^{-1} D_{KL}(\nu||\tau_{\mathcal{X}}) 
    + \text{MMD}_{k}(\nu, \nu_n),
\end{split}
\end{align}
where $\text{MMD}_{k}(\nu, \nu_n) = \left( \int_{\mathcal{X} \times \mathcal{X}} k(x, x') \ d(\nu - \nu_n)(x) \ d(\nu - \nu_n)(x') \right)^{1/2}$ is known as the maximum mean discrepancy (MMD) for the kernel $k$ \citep{gretton2012akernel}. See \autoref{lem:link} in \autoref{sec:proofs_implicit_f2} for the derivation.
And the analog of problem \eqref{eq:general_primal_problem2} is
\begin{align}
\begin{split} \label{eq:EBM_primal2}
    \min_{\nu \in \mathcal{P}(\mathcal{X})} &\tilde{\beta}^{-1} D_{KL}(\nu||\tau_{\mathcal{X}}) 
    + \text{MMD}_{k}^2(\nu, \nu_n)
\end{split}
\end{align}
The following corollary of \autoref{thm:duality_general} and \autoref{lem:primal} describes precisely the link between the solutions of problems \eqref{eq:EBM_primal} and \eqref{eq:EBM_primal2} and the solution of the maximum likelihood problem \eqref{eq:EBM_dual}.
\begin{corollary}
\label{cor:EBM}
Suppose that \autoref{ass:phi} holds for $\mathcal{Y} = \mathcal{X}$ and $\mathcal{Z} = \Theta$. The solution $\nu_1^{\star}$ of \eqref{eq:EBM_primal} is unique and of the form 
\begin{align}
    \frac{d\nu_1^{\star}}{d\tau_{\mathcal{X}}}(x) = \frac{1}{Z_{\beta}}\exp \left( - \beta \int_{\Theta} \phi(x, \theta) \ h^{\star}(\theta) d\tau_{\Theta}(\theta) \right),
\end{align}
where $h^{\star}$ is a solution of \eqref{eq:EBM_dual} and $Z_{\beta}$ is a normalization constant. Additionally, the unique solution $\nu_2^{\star}$ of \eqref{eq:EBM_primal2} is equal to the solution $\nu_1^{\star}$ of \eqref{eq:EBM_primal} when $\beta = 2 \text{MMD}_{k}(\nu_2^{\star}, \nu_n) \tilde{\beta}
$.
\end{corollary}
Undoing the change of variables \eqref{eq:f_h_change}, we see that $f^{\star}(x) = \beta \int \phi(x, \theta) \ h^{\star}(\theta) d\tau_{\Theta}(\theta)$ is the energy in $\mathcal{B}_{\mathcal{H}}(\beta)$ that maximizes the likelihood. Hence, although problems \eqref{eq:EBM_primal}-\eqref{eq:EBM_primal2} are implicit in the sense that they do not involve energy functions, the solutions $\nu_1^{\star}$ and $\nu_2^{\star}$ coincide with the Gibbs measure $\nu_{f^{\star}}$ that is obtained through maximum likelihood EBM training. 

Consequently, solving \eqref{eq:EBM_primal} or \eqref{eq:EBM_primal2} provides an implicit way to train maximum likelihood $\mathcal{F}_2$-EBMs. Maximum likelihood $\mathcal{F}_2$-EBMs are classically trained via gradient descent on a parametrized form of the energy, via either a feature discretization of \eqref{eq:EBM_dual} or a representer theorem applied on \eqref{eq:hat_mu2}. Their computational bottleneck is the gradient estimation procedure, which relies on sampling from the trained model at every step; a task that is exponentially costly in $\beta$ for non-convex energies. 

\subsection{How to train $\mathcal{F}_2$-EBMs implicitly} \label{subsec:implicit_training}
Suppose from now on that $\mathcal{X}$ is either a domain (connected open subset) of $\R^{d_1}$ or a Riemannian manifold embedded in $\R^{d_1}$, case in which differential operators are understood in the Riemannian sense. Since the objective functionals in \eqref{eq:EBM_primal} and \eqref{eq:EBM_primal2} are convex in $\nu$ (\autoref{thm:duality_general}), a natural approach to solve these problems is to approximate their Wasserstein gradient flows. Namely, \autoref{lem:wasserstein_EBM1} in \autoref{sec:proofs_implicit_f2} shows that for \eqref{eq:EBM_primal} the Wasserstein gradient flow takes the form of a McKean-Vlasov equation \citep{mckean1967aclass}:
\begin{align} \label{eq:wasserstein_EBM}
    \partial_t \nu_t = \nabla \cdot \left( \nu_t  \left(- \beta^{-1} \nabla \log \frac{d\tau_{\mathcal{X}}}{d\lambda} (x) + \frac{\int_{\mathcal{X}} \nabla_x k(x,x') \ d(\nu_t - \nu_n)(x')}{ 
    \text{MMD}_{k}(\nu_t, \nu_n)} 
    \right) \right) + \beta^{-1} \Delta \nu_t,
\end{align}
where $\lambda$ is the Lebesgue or Hausdorff measure over $\mathcal{X}$, and for \eqref{eq:EBM_primal2} it is:
\begin{align} \label{eq:wasserstein_EBM_2}
    \partial_t \nu_t = \nabla \cdot \left( \nu_t  \left(- \tilde{\beta}^{-1} \nabla \log \frac{d\tau_{\mathcal{X}}}{d\lambda} (x) + 2\int_{\mathcal{X}} \nabla_x k(x,x') \ d(\nu_t - \nu_n)(x') \right) \right) + \tilde{\beta}^{-1} \Delta \nu_t,
\end{align}
Remark the striking similarity of this equation with the ones found in \cite{rotskoff2018neural, mei2018mean}, which study McKean-Vlasov equations for overparametrized two-layer neural network training. As is customary, we approximate McKean-Vlasov equations via coupled particle systems: in the case of \eqref{eq:wasserstein_EBM},
\begin{align} \label{eq:particle_system}
    dX_t^{(i)} = \left(\beta^{-1} \nabla \log \frac{d\tau_{\mathcal{X}}}{d\lambda} (X_t^{(i)}) 
    - \frac{\int_{\mathcal{X}} \nabla_x k(X_t^{(i)},x') \ d(\nu_{t, N} - \nu_n)(x')}{\text{MMD}_{k}(\nu_{t, N}, \nu_n)}
    \right)\ dt + \sqrt{2 \beta^{-1}} \ dW_t^{(i)}
\end{align}
for $i = 1, \dots, N$, where $\nu_{t, N} = \frac{1}{N} \sum_{i=1}^N \delta_{X_t^{(i)}}$, and in the case of \eqref{eq:wasserstein_EBM_2},
\begin{align} \label{eq:particle_system2}
    dX_t^{(i)} = \left(\tilde{\beta}^{-1} \nabla \log \frac{d\tau_{\mathcal{X}}}{d\lambda} (X_t^{(i)}) 
    - 2 \int_{\mathcal{X}} \nabla_x k(X_t^{(i)},x') \ d(\nu_{t, N} - \nu_n)(x')
    \right)\ dt + \sqrt{2 \tilde{\beta}^{-1}} \ dW_t^{(i)}
\end{align}
A classical argument known as propagation of chaos \citep{sznitmantopics1991} shows that when the number of particles $N$ goes to infinity, $(\nu_{t, N})_{t \in [0,T]}$ converges weakly to the solution $(\nu_{t})_{t \in [0,T]}$ of \eqref{eq:wasserstein_EBM} for any fixed $T > 0$. Although this is only a qualitative guarantee, \cite{rotskoff2018neural, chen2020dynamical} provide quantitative central limit theorems for McKean-Vlasov equations similar to \eqref{eq:wasserstein_EBM_2}. Loosely speaking, they find that the variance is no larger than the Monte-Carlo variance one would obtain by sampling i.i.d. from the solution measure 
The Euler-Maruyama discretizations of the SDEs \eqref{eq:particle_system} and \eqref{eq:particle_system2} yield two alternative implementable algorithms for implicit EBM training.

\begin{algorithm}
\caption{Implicit $\mathcal{F}_2$-EBM training (discretization of equations \eqref{eq:particle_system}/\eqref{eq:particle_system2})}
\begin{algorithmic}
\STATE \textbf{Input:} $n$ samples $\{x_i\}_{i=1}^n$ of the target distribution, $N$ initialization samples $\{X_0^{(i)}\}_{i=1}^N$, inverse temperature $\beta$ (if \eqref{eq:particle_system}), reparametrized inverse temperature $\tilde{\beta}$ (if \eqref{eq:particle_system}). 
\FOR {$t=0,\dots,T-1$}
    \STATE \textbf{If \eqref{eq:particle_system}:} Compute the $\text{MMD}_{k}^2(\nu_{t, N}, \nu_n) = \frac{1}{n^2} \sum_{i,j=1}^N k(X_t^{(i)}, X_t^{(j)}) + \frac{1}{m^2} \sum_{i,j=1}^n k(x_i, x_j) - \frac{2}{N n} \sum_{i=1}^N \sum_{j=1}^n k(X_t^{(i)}, x_j)$
    \FOR{$i=1,\dots,N$}
        \STATE Sample $\zeta_{t}^{(i)}$ from the $d_1$-variate standard Gaussian.
        \STATE Perform Euler-Maruyama update: \STATE \textbf{If \eqref{eq:particle_system}}, 
        $X_{t+1}^{(i)} = X_{t}^{(i)} - s \int_{\mathcal{X}} \nabla_x k(X_t^{(i)},x') \ d(\nu_{t, N} - \nu_n)(x')/\text{MMD}_{k}(\nu_{t, N}, \nu_n) + s \beta^{-1} \nabla \log \frac{d\tau_{\mathcal{X}}}{d\lambda} (X_t^{(i)}) + \sqrt{2 \beta^{-1} s} \zeta_{t}^{(i)}$. 
        \textbf{If \eqref{eq:particle_system2}}, 
        $X_{t+1}^{(i)} = X_{t}^{(i)} - 2 s \int_{\mathcal{X}} \nabla_x k(X_t^{(i)},x') \ d(\nu_{t, N} - \nu_n)(x') + s \tilde{\beta}^{-1} \nabla \log \frac{d\tau_{\mathcal{X}}}{d\lambda} (X_t^{(i)}) + \sqrt{2 \tilde{\beta}^{-1} s} \zeta_{t}^{(i)}$.
    \ENDFOR
\ENDFOR
\STATE Output: samples $\{X_T^{(i)}\}_{i=1}^N$, if \eqref{eq:particle_system}, energy 
$E_T(x) := \beta \frac{\int_{\mathcal{X}} \nabla_x k(x,x') \ d(\nu_{T,N} - \nu_n)(x')}{\text{MMD}_{k}(\nu_{T,N}, \nu_n)}$,
if \eqref{eq:particle_system2}, energy $E_T(x) := 2 \tilde{\beta} \int_{\mathcal{X}} k(x,x') \ d(\nu_{T,N} - \nu_n)(x')$.
\end{algorithmic}
\label{alg:implicit_ebm}
\end{algorithm}

\subsection{Comparison with \cite{arbel2019maximum}}
Crucially, Algorithm \ref{alg:implicit_ebm} when discretizing \eqref{eq:particle_system2} is exactly the algorithm studied by \cite{arbel2019maximum}. They start from pure MMD Wasserstein gradient flows, and they study convergence for those. They introduce noise injection/entropy regularization as a way to obtain certain convergence guarantees, and experimentally in their Figure 1 they observe a dramatic improvement in the training and test error against the pure MMD flow. Our theory justifies this behavior; their algorithm is implicitly training an $\mathcal{F}_2$-EBM and the noise level controls the RKHS radius over which the energy is optimized.

They propose using a schedule in which the noise decreases to zero (in our notation, $\beta \rightarrow +\infty$). This corresponds to optimizing over growing RKHS balls. Leveraging statistical learning results from \cite{domingoenrich2021energybased}, the generalization error can be written as a statistical (Rademacher complexity) term which increases with the radius $\beta$, plus an approximation term decreasing with $\beta$. Thus, there exists an optimal non-zero noise level which should be maintained.

\subsection{How to recover an explicit form of the energy} \label{subsec:energy_estimate}
Let $\nu^{\star}$ be the unique stationary solution of \eqref{eq:wasserstein_EBM}, which is the unique minimizer of \eqref{eq:EBM_primal} (see \autoref{lem:stationary_minimizer}). 
Also by \autoref{lem:stationary_minimizer}, this solution must fulfill
\begin{align} \label{eq:stationary_sol_mmd}
    \frac{d\nu^{\star}}{d\tau_{\mathcal{X}}} = \frac{1}{Z_{\beta}}\exp \left( - \beta \frac{\int_{\mathcal{X}} k(x,x') \ d(\nu^{\star} - \nu_n)(x')}{
    \text{MMD}_{k}(\nu^{\star}, \nu_n)}  \right)
\end{align}
This equality leads us to believe that when we run Algorithm \ref{alg:implicit_ebm},
$E_t(x) := \beta \int_{\mathcal{X}} k(x,x') \ d(\nu_{t,N} - \nu_n)(x')/\text{MMD}_{k}(\nu_{t,N}, \nu_n)$
can be used as an rough estimate of the energy of the trained implicit EBM at time $t$, although of course this intuition is only accurate when $\nu_{t,N}$ is close enough to the equilibrium measure $\nu^{\star}$. For consistency with \eqref{eq:hat_mu2}, it is also interesting to note that the estimate $E_t$ has constant RKHS norm $\|E_t\|_{\mathcal{H}} = \beta$, since $\| \int_{\mathcal{X}} k(x,x') \ d(\nu_{t,N} - \nu_n)(x') \|_{\mathcal{H}} = \text{MMD}_{k}(\nu_{t,N}, \nu_n)$.

Similar equations can be derived for the dynamics \eqref{eq:wasserstein_EBM_2}, which lead to an energy estimate of the form $E_t(x) := 2 \tilde{\beta} \int_{\mathcal{X}} k(x,x') \ d(\nu_{t,N} - \nu_n)(x')$.

\section{Training overparametrized two-layer neural networks via sampling} \label{sec:two_layer_nn_sampling}
In the previous section we described how the general duality result from \autoref{sec:general_duality} can be leveraged to train EBMs implicitly via the Wasserstein gradient flow of a functional formally similar to the two-layer neural network regression loss. In this section we take the reverse approach: we use the results from \autoref{sec:general_duality} to describe how overparametrized two-layer neural networks can be trained via techniques developed for maximum likelihood EBMs.

Let $\mathcal{X} \subseteq \R^d$ and let $\tau_{\mathcal{X}}$ be a fixed base probability measure over $\mathcal{X}$. Let $\mathcal{F}_1$ be the Barron space. Overparametrized two-layer neural network regression for some target $g: \mathcal{X} \to \R$ corresponds to solving 
\begin{align} \label{eq:nn_regression_0}
\min_{f \in \mathcal{B}_{\mathcal{F}_1}(\beta_0)} \int_{\mathcal{X}} \left( f(x) - g(x) \right)^2 d\tau_{\mathcal{X}}(x)
\end{align}
for an arbitrary ball radius $\beta_0$.
This problem has been tackled via Wasserstein gradient flows and propagation of chaos by several works \citep{rotskoff2018neural, chizat2018global, mei2018mean, sirignano2019mean}. We briefly summarize their construction up to slight differences. Functions in $\mathcal{B}_{\mathcal{F}_1}(\beta_0)$ can be written as $f(x) = \int_{\Theta} \phi(x, \theta) \ d\gamma(\theta)$ for some signed Radon measure $\gamma$ with bounded total variation norm ${\|\gamma\|}_{\text{TV}} := \int_{\Theta} 
\ d|\gamma|(\theta) \leq \beta_0$. Furthermore, if we set $\Omega = \Theta \times \R$ and take a surjective $\chi : \R \to \R$, 
we obtain the parametrization $f(x) = \int_{\Omega} \chi(r) \phi(x, \theta) \ d\mu(\theta,r)$ for some $\mu \in \mathcal{P}(\Omega)$ such that $\int |\chi(r)| 
\ d\mu(\theta,r) \leq \beta$. With this characterization, and writing compactly $\omega := (\theta,r)$ and $\tilde{\phi}(x,\omega) = \chi(r) \phi(x, \theta)$, we can rewrite \eqref{eq:nn_regression_0} as
\begin{align} \label{eq:nn_regression_1}
    \min_{\mu \in \mathcal{P}(\Omega)} \int_{\mathcal{X}} \left( \int_{\Omega} 
    \tilde{\phi}(x,\omega) \ d\mu(\omega)
    - g(x) \right)^2 d\tau_{\mathcal{X}}(x) + \delta \int_{\Omega} |\chi| 
    \ d\mu + \tilde{\beta}^{-1} \int_{\Omega} \log \left(\frac{d\mu}{d\lambda} \right) \ d\mu, 
\end{align}
where $\lambda$ denotes the Lebesgue measure over $\Theta$.
To go from \eqref{eq:nn_regression_0} to \eqref{eq:nn_regression_1}, we have switched from a constraint on the $\mathcal{F}_1$ norm to a penalization term $\delta \int_{\Omega} |\chi| 
\ d\mu$, and we have also added a differential entropy regularizer $\tilde{\beta}^{-1} \int_{\Omega} \log \left(\frac{d\mu}{d\lambda} \right) \ d\mu$, which \cite{rotskoff2018neural, mei2018mean} introduce to simplify their analysis. 

At this point, remark that if we define the probability measure $\tau_{\Omega}$ to have density $\frac{d\tau_{\Omega}}{d\lambda}(\theta, r) = \exp(-\beta \delta |\chi| 
)/Z$ w.r.t the Lebesgue measure, then we have
\begin{align}
\begin{split}
&\beta^{-1} \int_{\Omega} \log \left(\frac{d\mu}{d\lambda} \right) \ d\mu + \delta \int_{\Omega} |\chi| 
\ d\mu = \beta^{-1} \int_{\Omega} \log \left(\frac{d\mu}{d\lambda} \right) \ d\mu - \beta^{-1} \int_{\Omega} \log(\exp(-\beta \delta |\chi| 
)) \ d\mu  \\ &= \beta^{-1} \int_{\Omega} \log \left( \frac{d\mu}{d\lambda} \frac{1}{\exp(-\beta \delta |\chi| 
)} \right) \ d\mu = \beta^{-1} \int_{\Omega} \log \left( \frac{d\mu}{d\tau_{\Omega}} \right) \ d\mu + K = \tilde{\beta}^{-1} D_{KL}(\mu||\tau_{\Omega}) + K
\end{split}
\end{align}
for some constant $K$ arising from the normalization factor of $\tau_{\Omega}$. That is, up to a constant term equation \eqref{eq:nn_regression_1} can be rewritten as
\begin{align} \label{eq:nn_regression_2}
    \min_{\mu \in \mathcal{P}(\Omega)} \int_{\mathcal{X}} \left( \int_{\Omega} 
    \tilde{\phi}(x,\omega) \ d\mu(\omega) - g(x) \right)^2 d\tau_{\mathcal{X}}(x) + \tilde{\beta}^{-1} D_{KL}(\mu||\tau_{\Omega})
\end{align}
The key observation is that is equation is formally equal to \eqref{eq:general_primal_problem2} when we set $\mathcal{Z} = \mathcal{X}$, $\mathcal{Y} = \Omega$ and $\phi = \tilde{\phi}$. Most importantly, as shown by \autoref{cor:NN} we can apply \autoref{lem:primal} and the Fenchel duality result \autoref{thm:duality_general} to obtain links with the following problem, which is the analog of \eqref{eq:general_dual_problem}:
\begin{align}
\label{eq:nn_regression_dual}
    &\max_{\substack{h \in L^2(\mathcal{X}, \tau_{\mathcal{X}}) \\ \|h\|_{L^2} \leq 1}} - \int_{\mathcal{X}} g(x) h(x)\ d\tau_{\mathcal{X}}(x) - \frac{1}{\beta} \log\left(\int_{\Omega} \exp \left(-\beta \int_{\mathcal{X}} \tilde{\phi}(x,\omega) h(x)\ d\tau_{\mathcal{X}}(x) \right) d\tau_{\Omega}(\omega) \right)
\end{align}
\begin{corollary}
\label{cor:NN}
Let $h^{\star}$ be a solution of \eqref{eq:nn_regression_dual}. Then, $\mu^{\star} \in \mathcal{P}(\Omega)$ with density 
\begin{align}
    \frac{d\mu^{\star}}{d\tau_{\Omega}}(\omega) = \frac{1}{Z_{\beta}}\exp \left( - \beta \int_{\mathcal{X}} \tilde{\phi}(x,\omega) \ h^{\star}(\omega) d\tau_{\Omega}(\omega) \right),
\end{align}
is a solution of \eqref{eq:nn_regression_2} with 
$\tilde{\beta} = \beta \left( 4 \int_{\mathcal{X}} \left( \int_{\Omega} \tilde{\phi}(x, \omega) \ d\mu^{\star}(\omega) - g(x) \right)^2 d\tau_{\mathcal{X}}(x) \right)^{-1/2}$.
\end{corollary}

\section{Proofs of \autoref{sec:general_duality}} \label{sec:proofs_duality}

The proofs of \autoref{thm:duality_general} and \autoref{thm:duality_general_f1} are based on the proofs found in Appendix E of \cite{domingoenrich2021energybased}. 
We make use of Fenchel strong duality, which is stated in \autoref{thm:fenchel}.

\begin{theorem}
[Fenchel strong duality; \cite{borwein2005techniques}, pp. 135-137] 
\label{thm:fenchel}
Let $X$ and $Y$ be Banach spaces, 
$f:X\to \mathbb {R} \cup \{+\infty \}$ and 
$g:Y\to \mathbb {R} \cup \{+\infty \}$ be convex functions and 
$A:X\to Y$ be a bounded linear map. Define the Fenchel problems:
\begin{align}
\begin{split}
p^{*} &= \inf _{x\in X}\{f(x)+g(Ax)\} \\
d^{*} &= \sup _{y^{*}\in Y^{*}}\{-f^{*}(A^{*}y^{*})-g^{*}(-y^{*})\},
\end{split}
\end{align}
where $f^{*}(x^{*}) = \sup_{x \in X} \{ \langle x, x^{*} \rangle - f(x)\}, \ g^{*}(y^{*}) = \sup_{y \in Y} \{ \langle y, y^{*} \rangle - g(y)\}$ are the convex conjugates of $f,g$ respectively, and $A^{*} : Y^{*} \to X^{*}$ is the adjoint operator. Then, $p^{*} \geq d^{*}$. Moreover if $f,g,$ and $A$ satisfy either
\begin{enumerate}
\item $f$ and $g$ are lower semi-continuous and 
$0\in \operatorname{core} (\operatorname{dom} g-A \operatorname{dom} f)$ where 
$\operatorname{core}$ is the algebraic interior and 
$\operatorname{dom} h$, where $h$ is some function, is the set 
$\{z:h(z)<+\infty \}$, 
\item or
$A\operatorname{dom} f\cap \operatorname{cont} g\neq \emptyset$  where 
$\operatorname{cont}$ are is the set of points where the function is continuous.
\end{enumerate}
Then strong duality holds, i.e. 
$p^{*}=d^{*}$. If 
$d^{*}\in \mathbb {R}$ then supremum is attained.
\end{theorem}

We also rely on a generalization of von Neumann's minimax theorem \citep{neumann1928theorie}. For our purposes, the theorem stated below by \cite{kneser1952surun} suffices, but a further generalization by \cite{sion1958} to quasi-convex and quasi-concave functions is more widely known in the literature. Note however that the compactness assumption on one of the sets cannot be relaxed.

\begin{theorem}[\cite{kneser1952surun}]
\label{thm:kneser}
Let $A$ be a non-empty compact convex subset of a locally convex topological vector space space $E$ and $B$ a non-empty convex subset of a locally convex topological vector space space $F$. Let the function $f : X \times Y \rightarrow \R$ be such that:
\begin{itemize}
    \item For each $y \in B$, the function $x \mapsto f(x,y)$ is upper semicontinuous and concave,
    \item For each $x \in A$, the function $y \mapsto f(x,y)$ is convex.
\end{itemize}
Then we have
\begin{align}
    \sup_{x \in A} \inf_{y \in B} f(x,y) = \inf_{y \in B} \max_{x \in A} f(x,y).
\end{align}
\end{theorem}

We also make use of the Riesz-Markov-Kakutani theorem, which we reproduce in \autoref{thm:rieszmarkov}. 
\begin{theorem}[Riesz-Markov-Kakutani representation theorem]
\label{thm:rieszmarkov}
Let $X$ be a locally compact Hausdorff space and let $C_0(X)$ be the space of continuous functions from $X$ to $\R$ vanishing at infinity, i.e. such that $f(x) \rightarrow 0$ when $\|x\|_2 \rightarrow \infty$. For any continuous linear functional $\psi$ on $C_0(X)$, there is a unique (countably additive) finite signed regular Borel measure $\mu$ on $X$ such that
\begin{align}
    \forall f\in C_{0}(X):\quad \psi (f)=\int _{X}f(x)\,d\mu (x).
\end{align}
The norm of $\psi$ as a linear functional is the total variation of $\mu$, that is $\|\psi\| = \|\mu\|_{\text{TV}} = |\mu|(X) = \mu_{+}(X) + \mu_{-}(X)$, where the decomposition $\mu = \mu_{+} - \mu_{-}$ into positive measures is given by the Hahn decomposition theorem.
Finally, $\psi$ is positive if and only if the measure $\mu$ is non-negative.
\end{theorem}

By definition, the space of finite signed Radon measures $\mathcal{M}(X)$ is the same as the space of finite signed regular Borel measures (Radon measures are Borel measures that are finite on compact sets, which is holding directly because we restrict to finite measures). In other words, \autoref{thm:rieszmarkov} states that we have an isometry between the topological dual $C_{0}^{*}(X)$ and $\mathcal{M}(X)$. The following theorem is an analogous result for the dual of the Banach space $C_{b}(X)$ of bounded continuous functions.

\begin{theorem}
[Riesz representation theorem for $C_{b}^{*}(X)$, \cite{dunford1958linear}]
\label{thm:rieszmarkov_cb}
Let $X$ be a normal topological space. Let $\text{rba}(X)$ be the space of finitely additive finite signed regular Borel measures $\mu$ on $X$. It holds that
\begin{align}
    C_b^{*}(X) = \text{rba}(X).
\end{align}
\end{theorem}

Finally, we recall the Banach-Alaoglu theorem from functional analysis, which we will use to show compactness and apply \autoref{thm:kneser}.
\begin{theorem}
\label{thm:banachalaoglu}
For any topological vector space $X$ with continuous dual space $X^{*}$, the closed unit ball of $X^{*}$ in the dual norm (i.e. $\mathcal{B}_{X^{*}} = \{ x^{*} \in X^{*} | \sup_{x \in X} \langle x^{*}, x \rangle \leq 1 \rangle \}$) is compact in the weak-* topology, which the weakest topology on $X^{*}$ making all maps $\langle x,\cdot \rangle : X^{*} \to \R$ continuous, as $x$ ranges over $X$. In particular, for Hilbert spaces $H$ we have that $\mathcal{B}_{H}$ is compact in the weak-* topology, which coincides with the weak topology in this case.
\end{theorem}

\textbf{\textit{Proof of \autoref{thm:duality_general}.}}
We use \autoref{thm:fenchel}. 

On the one hand, we set $X = \mathcal{M}_{\xi_1}(\mathcal{Y})$, which we define to be the space of Radon measures over $\mathcal{Y}$ such that the weighted total variation
\begin{align}
\|\nu\|_{\text{TV}, \xi_1} := \int_{\mathcal{Y}} \xi_1(y) \ d|\nu|(y)
\end{align}
is finite, where $\xi_1 : \mathcal{Y} \to \R$ is the strictly positive function given by \autoref{ass:phi}(i). By \autoref{lem:m_xi1}, $\mathcal{M}_{\xi_1}(\mathcal{Y})$ is a Banach space with norm $\|\cdot \|_{\text{TV}, \xi_1}$ and its continuous dual contains the set $C_{b, \xi_1}(\mathcal{Y})$ of continuous functions $f$ such that $f/\xi_1$ is bounded.

On the other hand, we set $Y = L^2(\mathcal{Z}) = \{ f : \mathcal{Z} \to \R \ | \ \int_{\mathcal{Z}} f(z)^2 
\ d\tau_{\mathcal{Z}}(z) < +\infty \}$, the Hilbert space of square-integrable functions on $\mathcal{Z}$ under the base measure $\tau_{\mathcal{Z}}$, 
which is of course self-dual.

Define $F: \mathcal{M}_{\xi_1}(\mathcal{Y}) \to \R \cup \{+\infty \}$ as
\begin{align} \label{eq:F_def}
    F(\nu) = 
    \begin{cases}
        \beta^{-1} D_{KL}(\nu||\tau_{\mathcal{Y}}) &\text{if } \nu \in \mathcal{P}(\mathcal{Y}), \\
        +\infty &\text{otherwise}
    \end{cases}.
\end{align}
\autoref{lem:F_conj} states that $F$ is a convex functional and that its convex conjugate $F^{*}: \mathcal{M}_{\xi_1}^{*}(\mathcal{Y}) \to \R \cup \{+\infty \}$ restricted to $C_{b, \xi_1}(\mathcal{Y})$ satisfies 
\begin{align} \label{eq:F_conj}
    F^{*}(q) = \beta^{-1} \log \left( \int_{\mathcal{Y}} \exp(\beta q(y')) \ d\tau_{\mathcal{Y}}(y') \right).
\end{align}
Define $G : L^2(\mathcal{Z}) \to \R \cup \{+\infty \}$ as
\begin{align} \label{eq:G_def}
    G(\psi) = \left(\int_\mathcal{Z} \left( \psi(z) - g(z) \right)^2 
    \ d\tau_{\mathcal{Z}}(z) \right)^{1/2}, 
\end{align}
\autoref{lem:G_conj} states that $G$ is a convex functional and that its convex conjugate $G^{*} : L^2(\mathcal{Z}) \to \R \cup \{+\infty \}$ is of the form
\begin{align} \label{eq:G_conj}
    G^{*}(\psi) = \begin{cases}
        \int_{K} g(z) \psi(z) \ d\tau_{\mathcal{Z}}(z) &\text{if } \|\psi\|_{L^2(\mathcal{Z})} \leq 1,\\
        +\infty &\text{otherwise}
    \end{cases}.
\end{align}
Define $A : \mathcal{M}_{\xi_1}(\mathcal{Y}) \rightarrow L^2(\mathcal{Z})$ as 
\begin{align}
(A \nu)(z) = \int_{\mathcal{Y}} \phi(y,z) \ d\nu(y).
\end{align}
The linear operator $A$ is well defined and continuous by \autoref{lem:A_well}. \autoref{lem:A_well} also states that $A^{*} : L^2(\mathcal{Z}) \rightarrow \mathcal{M}_{\xi_1}^{*}(\mathcal{Y})$ is of the form
\begin{align} \label{eq:A_star}
    (A^{*} h)(y) = \int_{\mathcal{Z}} \phi(y,z) h(z) \ d\tau_{\mathcal{Z}}(z)
\end{align}

Hence, we have that 
$\min_{\nu \in \mathcal{M}_{\xi_1}(\mathcal{Y})} \beta^{-1} D_{KL}(\nu||\tau_{\mathcal{Y}}) + ( \int_{\mathcal{Z}} ( \int_{\mathcal{Y}} \phi(y, z) \ d\nu(y) - g(z) )^2 \ d\tau_{\mathcal{Z}}(z) )^{1/2}$
can be written as
\begin{align} \label{eq:primal2}
    p^{*} = \inf_{\nu \in \mathcal{M}_{\xi_1}(\mathcal{Y})} \{F(\nu)+G(A\nu) \}.
\end{align}
And problem \eqref{eq:general_dual_problem} can be written as
\begin{align} \label{eq:dual2}
    d^{*} = \sup_{\substack{h \in L^2(\mathcal{Z}), \\ \|h\|_{L^2} \leq 1}} \{-F^{*}(-A^{*} h)-G^{*}(h)\}.
\end{align}
To apply \autoref{thm:fenchel}, it only remains to show that condition 2 holds. That is, we have to check that $A \ \text{dom} \ F \cap \text{cont} \ G \neq \emptyset$. Consider $\psi = A \nu$ for some $\nu \in \mathcal{P}(\mathcal{Y}) \cap \mathcal{M}_{\xi_1}(\mathcal{Y})$ absolutely continuous w.r.t. $\tau_{\mathcal{Y}}$. Then $\psi \in A \ \text{dom} \ F$. Moreover, since $G$ is a continuous functional, we have that $\text{cont} \ G = L^2(\mathcal{Z})$. Thus, $\psi$ also belongs to $\text{cont} \ G$ and we conclude that $A \ \text{dom} \ F \cap \text{cont} \ G \neq \emptyset$.

By \autoref{thm:fenchel}, $p^{*} = d^{*}$, and since $p^{*}$ is finite, we have that the supremum in \eqref{eq:dual2} is attained; let $h^{\star}$ be one maximizer. We show that $p^{*} = \inf_{\nu \in \mathcal{M}_{\xi_1}(\mathcal{Y})} \{F(\nu)+G(A\nu) \} = \inf_{\nu \in \mathcal{M}_{\xi_1}(\mathcal{Y}) \cap \mathcal{P}(\mathcal{Y})} \{F(\nu)+G(A\nu) \}$ admits a minimizer by the direct method of the calculus of variations. First, notice that $F$ and $G \circ A$ are lower semicontinuous in the topology of weak convergence:
\begin{itemize}
    \item $F$ by the lower semicontinuity of the KL divergence \citep{posner_random},
    \item and $G \circ A$ because can be written as a supremum of continuous functions as shown in \eqref{eq:G_as_sup}, and thus its sublevel sets are closed because they are the intersection of closed sublevel sets. Closed sublevel sets is equivalent to lower semicontinuity.
\end{itemize}  
Second, $\mathcal{P}(\mathcal{Y}) \cap \mathcal{M}_{\xi_1}(\mathcal{Y})$ is compact, because $\mathcal{P}(\mathcal{Y})$ is compact and $\mathcal{M}_{\xi_1}(\mathcal{Y})$ is closed as it is a Banach space. Hence, the direct method of the calculus of variations applies. Let $\nu^{\star}$ be one minimizer of $p^{*}$. 

It remains to show that 
\begin{align} \label{eq:gamma_euler}
\frac{d\nu^{\star}}{d\tau_{\mathcal{Y}}}(y) = \frac{1}{Z_{\beta}}\exp \left( - \beta \int_{\mathcal{Z}} \phi(y, z) \ h^{\star}(z) d\tau_{\mathcal{Z}}(z) \right).
\end{align}
We make use of the argument to prove Fenchel weak duality, which is:
\begin{align}
\begin{split} \label{eq:fenchel_weak}
    &\sup_{\substack{h \in L^2(\mathcal{Z}), \\ \|h\|_{L^2} \leq 1}} \{-F^{*}(-A^{*} h)-G^{*}(h)\} = -F^{*}(-A^{*} h^{\star})-G^{*}(h^{\star}) \\&= 
    - \sup_{\nu \in \mathcal{M}_{\xi_1}(\mathcal{Y})} \left\{\langle - A^{*}h^{\star}, \nu \rangle - F(\nu) \right\} - \sup_{\psi \in L^2(\mathcal{Z})} \left\{\langle h^{\star}, \psi \rangle - G(\psi) \right\}
    \\&\leq - \sup_{\nu \in \mathcal{M}_{\xi_1}(\mathcal{Y})} \left\{\langle - A^{*}h^{\star},\nu \rangle - F(\nu) + \langle h^{\star}, A \nu \rangle - G(A \nu) \right\} \\&= - \sup_{\nu \in \mathcal{M}_{\xi_1}(\mathcal{Y})} \left\{ - F(\nu) - G(A \nu) \right\} = \inf_{\nu \in \mathcal{M}_{\xi_1}(\mathcal{Y})} \{F(\nu)+G(A\nu)\} \\ &= F(\nu^{\star})+G(A\nu^{\star})
\end{split}
\end{align}
Thus, for strong duality to hold we must have that 
\begin{align} \label{eq:strong_duality_cond}
\nu^{\star} = \argmin_{\nu \in \mathcal{M}_{\xi_1}(\mathcal{Y})} \left\{\langle - A^{*}h^{\star}, \nu \rangle - F(\nu) \right\}
\end{align}
By \autoref{lem:gamma_star}(i), this implies that equation \eqref{eq:gamma_euler} holds, and by \autoref{lem:gamma_star}(ii) we have that $\nu^{\star} = \argmin_{\nu \in \mathcal{P}(\mathcal{Y})} \{F(\nu)+G(A\nu) \}$.
\qed

\begin{lemma} \label{lem:m_xi1}
Let $\mathcal{M}_{\xi_1}(\mathcal{Y})$ be the vector space of Radon measures over $\mathcal{Y}$ such that the weighted total variation $\|\nu\|_{\text{TV}, \xi_1} := \int_{\mathcal{Y}} \xi_1(y) \ d|\nu|(y)$ 
is finite, where $\xi_1 : \mathcal{Y} \to \R$ is the strictly positive function given by \autoref{ass:phi}(i). $\mathcal{M}_{\xi_1}(\mathcal{Y})$ is a Banach space with norm $\|\cdot\|_{\text{TV}, \xi_1}$. 

Let $C_{b, \xi_1}(\mathcal{Y})$ be the set of continuous functions $f$ such that $f/\xi_1 \in C_b(\mathcal{Y})$, i.e. is a bounded continuous function. The continuous dual $\mathcal{M}_{\xi_1}^{*}(\mathcal{Y})$ contains the set $C_{b, \xi_1}(\mathcal{Y})$.   
\end{lemma}

\begin{proof}
If we define the linear map $\tilde{\xi_1} : \mathcal{M}_{\xi_1}(\mathcal{Y}) \to \mathcal{M}(\mathcal{Y})$ as $\nu \mapsto \tilde{\nu}$, where $\tilde{\nu}$ is absolutely continuous w.r.t $\nu$ and has density $\frac{d\tilde{\nu}}{d\nu}(y) = \xi_1(y)$, we have that $ \|\nu\|_{\text{TV}, \xi_1} = \|\tilde{\xi_1}(\nu)\|_{\text{TV}}$. Notice that $\tilde{\xi_1}$ is surjective, because for all $\tilde{\nu} \in \mathcal{M}(\mathcal{Y})$, the measure $\nu$ with density $\frac{d\nu}{d\tilde{\nu}}(y) = \xi_1(y)^{-1}$ is a Radon measure (possibly not signed, because we cannot guarantee that $\nu_{+}$ nor $\nu_{-}$ is finite) such that $\tilde{\xi_1} \nu = \tilde{\nu}$. $\tilde{\xi_1}$ is a surjective isometry between $\mathcal{M}_{\xi_1}(\mathcal{Y})$ and $\mathcal{M}(\mathcal{Y})$, which shows that $\mathcal{M}_{\xi_1}(\mathcal{Y})$ is a Banach space. 

Let $\mathcal{M}^{*}(\mathcal{Y})$ be the dual space of $\mathcal{M}(\mathcal{Y})$. By the Riesz-Markov-Kakutani representation theorem (\autoref{thm:rieszmarkov}) and the fact that the double dual space contains the primal space, $\mathcal{M}^{*}(\mathcal{Y})$ immediately contains the set of continuous functions $C_0(\mathcal{Y})$ on $\mathcal{Y}$ vanishing at infinity. Furthermore, $\mathcal{M}^{*}(\mathcal{Y})$ contains the larger set of bounded continuous functions $C_b(\mathcal{Y})$, because if $f \in C_b(\mathcal{Y})$, for any $\tilde{\nu} \in \mathcal{M}(\mathcal{Y})$, 
\begin{align}
\langle f, \tilde{\nu} \rangle = \int_{\mathcal{Y}} f(y) \ d\tilde{\nu}(y) \leq \sup_{y \in \mathcal{Y}} f(y) \|\tilde{\nu}\|_{\text{TV}}. 
\end{align}
In an analogous way, $\mathcal{M}^{*}_{b, \xi_1}(\mathcal{Y})$ contains the set $C_{b, \xi_1}(\mathcal{Y})$, because if $f \in C_{b, \xi_1}(\mathcal{Y})$, for any $\nu \in \mathcal{M}_{b, \xi_1}(\mathcal{Y})$,
\begin{align}
\langle f, \nu \rangle = \int_{\mathcal{Y}} \frac{f(y)}{\xi_1(y)} \xi_1(y) \ d\nu(y) = \int_{\mathcal{Y}} \frac{f(y)}{\xi_1(y)} \ d\tilde{\nu}(y) \leq \|\tilde{\nu}\|_{\text{TV}} \sup_{y \in \mathcal{Y}} \frac{f(y)}{\xi_1(y)} = \|\nu\|_{\text{TV}, \xi_1} \sup_{y \in \mathcal{Y}} \frac{f(y)}{\xi_1(y)}
\end{align}
\end{proof}

\begin{lemma} \label{lem:F_conj}
$F: \mathcal{M}_{\xi_1}(\mathcal{Y}) \to \R \cup \{+\infty \}$ defined in equation \eqref{eq:F_def} is convex. The restriction of its convex conjugate $F^{*} : \mathcal{M}_{\xi_1}^{*}(\mathcal{Y}) \to \R$ to the set $C_{b, \xi_1}(\mathcal{Y}) := \{ f \in C(\mathcal{Y}) \ | \ f(\cdot)/\xi_1(\cdot) \in C_b(\mathcal{Y}) \} \subseteq \mathcal{M}_{\xi_1}^{*}(\mathcal{Y})$ is given by equation \eqref{eq:F_conj}.
\end{lemma}

\begin{proof}
It is well known that the KL divergence is convex. We compute the (restriction of the) convex conjugate via a classical argument (c.f. Lemma B.37 of \cite{mohri2012foundations}): for any $q : \mathcal{Z} \to \R$ belonging to $C_{b, \xi_1}(\mathcal{Y})$, define $\tilde{q} \in \mathcal{P}(\mathcal{Y})$ with density $\frac{d\tilde{q}}{d\tau_{\mathcal{Y}}}(y) = \exp(\beta q(y))/\int_{\mathcal{Y}} \exp(\beta q(y')) d\tau_{\mathcal{Y}}(y')$. Then,
\begin{align}
\begin{split} \label{eq:F_tilde_derivation}
F^{*}(q) &= \sup_{\nu \in \mathcal{M}_{\xi_1}(\mathcal{Y}) \cap \mathcal{P}(\mathcal{Y})} \int_{\mathcal{Y}} q(y)\ d\nu(y) - \beta^{-1} D_{KL}(\nu||\tau_{\mathcal{Y}}) \\ &= \sup_{\nu \in \mathcal{M}_{\xi_1}(\mathcal{Y}) \cap \mathcal{P}(\mathcal{Y})} \int_{\mathcal{Y}} \log(\exp(q(y))) \ d\nu(y) - \beta^{-1} \int_{\mathcal{Y}} \log\left(\frac{d\nu}{d\tau_{\mathcal{Y}}}(y) \right) \ d\nu(y) \\ &= \sup_{\nu \in \mathcal{M}_{\xi_1}(\mathcal{Y}) \cap \mathcal{P}(\mathcal{Y})} \beta^{-1} \int_{\mathcal{Y}} \log\left(\frac{d\tau_{\mathcal{Y}}}{d\nu}(y) \exp(\beta q(y)) \right) \ d\nu(y) \\ &= \sup_{\nu \in \mathcal{M}_{\xi_1}(\mathcal{Y}) \cap \mathcal{P}(\mathcal{Y})} \beta^{-1} \int_{\mathcal{Y}} \log\left(\frac{d\tilde{q}}{d\nu}(y) \right) \ d\nu(y) \\ &= \sup_{\nu \in \mathcal{M}_{\xi_1}(\mathcal{Y}) \cap \mathcal{P}(\mathcal{Y})} - \beta^{-1} D_{KL}(\nu||\tilde{q}) + \beta^{-1} \log \left( \int_{\mathcal{Y}} \exp(\beta q(y')) \ d\tau_{\mathcal{Y}}(y') \right) \\ &= \beta^{-1} \log \left( \int_{\mathcal{Y}} \exp(\beta q(y')) \ d\tau_{\mathcal{Y}}(y') \right)
\end{split}
\end{align}
It remains to justify the last equality, which follows from checking that $\tilde{q} \in \mathcal{M}_{\xi_1}(\mathcal{Y}) \cap \mathcal{P}(\mathcal{Y})$. For this, we need to see that $\|\tilde{q}\|_{\text{TV}, \xi_1}$ is finite making use of \autoref{ass:phi}(ii):
\begin{align}
\begin{split} \label{eq:tv_norm_bound}
    &\xi_1(y) + \log(\xi_1(y)) = o\left(-\log\left(\frac{d\tau_{\mathcal{Y}}}{d\lambda}(y) \right) - (d+\epsilon) \log \|y\|_2 \right) \text{ as } \|y\|_2 \to +\infty \\
    &\implies \lim_{\|y\|_2 \to +\infty} \log(\xi_1(y)) + q(y) - \log Z_q + \log\left(\frac{d\tau_{\mathcal{Y}}}{d\lambda}(y) \right) + (d+\epsilon) \log \|y\|_2 = - \infty
    \\ &\implies \exp \left( \log(\xi_1(y)) + q(y) - \log Z_q + \log\left(\frac{d\tau_{\mathcal{Y}}}{d\lambda}(y) \right) + (d+\epsilon) \log \|y\|_2 \right) = o(1)
    \\ &\implies \|\tilde{q}\|_{\text{TV}, \xi_1} = \int_{\mathcal{Y}} \xi_1(y) \ d|\tilde{q}|(y) = \int_{\mathcal{Y}} \exp \left( \log(\xi_1(y)) + q(y) - \log Z_q + \log\left(\frac{d\tau_{\mathcal{Y}}}{d\lambda}(y) \right) \right) d\lambda(y) \\ &= \int_{\R_{+}} \int_{\mathbb{S}^{d-1}} \mathds{1}_{r \theta \in \mathcal{Y}} \exp \left( \log(\xi_1(r\theta)) + q(r\theta) - \log Z_q + \log\left(\frac{d\tau_{\mathcal{Y}}}{d\lambda}(r\theta) \right) + (d+\epsilon) \log r \right) K_1 r^{-1-\epsilon} \ dr d\lambda(\theta) \\ &\leq \int_{\R_{+}} \int_{\mathbb{S}^{d-1}} K_2 K_1 r^{-1-\epsilon} \ dr d\lambda(\theta) = \frac{K_2 K_1 \text{vol}(\mathbb{S}^{d-1})}{\epsilon}.  
\end{split}
\end{align}
In this equation, $\lambda$ denotes the Lebesgue measure over $\mathcal{Y}$. In the last inequality, we use that $\exp \bigg( \log(\xi_1(y)) + q(y) - \log Z_q + \log\left(\frac{d\tau_{\mathcal{Y}}}{d\lambda}(y) \bigg) + (d+\epsilon) \log \|y\|_2 \right) = o(1)$ to show the existence of some constant bound $K_2$ of this expression over all $y \in \mathcal{Y}$.
\end{proof}

\begin{lemma} \label{lem:G_conj}
$G : L^2(\mathcal{Z}) \to \R \cup \{+\infty \}$ given by equation \eqref{eq:G_def} is convex. Its convex conjugate $G^{*} : L^2(\mathcal{Z}) \to \R \cup \{+\infty\}$ is given by equation \eqref{eq:G_conj}.
\end{lemma}
\begin{proof}
We can easily check that $G$ is convex by writing \begin{align} \label{eq:G_as_sup}
    \left(\int_\mathcal{Z} \left( \psi(z) - g(z) \right)^2 
    \ d\tau_{\mathcal{Z}}(z) \right)^{1/2} = \sup_{\substack{h \in L^2(\mathcal{Z}), \\ \|h\|_{L^2} \leq 1}} \int_\mathcal{Z} \left( \psi(z) - g(z) \right) h(z) 
    \ d\tau_{\mathcal{Z}}(x),
\end{align}
(in more compact notation $\| \chi - g \|_{L^2(\mathcal{Z})} = \sup_{h \in L^2(\mathcal{Z}), \ \|h\|_{L^2} \leq 1} \langle \chi - g, h \rangle$),
and recalling that a supremum of convex functions is convex. 

By definition, for any $\psi \in L^2(\mathcal{Z},\tau_{\mathcal{Z}})$, we have that $G^{*}(\psi)$ is equal to
\begin{align}
\begin{split} \label{eq:G_star}
    &\sup_{\chi \in L^2(\mathcal{Z})} \left\{
    \langle \chi, \psi \rangle_{L^2(\mathcal{Z})} - G(\chi)
    \right\} = \sup_{\chi \in L^2(\mathcal{Z})} \left\{
    \langle \chi, \psi \rangle_{L^2(\mathcal{Z})} - \| \chi - g \|_{L^2(\mathcal{Z})}
    \right\} 
    \\ &= \sup_{\chi \in L^2(\mathcal{Z})} \bigg\{
    \langle \chi, \psi \rangle_{L^2(\mathcal{Z})} - \sup_{\substack{\hat{\psi} \in L^2(\mathcal{Z}), \\ \|\hat{\psi}\|_{L^2} \leq 1}} \langle \chi - g, \hat{\psi} \rangle_{L^2(\mathcal{Z})} 
    \bigg\}
    \\ &= \sup_{\chi \in L^2(\mathcal{Z})} \inf_{\substack{\hat{\psi} \in L^2(\mathcal{Z}), \\ \|\hat{\psi}\|_{L^2} \leq 1}} \left\{
    \langle \chi, \psi - \hat{\psi} \rangle_{L^2(\mathcal{Z})} + \langle g, \hat{\psi} \rangle_{L^2(\mathcal{Z})} 
    \right\}.
\end{split}
\end{align}
At this point, we want to apply \autoref{thm:kneser}. For that, we set $A = \mathcal{B}_{L^2(\mathcal{Z})} \subseteq E = L^2(\mathcal{Z})$ and $B = F = L^2(\mathcal{Z})$. We can endow $B$ with the strong (or norm) topology, but $A$ requires a weaker topology that makes it compact. We endow $A$ with the weak-* topology, which by the Banach-Alaoglu theorem (\autoref{thm:banachalaoglu}) for Hilbert spaces makes it compact.
We have that $(\hat{\psi}, \chi) \mapsto H(\hat{\psi}, \chi) = -\langle \chi, \psi - \hat{\psi} \rangle_{L^2(\mathcal{Z})} - \langle g, \hat{\psi} \rangle_{L^2(\mathcal{Z})}$ is concave in $\hat{\psi} \in A$ and convex in $\chi \in B$, because it is affine in both variables. $H$ is continuous in $\chi$ (via Cauchy-Schwarz) and it is continuous in $\hat{\psi}$ in the weak-* (or weak) topology, because it is precisely the weakest one that makes maps of the form $\hat{\psi} \mapsto \langle \hat{\psi}, g - \chi \rangle_{L^2(\mathcal{Z})}$ continuous. Thus, we obtain that $\sup_{\hat{\psi} \in A} \inf_{\chi \in B} H(\chi, \hat{\psi}) = \inf_{\chi \in B} \sup_{\hat{\psi} \in A} H(\chi, \hat{\psi})$. Alternatively, if we flip the signs, we get that the right-hand side of \eqref{eq:G_star} is equal to: 
\begin{align}
\begin{split}
    &\inf_{\substack{\hat{\psi} \in L^2(\mathcal{Z}), \\ \|\hat{\psi}\|_{L^2} \leq 1}} \sup_{\chi \in L^2(\mathcal{Z})} \left\{
    \langle \chi, \psi - \hat{\psi} \rangle_{L^2(\mathcal{Z})} + \langle g, \hat{\psi} 
    \rangle_{L^2(\mathcal{Z})} 
    \right\}
    \\ &= 
    \begin{cases}
        \int_{K} g(z) \psi(z) \ d\tau_{\mathcal{Z}}(z) &\text{if } \|\psi\|_{L^2} \leq 1,\\
        +\infty &\text{otherwise}
    \end{cases}
\end{split}
\end{align}
The equality holds because unless $\hat{\psi} = \psi$, the value of the supremum is $+\infty$.
\end{proof}

\begin{lemma} \label{lem:A_well}
The linear operator $A : \mathcal{M}_{\xi_1}(\mathcal{Y}) \rightarrow L^2(\mathcal{Z})$ defined as 
$(A \nu)(z) = \int_{\mathcal{Y}} \phi(y,z) \ d\nu(y)$ is well defined and continuous. Its operator norm is upper bounded by $\sup_{y \in \mathcal{Y}} |\xi_2(y)|/\xi_1(y)$, where $\xi_1$ and $\xi_2$ are defined in \autoref{ass:phi}. 
Moreover, the adjoint operator $A^{*} : L^2(\mathcal{Z}) \rightarrow \mathcal{M}_{\xi_1}^{*}(\mathcal{Y})$ is defined as $(A^{*} h)(y) = \int_{\mathcal{Z}} \phi(y,z) h(z) \ d\tau_{\mathcal{Z}}(z)$.
\end{lemma}
\begin{proof}
Remark that $\int_{\mathcal{Y}} \phi(y, \cdot) \ d\nu(y)$ does belong to $L^2(\mathcal{Z})$ because
\begin{align}
\begin{split}
    &\int_{\mathcal{Z}} \left(\int_{\mathcal{Y}} \phi(y, z) \ d\nu(y)\right)^2\ d\tau_{\mathcal{Z}}(z) = \int_{\mathcal{Z}} \left(\int_{\mathcal{Y}} \frac{\phi(y, z)}{\xi_1(y)} \ d\tilde{\nu}(y)\right)^2\ d\tau_{\mathcal{Z}}(z) \\ &= \int_{\mathcal{Z}} \left(\int_{\mathcal{Y}} \frac{\phi(y, z)}{\xi_1(y)} \|\tilde{\nu}\|_{\text{TV}} \ d\frac{\tilde{\nu}}{\|\tilde{\nu}\|_{\text{TV}}}(y)\right)^2\ d\tau_{\mathcal{Z}}(z) \leq \|\tilde{\nu}\|_{\text{TV}} \int_{\mathcal{Z}} \int_{\mathcal{Y}} \left(\frac{\phi(y, z)}{\xi_1(y)} \right)^2 \ d|\tilde{\nu}|(y) \ d\tau_{\mathcal{Z}}(z) \\ &= \|\tilde{\nu}\|_{\text{TV}} \int_{\mathcal{Y}} \frac{1}{\xi_1(y)^2} \int_{\mathcal{Z}} \phi(y, z)^2 \ d\tau_{\mathcal{Z}}(z) \ d|\tilde{\nu}|(y) = \|\tilde{\nu}\|_{\text{TV}} \int_{\mathcal{Y}} \frac{\xi_2(y)^2}{\xi_1(y)^2} \ d|\tilde{\nu}|(y) \\ &\leq \|\tilde{\nu}\|_{\text{TV}}^2 \left(\sup_{y \in \mathcal{Y}} \frac{|\xi_2(y)|}{\xi_1(y)} \right)^2. 
\end{split}
\end{align}
In the first equality we have used the change of variable $\tilde{\nu} = \tilde{\xi_1}(\nu)$. In the first inequality we have used the Cauchy-Schwarz inequality, and in the following equality we used Fubini's theorem, which holds because the integrand is positive. In the last equality we have used the definition of $\xi_2$ given by \autoref{ass:phi}(iii). Also by \autoref{ass:phi}(iii), the right-most expression is finite, implying that $A \nu \in L^2(\mathcal{Z})$. Furthermore, since $\|\tilde{\nu}\|_{\text{TV}} = \|\nu\|_{\text{TV}, \xi_1}$, we also conclude that $A$ is a continuous operator with norm bounded by $\sup_{y \in \mathcal{Y}} |\xi_2(y)|/\xi_1(y)$.

We have that $A^{*} : L^2(\mathcal{Z}) \rightarrow \mathcal{M}_{\xi_1}^{*}(\mathcal{Y})$ is defined as $(A^{*} h)(y) = \int_{\mathcal{Z}} \phi(y,z) h(z) \ d\tau_{\mathcal{Z}}(z)$, because
\begin{align}
\begin{split} \label{eq:adjoint_exp}
    \langle A \nu, h \rangle &= \int_{\mathcal{Z}} (A \nu)(z) h(z) \ d\tau_{\mathcal{Z}}(z) = \int_{\mathcal{Z}} \int_{\mathcal{Y}} \phi(y,z) \ d\nu(y) \ h(z) \ d\tau_{\mathcal{Z}}(z) \\ &= \int_{\mathcal{Y}} \int_{\mathcal{Z}} \phi(y,z) h(z) \ d\tau_{\mathcal{Z}}(z) \ d\nu(y).
\end{split}
\end{align}
In the last equality we have applied Fubini's theorem, which holds because by the Cauchy-Schwarz inequality,
\begin{align}
\begin{split} \label{eq:fubini_1}
    &\int_{\mathcal{Y}} \int_{\mathcal{Z}} |\phi(y,z) h(z)| \ d\tau_{\mathcal{Z}}(z) \ d\nu(y) \leq \|\tilde{\nu}\|_{\text{TV}} \int_{\mathcal{Y}} \int_{\mathcal{Z}} \frac{|\phi(y,z)|}{\xi_1(y)} |h(z)| \ d\tau_{\mathcal{Z}}(z) \ d\frac{|\tilde{\nu}|}{\|\tilde{\nu}\|_{\text{TV}}}(y) \\ 
    &\leq \left( \int_{\mathcal{Y}} \int_{\mathcal{Z}} \left( \frac{|\phi(y,z)|}{\xi_1(y)} \right)^2 \ d\tau_{\mathcal{Z}}(z) \ d|\tilde{\nu}|(y) \right)^{1/2} \left( \int_{\mathcal{Y}} \int_{\mathcal{Z}} |h(z)|^2 \ d\tau_{\mathcal{Z}}(z) \ d|\tilde{\nu}|(y) \right)^{1/2} 
    \\ &= \left( \int_{\mathcal{Y}} \left( \frac{\xi_2(y)}{\xi_1(y)} \right)^2 \ d|\tilde{\nu}|(y) \right)^{1/2} \|h\|_{L^2(\mathcal{Z}, \tau_{\mathcal{Z}})} \|\tilde{\nu}\|_{\text{TV}}^{1/2} \leq \|\tilde{\nu}\|_{\text{TV}} \|h\|_{L^2(\mathcal{Z}, \tau_{\mathcal{Z}})} \sup_{y \in \mathcal{Y}} \frac{|\xi_2(y)|}{\xi_1(y)} < +\infty
\end{split}
\end{align}
As a safety check, notice that when $h \in L^2(\mathcal{Z})$, we have that $\int_{\mathcal{Z}} \phi(\cdot,z) h(z) \ d\tau_{\mathcal{Z}}(z)$ indeed belongs to $\mathcal{M}_{\xi_1}^{*}(\mathcal{Y})$, because
\begin{align}
\int_{\mathcal{Z}} \left|\frac{\phi(y,z) h(z)}{\xi_1(y)} \right| \ d\tau_{\mathcal{Z}}(z) \leq \left(\int_{\mathcal{Z}} \left|\frac{\phi(y,z)}{\xi_1(y)} \right|^2 \ d\tau_{\mathcal{Z}}(z) \right)^{1/2} \|h\|_{L^2(\mathcal{Z}, \tau_{\mathcal{Z}})} \leq \sup_{y \in \mathcal{Y}} \frac{|\xi_2(y)|}{\xi_1(y)} \|h\|_{L^2(\mathcal{Z}, \tau_{\mathcal{Z}})}
\end{align}
is uniformly bounded over $y \in \mathcal{Y}$ and thus $\int_{\mathcal{Z}} \phi(\cdot,z) h(z) \ d\tau_{\mathcal{Z}}(z) \in C_{b, \xi_1}(\mathcal{Y}) \subseteq \mathcal{M}_{\xi_1}^{*}(\mathcal{Y})$.
\end{proof}

\begin{lemma} \label{lem:gamma_star}
(i) Let $F$ as defined in equation \eqref{eq:F_def}, $A^{*}$ as defined in \eqref{eq:A_star} and $h^{\star}$ as in \eqref{eq:fenchel_weak}. Then, the unique $\nu^{\star} = \argmax_{\nu \in \mathcal{M}_{\xi_1}(\mathcal{Y})} \left\{\langle - A^{*}h^{\star}, \nu \rangle - F(\nu) \right\}$ satisfies
\begin{align}
    \frac{d\nu^{\star}}{d\tau_{\mathcal{Y}}}(y) = \frac{1}{Z_{\nu^{\star}}}\exp \left( - \beta \int_{\mathcal{Z}} \phi(y, z) \ h^{\star}(z) d\tau_{\mathcal{Z}}(z) \right).
\end{align}
and we also have that $\nu^{\star} = \argmax_{\nu \in \mathcal{P}(\mathcal{Y})} \left\{\langle - A^{*}h^{\star}, \nu \rangle - F(\nu) \right\}$.

(ii) We also have that $\nu^{\star} = \argmin_{\nu \in \mathcal{P}(\mathcal{Y})} \{F(\nu)+G(A\nu) \}$.
\end{lemma}
\begin{proof}
First, notice that 
\begin{align}
    \argmax_{\nu \in \mathcal{M}_{\xi_1}(\mathcal{Y})} \left\{\langle - A^{*}h^{\star}, \nu \rangle - F(\nu) \right\} = \argmax_{\nu \in \mathcal{M}_{\xi_1}(\mathcal{Y}) \cap \mathcal{P}(\mathcal{Y})} \left\{\langle - A^{*}h^{\star}, \nu \rangle - F(\nu) \right\}
\end{align}
because $F(\nu) = + \infty$ when $\nu \notin \mathcal{P}(\mathcal{Y})$.
Since the KL-divergence is strictly convex, this problem \eqref{eq:strong_duality_cond} has a unique solution $\nu^{\star}$. 

Now, define $\nu_1 \in \mathcal{P}(\mathcal{Y})$ with density  $\frac{d\nu_1^{\star}}{d\tau_{\mathcal{Y}}}(y) = \frac{1}{Z_{\nu_1^{\star}}}\exp \left( - \beta \int_{\mathcal{Z}} \phi(y, z) \ h^{\star}(z) d\tau_{\mathcal{Z}}(z) \right)$ (the following arguments show that indeed this measure is normalizable).

Consider the relaxation
\begin{align} \label{eq:strong_duality_cond2}
    \argmax_{\nu \in \mathcal{P}(\mathcal{Y})} \left\{\langle - A^{*}h^{\star}, \nu \rangle - F(\nu) \right\}.
\end{align}
This problem is strictly convex (because the KL divergence is) and it has at most one solution, which is the unique solution of an Euler-Lagrange equation. This Euler-Lagrange equation is satisfied by $\nu_1^{\star}$, hence $\nu_1^{\star} = \argmax_{\nu \in \mathcal{P}(\mathcal{Y})} \left\{\langle - A^{*}h^{\star}, \nu \rangle - F(\nu) \right\}.$

We will see that $\nu_1^{\star}$ belongs to $\mathcal{M}_{\xi_1}(\mathcal{Y})$, which implies that $\nu_1^{\star} = \nu^{\star}$.
Remark that problem \eqref{eq:strong_duality_cond2} has Euler-Lagrange condition $\frac{d\nu_1^{\star}}{d\tau_{\mathcal{Y}}}(y) = \frac{1}{Z_{\nu_1^{\star}}}\exp \left( - \beta \int_{\mathcal{Z}} \phi(y, z) \ h^{\star}(z) d\tau_{\mathcal{Z}}(z) \right)$. 
Next, notice that by the Cauchy-Schwarz inequality and the definition of $\xi_2$ in \autoref{ass:phi}(iii), 
\begin{align} \label{eq:cs_bound}
    \left| - \beta \int \phi(y, z) \ h^{\star}(z) d\tau_{\mathcal{Z}}(z) \right| \leq \beta \left( \int_{\mathcal{Z}} \phi(y, z)^2 d\tau_{\mathcal{Z}}(z) \right)^{1/2} \left( \int h^{\star}(z)^2 d\tau_{\mathcal{Z}}(z) \right)^{1/2} = \beta \xi_2(y) \|h^{\star}\|_{L^2}.
\end{align}
Thus, $- \beta \int \phi(\cdot, z) \ h^{\star}(z) d\tau_{\mathcal{Z}}(z) \in C_{b,\xi_1}(\mathcal{Y})$, and in analogy with \eqref{eq:tv_norm_bound},
\begin{align}
\begin{split}
    &\|\nu_1^{\star}\|_{\text{TV}, \xi_1} = \int_{\mathcal{Y}} \xi_1(y) \ d\nu_1^{\star}(y) \\ &= \int_{\mathcal{Y}} \exp \left( \log(\xi_1(y)) - \beta \int \phi(y, z) \ h^{\star}(z) d\tau_{\mathcal{Z}}(z) - \log Z_{\nu_1^{\star}} + \log\left(\frac{d\tau_{\mathcal{Y}}}{d\lambda}(y) \right) \right) d\lambda(y) \\ &\leq \int_{\mathcal{Y}} \exp \left( \log(\xi_1(y)) + \beta \xi_2(y) \|h^{\star}\|_{L^2} - \log Z_{\nu_1^{\star}} + \log\left(\frac{d\tau_{\mathcal{Y}}}{d\lambda}(y) \right) \right) d\lambda(y) 
    \\ &\leq \int_{\R_{+}} \int_{\mathbb{S}^{d-1}} K_2 K_1 r^{-1-\epsilon} \ dr d\lambda(\theta) = \frac{K_2 K_1 \text{vol}(\mathbb{S}^{d-1})}{\epsilon}.
\end{split}
\end{align}
In the second equality we used the Euler-Lagrange condition, in the first inequality we used equation \eqref{eq:cs_bound} and in the second inequality we skipped a step which proceeds as in \eqref{eq:tv_norm_bound}; the key point is that $\beta \xi_2(\cdot) \|h^{\star}\|_{L^2}$ is $O(\xi_1)$ by \autoref{ass:phi}(iii) and thus $\exp \bigg( \log(\xi_1(y)) + \beta \xi_2(\cdot) \|h^{\star}\|_{L^2} - \log Z_{\nu_1^{\star}} + \log\left(\frac{d\tau_{\mathcal{Y}}}{d\lambda}(y) \right) + (d+\epsilon) \log \|y\|_2 \bigg) = o(1)$.

(ii) Consider the problem
\begin{align}
\argmin_{\nu \in \mathcal{P}(\mathcal{Y})} \beta^{-1} D_{KL}(\nu||\tau_{\mathcal{Y}}) + \left( \int_{\mathcal{Z}} \left( \int_{\mathcal{Y}} \phi(y, z) \ d\nu(y) - g(z) \right)^2 \ d\tau_{\mathcal{Z}}(z) \right)^{1/2}. 
\end{align}
If it exists, the unique solution $\nu_2^{\star} \in \mathcal{P}(\mathcal{Y})$ of this problem is the unique solution of the following Euler-Lagrange condition:
\begin{align}
\begin{cases} \label{eq:nu_2_star_euler}
    \frac{d\nu_2^{\star}}{d\tau_{\mathcal{Y}}}(y) = \exp \left(- \frac{\beta \int_{\mathcal{Z}} \phi(y, z) ( \int_{\mathcal{Y}} \phi(y', z) \ d\nu_2^{\star}(y') - g(z) ) \ d\tau_{\mathcal{Z}}(z)}{( \int_{\mathcal{Z}} ( \int_{\mathcal{Y}} \phi(y', z) \ d\nu_2^{\star}(y') - g(z) )^2 \ d\tau_{\mathcal{Z}}(z) )^{1/2}} \right) &\text{ if } \int_{\mathcal{Y}} \phi(y', \cdot) \ d\nu_2^{\star}(y') \neq g(\cdot) \\
    g(\cdot) = \int_{\mathcal{Y}} \phi(y', \cdot) \ d\nu_2^{\star}(y')  &\text{ otherwise}.
\end{cases}
\end{align}
Going back to \eqref{eq:fenchel_weak}, we observe that for strong duality hold we must have 
\begin{align} \label{eq:A_nu_strong_duality}
A\nu^{\star} = \argmax_{\psi \in L^2(\mathcal{Z})} \left\{\langle h^{\star}, \psi \rangle - G(\psi) \right\} = \argmax_{\psi \in L^2(\mathcal{Z})} \left\{\langle h^{\star}, \psi \rangle - \|\psi - g\|_{L^2(\mathcal{Z})} \right\}
\end{align}
The Euler-Lagrange condition for $\argmax_{\psi \in L^2(\mathcal{Z})} \left\{\langle h^{\star}, \psi \rangle - \|\psi - g\|_{L^2(\mathcal{Z})} \right\}$ is:
\begin{align}
    h^{\star} - \frac{\psi - g}{\|\psi - g\|_{L^2(\mathcal{Z})}} = 0,
\end{align}
which in the case $\|h^{\star}\|_{L^2(\mathcal{Z})} \neq 1$ implies that $\psi = g$. Thus, for \eqref{eq:A_nu_strong_duality} to hold we must have either $h^{\star} = \frac{A\nu^{\star} - g}{\|A\nu^{\star} - g\|_{L^2(\mathcal{Z})}}$ or $A\nu = g$. In either of the two cases, using part (i) we see that $\nu^{\star}$ satisfies \eqref{eq:nu_2_star_euler}, which means that $\nu^{\star} = \nu_2^{\star}$.
\end{proof}

\begin{theorem}
\label{thm:dualityfonefirst}
Let $\eta_1 : \mathcal{Z} \to \R$ be a strictly positive function such that $\sup_{y \in \mathcal{Y}} \phi(y,z)/\xi_1(y) = o(\eta_1(z))$ and $g(z) = o(\eta_1(z))$ as $z \rightarrow \infty$. Let $\mathcal{M}_{\eta_1}(\mathcal{Z})$ be the space of (countably additive) signed Radon measures $\gamma$ over $\mathcal{Z}$ such that $\|\gamma\|_{\text{TV},\eta_1} := \int_{\mathcal{Z}} \eta_1(z) \ d|\gamma|(z)$ is finite. Consider the problem 
\begin{align}
\begin{split} \label{eq:general_primal_problem_f1_2}
    \min_{\nu \in \mathcal{P}(\mathcal{Y})} &\beta^{-1} D_{KL}(\nu||\tau_{\mathcal{Y}}) + \left\| \frac{1}{\eta_1(\cdot)} \left( \int_{\mathcal{Y}} \phi(y, \cdot) \ d\nu(y) - g(\cdot) \right) \right\|_{L^{\infty}}.
\end{split}
\end{align}
and the problem
\begin{align}
\begin{split} \label{eq:general_dual_problem_f1_2}
    &\max_{\substack{\gamma \in \mathcal{M}_{\eta_1}(\mathcal{Z}) \\ \|\gamma\|_{\text{TV},\eta_1} \leq 1}} - \int_{\mathcal{Z}} g(z) \ d\gamma(z) - \frac{1}{\beta} \log\left(\int_{\mathcal{Y}} \exp \left(-\beta \int_{\mathcal{Z}} \phi(y, z) \ d\gamma(z) \right) d\tau_{\mathcal{Y}}(y) \right)
\end{split}
\end{align}
The two problems \eqref{eq:general_primal_problem_f1_2} and \eqref{eq:general_dual_problem_f1_2} are convex. The problem \eqref{eq:general_dual_problem_f1_2} is the dual problem of \eqref{eq:general_primal_problem_f1_2}. Moreover, the solution $\nu^{\star}$ of \eqref{eq:general_primal_problem_f1_2} is unique and its density satisfies
\begin{align}
    \frac{d\nu^{\star}}{d\tau_{\mathcal{Y}}}(y) = \frac{1}{Z_{\beta}}\exp \left( - \beta \int \phi(y, z) \ d\gamma^{\star}(z) \right),
\end{align}
where $\gamma^{\star}$ is a solution of \eqref{eq:general_dual_problem_f1_2} and $Z_{\beta}$ is a normalization constant. 
\end{theorem}

\begin{proof}
We apply \autoref{thm:fenchel}.

As in the proof of \autoref{thm:duality_general}, we set $X = \mathcal{M}_{\xi_1}(\mathcal{Y})$, which is the Banach space of Radon measures over $\mathcal{Y}$ such that the weighted total variation $\|\nu\|_{\text{TV}, \xi_1} := \int_{\mathcal{Y}} \xi_1(y) \ d|\nu|(y)$ is finite, and whose continuous dual $\mathcal{M}_{\xi_1}^{*}(\mathcal{Y})$ contains the set $C_{b, \xi_1}(\mathcal{Y})$ of continuous functions $f$ such that $f/\xi_1 \in C_b(\mathcal{Y})$.

Unlike in the proof of \autoref{thm:duality_general}, we set $Y = C_{0, \eta_1}(\mathcal{Z})$, which we define to be the space of continuous functions $f : \mathcal{Z} \to \R$ such that $\lim_{\|z\| \rightarrow \infty} f(z)/\eta_1(z) = 0$. By \autoref{lem:c_b_eta1}, $C_{0, \eta_1}(\mathcal{Z})$ is a Banach space endowed with the norm $\|f\|_{C_{0, \eta_1}} = \sup_{z \in \mathcal{Z}} f(z)/\eta_1(z)$, and we have that the continuous dual space $C_{0, \eta_1}^{*}(\mathcal{Z})$ is equal to the set $\mathcal{M}_{\eta_1}(\mathcal{Z})$ of Radon measures $\gamma$ over $\mathcal{Z}$ such that $\|\gamma\|_{\text{TV},\eta_1} := \int_{\mathcal{Z}} \eta_1(z) \ d|\gamma|(z)$ is finite.

$F: \mathcal{M}_{\xi_1}(\mathcal{Y}) \to \R \cup \{+\infty \}$ and its convex conjugate $F^{*}: \mathcal{M}_{\xi_1}^{*}(\mathcal{Y}) \to \R \cup \{+\infty \}$ are as specified in equations \eqref{eq:F_def}-\eqref{eq:F_conj} in the proof of \autoref{thm:duality_general}.

Define $G : C_{0, \eta_1}(\mathcal{Z}) 
\to \R \cup \{+\infty \}$ as
\begin{align} \label{eq:G_def2}
    G(\psi) = \sup_{\substack{\gamma \in \mathcal{M}_{\eta_1}(\mathcal{Z}) \\ \|\gamma\|_{\text{TV},\eta_1} \leq 1
    }} \int_{\mathcal{Z}} (\psi(z) - g(z)) \ d\gamma(z), 
\end{align}
which by \autoref{lem:G2_conj} can also be written as
\begin{align}
	G(\psi) = \sup_{z \in \mathcal{Z}} \frac{|\psi(z) - g(z)|}{\eta_1(z)}.
\end{align}
Also by \autoref{lem:G2_conj}, the convex conjugate 
$G^{*} : \mathcal{M}_{\eta_1}(\mathcal{Z}) \to \R$
is of the form 
\begin{align}
G^{*}(\gamma) = 
\begin{cases}
\int_{\mathcal{Z}} g(z) \ d\gamma(z) &\text{if } \|\gamma\|_{\text{TV}, \eta_1} \leq 1 \\
+\infty &\text{otherwise}
\end{cases}
\end{align}
The linear operator $A : \mathcal{M}_{\xi_1}(\mathcal{Y}) \to C_{0, \eta_1}(\mathcal{Z})$ is defined as $(A \nu)(z) = \int_{\mathcal{Y}} \phi(y,z) \ d\nu(y)$. It is well defined and continuous by \autoref{lem:A_well_2}. \autoref{lem:A_well_2} also states that the adjoint operator $A^{*} : \mathcal{M}_{\eta_1}(\mathcal{Z}) \rightarrow \mathcal{M}_{\xi_1}^{*}(\mathcal{Y})$ is $(A^{*} \gamma)(y) = \int_{\mathcal{Z}} \phi(y,z) \ d\gamma(z)$.
Hence, we have that problem \eqref{eq:general_primal_problem_f1_2} can be written as
\begin{align} \label{eq:primal2_2}
    p^{*} = \inf_{\nu \in \mathcal{M}_{\xi_1}(\mathcal{Y})} \{F(\nu)+G(A\nu) \}.
\end{align}
And problem \eqref{eq:general_dual_problem_f1_2} can be written as
\begin{align} \label{eq:dual2_2}
    d^{*} = \sup_{\substack{\gamma \in \mathcal{M}_{\eta_1}(\mathcal{Z}) \\ \|\gamma\|_{\text{TV},\eta_1} \leq 1}} \{-F^{*}(-A^{*} \gamma)-G^{*}(\gamma)\}.
\end{align}
To apply \autoref{thm:fenchel}, it only remains to show that condition 2 holds. That is, we have to check that $A \ \text{dom} \ F \cap \text{cont} \ G \neq \emptyset$. Consider $\psi = A \nu$ for some $\nu \in \mathcal{P}(\mathcal{Y})$ absolutely continuous w.r.t. $\tau_{\mathcal{Y}}$. Then $\psi \in A \ \text{dom} \ F$. Moreover, since $G$ is a continuous functional, we have that $\text{cont} \ G = C_{0, \eta_1}(\mathcal{Z})$. Thus, $\psi$ also belongs to $\text{cont} \ G$ and we conclude that $A \ \text{dom} \ F \cap \text{cont} \ G \neq \emptyset$. 

By \autoref{thm:fenchel}, $p^{*} = d^{*}$, and since $p^{*}$ is finite, we have that the supremum in \eqref{eq:dual2_2} is attained; let $\gamma^{\star}$ be one maximizer. As in the proof of \autoref{thm:duality_general}, we prove the existence of a minimizer $\nu^{\star}$ of $p^{*}$ by the direct method of the calculus of variations, and the link between $\gamma^{\star}$ and $\nu^{\star}$ is analogous.
\end{proof}

\begin{lemma} \label{lem:c_b_eta1}
Let $C_{0, \eta_1}(\mathcal{Z})$ be the vector space of functions $f : \mathcal{Z} \to \R$ such that $\lim_{\|z\| \rightarrow \infty} f(z)/\eta_1(z) = 0$. $C_{0, \eta_1}(\mathcal{Z})$ is a Banach space with norm $\|f\|_{C_{0, \eta_1}} = \sup_{z \in \mathcal{Z}} f(z)/\eta_1(z)$. The continuous dual space $C_{0, \eta_1}^{*}(\mathcal{Z})$ is equal to the set $\mathcal{M}_{\eta_1}(\mathcal{Z})$ of Radon measures $\gamma$ over $\mathcal{Z}$ such that $\|\gamma\|_{\text{TV},\eta_1} := \int_{\mathcal{Z}} \eta_1(z) \ d|\gamma|(z)$ is finite.
\end{lemma}

\begin{proof}
Define the linear map $\hat{\eta_1} : C_{0, \eta_1}(\mathcal{Z}) \to C_{0}(\mathcal{Z})$ as $f \mapsto \tilde{\eta_1}(f):=f/\eta_1$, where $C_0(\mathcal{Z})$ is the Banach space of continuous functions on $\mathcal{Z}$ vanishing at infinity, endowed with the supremum norm. Notice that for all $f \in C_{0, \eta_1}(\mathcal{Z})$, we have that $\|\tilde{\eta_1}(f)\|_{C_0} = \|f\|_{C_{0, \eta_1}}$. Remark also that $\tilde{\eta_1}$ is surjective, because if $\tilde{f} \in C_0(\mathcal{Z})$, there exists $f := \eta_1 \tilde{f}$ such that $\tilde{\eta_1}(f) = \tilde{f}$. Thus, $\tilde{\eta_1}$ is a surjective isometry between $C_{0, \eta_1}(\mathcal{Z})$ and $C_{0}(\mathcal{Z})$, which shows that $C_{0, \eta_1}(\mathcal{Z})$ is a Banach space. 

And in analogy with \autoref{lem:m_xi1}, the linear mapping $\tilde{\eta_1} : \mathcal{M}_{\eta_1}(\mathcal{Z}) \to \mathcal{M}(\mathcal{Z})$ defined as $\gamma \mapsto \tilde{\gamma}$ such that $\frac{d\tilde{\gamma}}{d\gamma}(z) = \eta_1(z)$ is a surjective isometry.
To show that $C_{0, \eta_1}^{*}(\mathcal{Z})$ is $\mathcal{M}_{\eta_1}(\mathcal{Z})$, we will show both inclusions. Given $\gamma \in \mathcal{M}_{\eta_1}(\mathcal{Z})$, for any $f \in C_{0, \eta_1}(\mathcal{Z})$ we have that 
\begin{align}
\begin{split}
\int_{\mathcal{Z}} f(z) \ d\gamma(z) &= \int_{\mathcal{Z}} \frac{f(z)}{\eta_1(z)} \eta_1(z) \ d\gamma(z) = \int_{\mathcal{Z}} \hat{\eta_1}(f)(z) \ d\tilde{\eta_1}(\gamma)(z) \\ &\leq \|\hat{\eta_1}(f)\|_{C_0} \|\tilde{\eta_1}(\gamma)\|_{\text{TV}} = \|f\|_{C_{0, \eta_1}} \|\gamma\|_{\text{TV}, \eta_1}.
\end{split}
\end{align}
Thus, $\mathcal{M}_{\eta_1}(\mathcal{Z}) \subseteq C_{0, \eta_1}^{*}(\mathcal{Z})$. Conversely, let $\gamma \in C_{0, \eta_1}^{*}(\mathcal{Z})$. Since $\hat{\eta_1}^{-1} : C_{0}(\mathcal{Z}) \to C_{0, \eta_1}(\mathcal{Z})$ is a surjective isometry, we have that $\gamma \circ \hat{\eta_1}^{-1} \in C_{0}^{*}(\mathcal{Z})$. By the Riesz-Markov-Kakutani theorem (\autoref{thm:rieszmarkov}) the continuous dual space $C_{0}^{*}(\mathcal{Z})$ is the Banach space $\mathcal{M}(\mathcal{Z})$ of finite signed Radon measures with norm $\|\cdot \|_{\text{TV}}$. Thus, there exists $\tilde{\gamma} \in \mathcal{M}(\mathcal{Z})$ such that for any $\hat{f} \in C_{0}(\mathcal{Z})$, $\int_{\mathcal{Z}} \hat{f}(z) \ d\tilde{\gamma}(z) = \langle \gamma \circ \hat{\eta_1}^{-1}, \hat{f} \rangle$. Since $\langle \gamma \circ \hat{\eta_1}^{-1}, \hat{f} \rangle = \langle \gamma, \hat{\eta_1}^{-1}(\hat{f}) \rangle$ and $\int_{\mathcal{Z}} f(z) \ d\tilde{\gamma}(z) = \int_{\mathcal{Z}} \hat{f}(z) \eta_1(z) \frac{1}{\eta_1(z)} \ d\tilde{\gamma}(z) = \int_{\mathcal{Z}} \hat{\eta_1}^{-1}(\hat{f})(z) \ d\tilde{\eta_1}^{-1}(\tilde{\gamma})(z)$, we have that
\begin{align}
\forall f \in C_{0, \eta_1}(\mathcal{Z}), \quad  \int_{\mathcal{Z}} f(z) \ d\tilde{\eta_1}^{-1}(\tilde{\gamma})(z) = \langle \gamma, f \rangle,
\end{align}
proving that $C_{0, \eta_1}^{*}(\mathcal{Z}) \subseteq \mathcal{M}_{\eta_1}(\mathcal{Z})$.  
\end{proof}

\begin{lemma} \label{lem:G2_conj}
$G : C_{0,\eta_1}(\mathcal{Z}) \to \R \cup \{+\infty \}$ given by equation \eqref{eq:G_def} is convex. Its convex conjugate $G^{*} : \mathcal{M}_{\eta_1}(\mathcal{Z}) \to \R \cup \{+\infty\}$ is given by equation \eqref{eq:G_conj}.
\end{lemma}

\begin{proof}
$G$ is convex because it is the supremum of linear functions. 
$G^{*} : \mathcal{M}_{\eta_1}(\mathcal{Z}) \to \R$ is defined as
\begin{align}
\begin{split} \label{eq:G_star_2}
G^{*}(\gamma) &= \sup_{\psi \in C_{0, \eta_1}(\mathcal{Z})} \left\{ \int_{\mathcal{Z}} \psi(z) \ d\gamma(z) - \sup_{\substack{\gamma' \in \mathcal{M}_{\eta_1}(\mathcal{Z}) \\ \|\gamma'\|_{\text{TV},\eta_1} \leq 1}} \int_{\mathcal{Z}} (\psi(z) - g(z)) \ d\gamma'(z) \right\} \\ &= \sup_{\psi \in C_{0, \eta_1}(\mathcal{Z})} \inf_{\substack{\gamma' \in \mathcal{M}_{\eta_1}(\mathcal{Z}) \\ \|\gamma'\|_{\text{TV},\eta_1} \leq 1}} \left\{ \int_{\mathcal{Z}} \psi(z) \ d(\gamma-\gamma')(z) + \int_{\mathcal{Z}} g(z) \ d\gamma'(z) \right\}
\end{split}
\end{align}
At this point, we want to apply \autoref{thm:kneser} in a similar fashion to the proof of \autoref{lem:G_conj}. In this case, we set $A = \mathcal{B}_{\mathcal{M}_{\eta_1}(\mathcal{Z})} \subseteq E = \mathcal{M}_{\eta_1}(\mathcal{Z})$ and $B = F = \mathcal{M}_{\eta_1}(\mathcal{Z})$. We can endow $B$ with the strong (or norm) topology, but $A$ requires a weaker topology that makes it compact. 
Since $\mathcal{M}_{\eta_1}(\mathcal{Z})$ is the continuous dual of $C_{0,\eta_1}(\mathcal{Z})$, we endow $A$ with the weak-* topology, which by the Banach-Alaoglu theorem (\autoref{thm:banachalaoglu}) makes it compact. We have that $(\gamma', \psi) \mapsto H(\gamma', \psi) =  - \int_{\mathcal{Z}} \psi(z) \ d(\gamma-\gamma')(z) - \int_{\mathcal{Z}} g(z) \ d\gamma'(z)$ is concave in $\gamma' \in A$ and convex in $\psi \in B$ because it is affine in both variables. $H$ is continuous in $\psi$ because $\int_{\mathcal{Z}} \psi(z) \ d(\gamma-\gamma')(z) = \int_{\mathcal{Z}} \frac{\psi(z)}{\eta_1(z)} \eta_1(z) \ d(\gamma-\gamma')(z) \leq \|\psi\|_{C_{0,\eta_1}} \|\gamma-\gamma'\|_{\text{TV}, \eta_1}.$
$H$ is continuous in $\gamma'$ in the weak-* topology, because it is precisely the weakest one that makes maps of the form $\gamma' \mapsto \int_{\mathcal{Z}} (g(z) - \psi(z)) \ d\gamma'(z)$ continuous. Thus, $\sup_{\gamma' \in A} \inf_{\psi \in B} H(\gamma', \psi) = \sup_{\gamma' \in A} \inf_{\psi \in B} H(\gamma', \psi)$, and flipping the signs, the right-hand side of \eqref{eq:G_star_2} is equal to:
\begin{align}
\begin{split}
\\ &\inf_{\substack{\gamma' \in \mathcal{M}_{\eta_1}(\mathcal{Z}) \\ \|\gamma'\|_{\text{TV},\eta_1} \leq 1}} \sup_{\psi \in C_{0, \eta_1}(\mathcal{Z})} \left\{ \int_{\mathcal{Z}} \psi(z) \ d(\gamma-\gamma')(z) + \int_{\mathcal{Z}} g(z) \ d\gamma'(z) \right\}
\\ &= 
\begin{cases}
\int_{\mathcal{Z}} g(z) \ d\gamma(z) &\text{if } \|\gamma\|_{\text{TV},\eta_1} \leq 1 \\
+\infty &\text{otherwise}
\end{cases}.
\end{split} 
\end{align}
\end{proof}

\begin{lemma} \label{lem:A_well_2}
The linear operator $A : \mathcal{M}_{\xi_1}(\mathcal{Y}) \to C_{0, \eta_1}(\mathcal{Z})$ defined as $(A \nu)(z) = \int_{\mathcal{Y}} \phi(y,z) \ d\nu(y)$ is well defined and continuous. Its operator norm is upper bounded by $\sup_{y \in \mathcal{Y}} |\xi_2(y)|/\xi_1(y)$, where $\xi_1$ and $\xi_2$ are defined in \autoref{ass:phi}. 
Moreover, the adjoint operator $A^{*} : \mathcal{M}_{\eta_1}(\mathcal{Z}) \rightarrow \mathcal{M}_{\xi_1}^{*}(\mathcal{Y})$ is $(A^{*} \gamma)(y) = \int_{\mathcal{Z}} \phi(y,z) \ d\gamma(z)$.
\end{lemma}

\begin{proof}
Remark that $\int_{\mathcal{Y}} \phi(y, \cdot) \ d\nu(y)$ does belong to $C_{0, \eta_1}(\mathcal{Z})$ because 
\begin{align}
    \lim_{\|z\| \rightarrow \infty} \frac{\int_{\mathcal{Y}} \phi(y, z) \ d\nu(y)}{\eta_1(z)} = \int_{\mathcal{Y}} \lim_{\|z\| \rightarrow \infty} \frac{\phi(y,z)}{\eta_1(z)} \ d\nu(y) = 0
\end{align}
The second equality follows from the assumption that $\phi(y,z) = o(\eta_1(z))$ for all $y \in \mathcal{Y}$. The first equality holds by the dominated convergence theorem, which can be applied because the integral of the absolute value can be uniformly dominated for all $z \in \mathcal{Z}$:
\begin{align}
\begin{split} \label{eq:A_continuous}
    \int_{\mathcal{Y}} \left| \frac{\phi(y,z)}{\eta_1(z)} \right| \ d|\nu|(y) &= \int_{\mathcal{Y}} \left| \frac{\phi(y,z)}{\eta_1(z) \xi_1(y)} \right| \xi_1(y) \ d|\nu|(y) \leq \|\nu\|_{\text{TV}, \xi_1} \sup_{y \in \mathcal{Y}} \frac{\phi(y,z)}{\eta_1(z) \xi_1(y)} \\ &\leq \|\nu\|_{\text{TV}, \xi_1} K,
\end{split}
\end{align}
for some constant $K$. In the first inequality we used the definition of $\|\cdot \|_{\text{TV}, \xi_1}$. In the last inequality we used that $\sup_{y \in \mathcal{Y}} \phi(y,z)/\xi_1(y) = o(\eta_1(z))$ as $\|z\| \to \infty$ by the definition of $\eta_1$. Equation \eqref{eq:A_continuous} also proves that $A$ is continuous, because $\|A \nu\|_{C_{0, \eta_1}} = \sup_{z \in \mathcal{Z}} |\int_{\mathcal{Y}} \phi(y,z) \ d\nu(y)|/\eta_1(z)$.

We have that $A^{*} : \mathcal{M}_{\eta_1}(\mathcal{Z}) \rightarrow \mathcal{M}_{\xi_1}^{*}(\mathcal{Y})$ is defined as $(A^{*} \gamma)(y) = \int_{\mathcal{Z}} \phi(y,z) \ d\gamma(z)$, because
\begin{align}
\begin{split} \label{eq:adjoint_exp_2}
    \int_{\mathcal{Z}} (A \nu)(z) \ d\gamma(z) &= \int_{\mathcal{Z}} \int_{\mathcal{Y}} \phi(y,z) \ d\nu(y) \ d\gamma(z) = \int_{\mathcal{Y}} \int_{\mathcal{Z}} \phi(y,z) \ d\gamma(z) \ d\nu(y).
\end{split}
\end{align}
In the last equality we applied Fubini's theorem, which holds because 
\begin{align}
\begin{split}
    &\int_{\mathcal{Z}} \int_{\mathcal{Y}} |\phi(y,z)| \ d|\nu|(y) \ d|\gamma|(z) = \int_{\mathcal{Z}} \int_{\mathcal{Y}} \frac{|\phi(y,z)|}{\eta_1(z) \xi_1(y)} \xi_1(y) \ d|\nu|(y) \ \eta_1(z) \ d|\gamma|(z) \\ &\leq \|\nu\|_{\text{TV}, \xi_1} \|\gamma\|_{\text{TV}, \eta_1} \sup_{y \in \mathcal{Y}} \frac{\phi(y,z)}{\eta_1(z) \xi_1(y)} = \|\nu\|_{\text{TV}, \xi_1} \|\gamma\|_{\text{TV}, \eta_1} K,
\end{split}
\end{align}
for some constant $K$.
\end{proof}

\textbf{\textit{Proof of \autoref{thm:duality_general_f1}}}
The proof makes use of \autoref{thm:dualityfonefirst}. We choose $\eta_1$ to be in the family $C = \{ \eta_{r} : \mathcal{Z} \to \R, \ \eta_{r}(z) = \max\{\exp(\|z\|-r), 1\} \ | \ r \in (0,+\infty)\}$. First, we prove that
\begin{align}
\begin{split} \label{eq:general_dual_problem_f1_3}
    &\sup_{\eta_1 \in C} \max_{\substack{\gamma \in \mathcal{M}_{\eta_1}(\mathcal{Z}) \\ \|\gamma\|_{\text{TV},\eta_1} \leq 1}} - \int_{\mathcal{Z}} g(z) \ d\gamma(z) - \frac{1}{\beta} \log\left(\int_{\mathcal{Y}} \exp \left(-\beta \int_{\mathcal{Z}} \phi(y, z) \ d\gamma(z) \right) d\tau_{\mathcal{Y}}(y) \right) \\ &= \max_{\substack{\gamma \in \mathcal{M}(\mathcal{Z}) \\ \|\gamma\|_{\text{TV}} \leq 1}} - \int_{\mathcal{Z}} g(z) \ d\gamma(z) - \frac{1}{\beta} \log\left(\int_{\mathcal{Y}} \exp \left(-\beta \int_{\mathcal{Z}} \phi(y, z) \ d\gamma(z) \right) d\tau_{\mathcal{Y}}(y) \right).
\end{split}
\end{align}
The right-hand side is larger or equal than the left-hand side because for all $r \in (0,+\infty)$, we have that $\mathcal{M}_{\eta_{r}}(\mathcal{Z}) \subseteq \mathcal{M}(\mathcal{Z})$, as $\|\gamma\|_{\text{TV},\eta_{r}} \geq \|\gamma\|_{\text{TV}}$. 

\autoref{lem:dense_set}(i) states that $\{\mathcal{M}_{\eta_{r}}(\mathcal{Z}) \ | \ r \in (0,+\infty) \}$ is dense in $\mathcal{M}(\mathcal{Z})$ in the TV norm topology. 
\autoref{lem:dense_set}(ii) states that the objective functional of \eqref{eq:general_dual_problem_f1_3} is continuous in the TV norm topology. These two facts imply the equality in \eqref{eq:general_dual_problem_f1_3}. To show that the maximum is attained in the left-hand side, we apply \autoref{lem:dense_set}(iii). Let $\gamma^{\star}$ be a maximizer.

Second, we prove that
\begin{align}
\begin{split} \label{eq:general_primal_problem_f1_3}
    &\sup_{\eta_1 \in C} \min_{\nu \in \mathcal{M}_{\xi_1}(\mathcal{Y})} \beta^{-1} D_{KL}(\nu||\tau_{\mathcal{Y}}) + \left\| \frac{1}{\eta_1(\cdot)} \left( \int_{\mathcal{Y}} \phi(y, \cdot) \ d\nu(y) - g(\cdot) \right) \right\|_{L^{\infty}} \\ &= \min_{\nu \in \mathcal{M}_{\xi_1}(\mathcal{Y})} \beta^{-1} D_{KL}(\nu||\tau_{\mathcal{Y}}) + \left\| \int_{\mathcal{Y}} \phi(y, \cdot) \ d\nu(y) - g(\cdot) \right\|_{L^{\infty}}.
\end{split}
\end{align}
Remark that for all $\nu \in \mathcal{M}_{\xi_1}(\mathcal{Y})$, $\left\| \int_{\mathcal{Y}} \phi(y, \cdot) \ d\nu(y) - g(\cdot) \right\|_{L^{\infty}}$ is finite because
$g \in C_b(\mathcal{Z})$ and $\int_{\mathcal{Y}} \phi(y, \cdot) \ d\nu(y) \in C_b(\mathcal{Z})$ because
\begin{align} \label{eq:in_cb}
    \sup_{z \in \mathcal{Z}} \left|\int_{\mathcal{Y}} \phi(y, z) \ d\nu(y) \right| = \sup_{z \in \mathcal{Z}} \left|\int_{\mathcal{Y}} \frac{\phi(y, z)}{\xi_1(y)} \xi_1(y) \ d\nu(y) \right| \leq K \|\nu\|_{\text{TV}, \eta_1}.
\end{align}
Given $\delta > 0$, let $R > 0$ such that $\|\int_{\mathcal{Y}} \phi(y, \cdot) \ d\nu(y) - g(\cdot)\|_{L^{\infty}} - \sup_{z \in \mathcal{Z} \cap \mathcal{B}_{\R^{d_2}}(R)} |\int_{\mathcal{Y}} \phi(y, \cdot) \ d\nu(y) - g(\cdot)| \leq \delta$. Hence, for $r > R$, $\|\int_{\mathcal{Y}} \phi(y, \cdot) \ d\nu(y) - g(\cdot)\|_{L^{\infty}} - \|(\int_{\mathcal{Y}} \phi(y, \cdot) \ d\nu(y) - g(\cdot) )/\eta_{r}(\cdot) \|_{L^{\infty}} \leq \delta$. Thus, equality holds in \eqref{eq:general_primal_problem_f1_3}. By the direct method of the calculus of variations (see the proof of \autoref{thm:duality_general}), we have that a minimizer $\nu^{\star}$ for the right-hand side of \eqref{eq:general_primal_problem_f1_3} exists.

Applying \autoref{thm:dualityfonefirst} on the right-hand sides of equations \eqref{eq:general_dual_problem_f1_3} and \eqref{eq:general_primal_problem_f1_3}, we see that they are equal. Thus, the left-hand sides are equal. Let us set $F$ and $F^{*}$ as in the proof of \autoref{thm:dualityfonefirst}. Let us set $G : C_b(\mathcal{Y}) \to \R$ as 
\begin{align} \label{eq:G_def_3}
G(\psi) = \| \psi(\cdot) - g(\cdot) \|_{L^{\infty}} = \sup_{\gamma' \in \mathcal{M}(\mathcal{Z}), \|\gamma'\|_{\text{TV}} \leq 1} \int_{\mathcal{Z}} (\psi(z) - g(z)) \ d\gamma'(z),
\end{align}
and define $\tilde{G} : \mathcal{M}(\mathcal{Z}) \to \R$ as 
\begin{align}
    \begin{cases} \label{eq:tilde_G}
\int_{\mathcal{Z}} g(z) \ d\gamma(z) &\text{if } \|\gamma\|_{\text{TV}} \leq 1 \\
+\infty &\text{otherwise}
\end{cases}.
\end{align}
By \autoref{lem:G_tildeG}, for any $\gamma \in \mathcal{M}(\mathcal{Z})$, we have $\tilde{G}(\gamma) = \sup_{\psi \in C_{b}(\mathcal{Z})} \left\{\langle \gamma, \psi \rangle - G(\psi) \right\}$. 

Then, the equality between the left-hand sides of \eqref{eq:general_dual_problem_f1_3} and \eqref{eq:general_primal_problem_f1_3} can be rewritten as $\sup_{\gamma \in \mathcal{M}(\mathcal{Z}), \|\gamma\|_{\text{TV}} \leq 1} \{-F^{*}(-A^{*} \gamma)-\tilde{G}(\gamma)\} = \inf_{\nu \in \mathcal{M}_{\xi_1}(\mathcal{Y})} \{F(\nu)+G(A\nu)\}$.
We reproduce the argument of 
\eqref{eq:fenchel_weak}
and we conclude that 
\begin{align}
\begin{split}
    &\sup_{\substack{\gamma \in \mathcal{M}(\mathcal{Z}) \\ \|\gamma\|_{\text{TV}} \leq 1}} \{-F^{*}(-A^{*} \gamma)-G^{*}(\gamma)\} = -F^{*}(-A^{*} \gamma^{\star})-G^{*}(\gamma^{\star}) \\&= 
    - \sup_{\nu \in \mathcal{M}_{\xi_1}(\mathcal{Y})} \left\{\langle - A^{*}\gamma^{\star}, \nu \rangle - F(\nu) \right\} - \sup_{\psi \in C_{b}(\mathcal{Z})} \left\{\langle \gamma^{\star}, \psi \rangle - G(\psi) \right\}
    \\&\leq - \sup_{\nu \in \mathcal{M}_{\xi_1}(\mathcal{Y})} \left\{\langle - A^{*}\gamma^{\star},\nu \rangle - F(\nu) + \langle \gamma^{\star}, A \nu \rangle - G(A \nu) \right\} \\&= - \sup_{\nu \in \mathcal{M}_{\xi_1}(\mathcal{Y})} \left\{\langle - F(\nu) - G(A \nu) \right\} = \inf_{\nu \in \mathcal{M}_{\xi_1}(\mathcal{Y})} \{F(\nu)+G(A\nu)\} \\ &= F(\nu^{\star})+G(A\nu^{\star})
\end{split}
\end{align}
In the first equality we used that $G^{*}(\gamma^{\star}) = \sup_{\psi \in C_{b}(\mathcal{Z})} \left\{\langle \gamma^{\star}, \psi \rangle - G(\psi) \right\}$, which holds because for $\gamma \in \mathcal{M}(\mathcal{Z})$,
\begin{align}
\begin{split}
    &\sup_{\psi \in C_{b}(\mathcal{Z})} \left\{ \int_{\mathcal{Z}} \psi(z) \ d\gamma(z) - \sup_{\substack{\gamma' \in \mathcal{M}(\mathcal{Z}) \\ \|\gamma'\|_{\text{TV}} \leq 1}} \int_{\mathcal{Z}} (\psi(z) - g(z)) \ d\gamma'(z) \right\} \\ &= \sup_{\psi \in C_{b}(\mathcal{Z})} \inf_{\substack{\gamma' \in \mathcal{M}(\mathcal{Z}) \\ \|\gamma'\|_{\text{TV}} \leq 1}} \left\{ \int_{\mathcal{Z}} \psi(z) \ d(\gamma-\gamma')(z) + \int_{\mathcal{Z}} g(z) \ d\gamma'(z) \right\}  
    \\ &= \inf_{\substack{\gamma' \in \mathcal{M}(\mathcal{Z}) \\ \|\gamma'\|_{\text{TV}} \leq 1}} \sup_{\psi \in C_{b}(\mathcal{Z})} \left\{ \int_{\mathcal{Z}} \psi(z) \ d(\gamma-\gamma')(z) + \int_{\mathcal{Z}} g(z) \ d\gamma'(z) \right\}
\\ &= 
\begin{cases}
\int_{\mathcal{Z}} g(z) \ d\gamma(z) &\text{if } \|\gamma\|_{\text{TV}} \leq 1 \\
+\infty &\text{otherwise}
\end{cases}.
\end{split}
\end{align}
The link between $\gamma^{\star}$ and $\nu^{\star}$ is analogous to the proof of \autoref{thm:duality_general} (see \autoref{lem:gamma_star}(i)). The fact that $\nu^{\star} = \argmin_{\nu \in \mathcal{P}(\mathcal{Y})} \{F(\nu)+G(A\nu) \}$ holds by an analogous reasoning.
\qed

\begin{lemma} \label{lem:dense_set}
(i) For any $r > 0$ let $\eta_{r} : \mathcal{Z} \to \R, \ \eta_{r}(z) = \max\{\exp(\|z\|-r), 1\}$. The set $\{\mathcal{M}_{\eta_{r}}(\mathcal{Z}) \ | \ r \in (0,+\infty) \}$ is dense in $\mathcal{M}(\mathcal{Z})$ in the TV norm topology. 

(ii) The functional $\gamma \mapsto - \int_{\mathcal{Z}} g(z) \ d\gamma(z) - \frac{1}{\beta} \log\left(\int_{\mathcal{Y}} \exp \left(-\beta \int_{\mathcal{Z}} \phi(y, z) \ d\gamma(z) \right) d\tau_{\mathcal{Y}}(y) \right)$ is continuous in the TV norm topology, and a fortiori, its first variation has bounded supremum norm. 

(iii) The functional $\gamma \mapsto - \int_{\mathcal{Z}} g(z) \ d\gamma(z) - \frac{1}{\beta} \log\left(\int_{\mathcal{Y}} \exp \left(-\beta \int_{\mathcal{Z}} \phi(y, z) \ d\gamma(z) \right) d\tau_{\mathcal{Y}}(y) \right)$ has a maximizer over $\mathcal{B}_{\mathcal{M}(Z)} = \{\gamma \in \mathcal{M}(\mathcal{Z}) \ | \ \|\gamma\|_{\text{TV}} \leq 1 \}$.
\end{lemma}

\begin{proof}
To prove (i), let $(r_n)_{n \geq 0}$ be a real sequence converging to $+\infty$. For any $\gamma \in \mathcal{M}(\mathcal{Z})$, we can build a sequence of measures $\gamma_{r_n} \in \mathcal{M}_{\eta_{r_n}}(\mathcal{Z})$ defined with density $\frac{d\gamma_{r_n}}{d\gamma}(z) = \min\{\exp(-\|z\|+r), 1\}$. For any $\delta > 0$, there exists $R > 0$ such that $\|\gamma\|_{\text{TV}} - \int_{\mathcal{Z} \cap \mathcal{B}_{\R^{d_2}}(R)} d|\gamma|(z) \leq \delta$. Notice that 
for all $r_n > R$,
\begin{align}
    \|\gamma - \gamma_{r_n}\|_{\text{TV}} \leq \int_{\mathcal{Z} \setminus \mathcal{B}_{\R^{d_2}}(r_n)} d|\gamma|(z) \leq \delta  
\end{align}

To prove (ii), notice that the first variation of the log-partition at $\gamma$ is the function 
\begin{align}
    z \mapsto \frac{- \beta \int_{\mathcal{Y}} \exp \left(-\beta \int_{\mathcal{Z}} \phi(y, z') \ d\gamma(z') \right) \phi(y, z) \ d\tau_{\mathcal{Y}}(y)}{\int_{\mathcal{Y}} \exp \left(-\beta \int_{\mathcal{Z}} \phi(y, z') \ d\gamma(z') \right) d\tau_{\mathcal{Y}}(y)},
\end{align}
which has supremum norm bounded by
\begin{align}
\frac{\beta K \int_{\mathcal{Y}} \xi_1(y) \exp \left(-\beta \int_{\mathcal{Z}} \phi(y, z') \ d\gamma(z') \right) d\tau_{\mathcal{Y}}(y)}{\int_{\mathcal{Y}} \exp \left(-\beta \int_{\mathcal{Z}} \phi(y, z') \ d\gamma(z') \right) d\tau_{\mathcal{Y}}(y)}. 
\end{align}
And this bound is finite because we can apply the argument of \eqref{eq:tv_norm_bound} with $q(y) = \int_{\mathcal{Z}} \phi(\cdot, z') \ d\gamma(z')$, as
\begin{align}
    \sup_{y \in \mathcal{Y}} \left|\frac{\int_{\mathcal{Z}} \phi(y, z') \ d\gamma(z')}{\xi_1(y)} \right| \leq  \int_{\mathcal{Z}} \left|\frac{\phi(y, z')}{\xi_1(y)} \right| \ d\gamma(z') \leq K \|\gamma\|_{\text{TV}}, \implies \int_{\mathcal{Z}} \phi(\cdot, z') \ d\gamma(z') \in C_{b, \xi_1}(\mathcal{Z}).
\end{align}
Moreover, the first variation of the map $\gamma \mapsto - \int_{\mathcal{Z}} g(z) \ d\gamma(z)$ is $-g$, which also has bounded supremum norm by the assumption of \autoref{thm:duality_general_f1}.

To prove the existence of a maximizer in (iii), we use the direct method of the calculus of variations. The functional is concave; the first term is linear and the second term is the negated convex conjugate of the KL-divergence composed with a linear map. We cannot use the TV norm topology for $\mathcal{M}(\mathcal{Z})$, because it does not make $\mathcal{B}_{\mathcal{M}(Z)}$ compact. We observe that the weak-* topology of $\text{rba}(\mathcal{Z})$ is the right choice. Here, $\text{rba}(\mathcal{Z})$ is the space of finitely additive finite signed regular Borel measures, which contains the space $\mathcal{M}(\mathcal{Z})$ of countably additive finite signed regular Borel measures, and it is the dual of $C_b(\mathcal{Z})$; see \autoref{thm:rieszmarkov_cb}. On the one hand, $\mathcal{B}_{\mathcal{M}(Z)}$ is compact in the weak-* topology of $\text{rba}(\mathcal{Z})$ by \autoref{lem:unit_ball_compact}.

On the other hand, we check that the functional is upper semicontinuous in this topology. The first term of the functional is continuous (thus, upper semicontinuous) in the weak-* topology of $\text{rba}(\mathcal{Z})$, because $-g \in C_b(\mathcal{Z})$ by assumption and $\text{rba}(\mathcal{Z}) = C_b^{*}(\mathcal{Z})$. We write the second term as
\begin{align}
\begin{split} \label{eq:second_term_derivation}
    &- \frac{1}{\beta} \log\left(\int_{\mathcal{Y}} \exp \left(-\beta \int_{\mathcal{Z}} \phi(y, z) \ d\gamma(z) \right) d\tau_{\mathcal{Y}}(y) \right) \\ &= - \sup_{\nu \in \mathcal{P}(\mathcal{Y})} \left\{ -\int_{\mathcal{Y}} \int_{\mathcal{Z}} \phi(y, z) \ d\gamma(z) \ d\nu(y) - \beta^{-1} D_{KL}(\nu||\tau_{\mathcal{Y}}) \right\}
    \\ &= \inf_{\nu \in \mathcal{P}(\mathcal{Y})} \left\{\int_{\mathcal{Z}} \int_{\mathcal{Y}} \phi(y, z) \ d\nu(y) \ d\gamma(z) + \beta^{-1} D_{KL}(\nu||\tau_{\mathcal{Y}}) \right\},
\end{split}
\end{align}
where the first equality follows from the argument in \autoref{eq:F_tilde_derivation} and in the second equality we used Fubini's theorem. Remark that for a fixed $\nu \in \mathcal{M}_{\xi_1}(\mathcal{Y})$, $\int_{\mathcal{Y}} \phi(y, \cdot) \ d\nu(y) \in C_b(\mathcal{Z})$ because of equation \eqref{eq:in_cb}.
Hence, the mapping $\gamma \mapsto \int_{\mathcal{Z}} \int_{\mathcal{Y}} \phi(y, z) \ d\nu(y) \ d\gamma(z) + \beta^{-1} D_{KL}(\nu||\tau_{\mathcal{Y}})$ is continuous (thus, upper semicontinuous) in the weak-* topology of $\text{rba}(\mathcal{Z})$. The pointwise infimum of upper semicontinuous functions is upper semicontinuous, and thus \eqref{eq:second_term_derivation} is upper semicontinuous as well.
\end{proof}

\begin{lemma} \label{lem:unit_ball_compact}
The unit TV norm ball of $\mathcal{B}_{\mathcal{M}(\mathcal{Z})}$, seen as a subset of $\text{rba}(\mathcal{Z})$, is compact in the weak-* topology of $\text{rba}(\mathcal{Z})$.
\end{lemma}
\begin{proof}
If we endow $\mathcal{M}(\mathcal{Z})$ with the weak-* topology given by its predual $C_0(\mathcal{Z})$ (\autoref{thm:rieszmarkov}), the Banach-Alaoglu theorem (\autoref{thm:banachalaoglu}) states that $\mathcal{B}_{\mathcal{M}(\mathcal{Z})}$ is compact. Since the weak-* topology is Hausdorff, and Hausdorff compact spaces are closed, we have that  $\mathcal{B}_{\mathcal{M}(\mathcal{Z})}$ is closed in the weak-* topology of $\mathcal{M}(\mathcal{Z})$. To show that $\mathcal{B}_{\mathcal{M}(\mathcal{Z})}$ is also closed in weak-* topology of $\text{rba}(\mathcal{Z})$, suppose that $\gamma \in \text{rba}(\mathcal{Z})$ is such that $(\gamma_n)_n \to \gamma$ in weak-* topology of $\text{rba}(\mathcal{Z})$ for some sequence $(\gamma_n)_n \subseteq \mathcal{M}(\mathcal{Z})$. Then, since $C_0(\mathcal{Z}) \subseteq C_b(\mathcal{Z})$, $(\gamma_n)_n \to \gamma$ in weak-* topology of $\mathcal{M}(\mathcal{Z})$, and the closedness of $\mathcal{M}(\mathcal{Z})$ implies that $\gamma \in \mathcal{M}(\mathcal{Z})$. 

We have that the TV norm closed unit ball $\mathcal{B}_{\text{rba}(\mathcal{Z})}$ of $\text{rba}(\mathcal{Z})$, which includes $\mathcal{B}_{\mathcal{M}(\mathcal{Z})}$, is compact in the weak-* topology again by the Banach-Alaoglu theorem. 
Since $\mathcal{B}_{\mathcal{M}(\mathcal{Z})}$ is a closed subset of the compact space $\mathcal{B}_{\text{rba}(\mathcal{Z})}$, it is itself compact in the weak-* topology of $\mathcal{B}_{\text{rba}(\mathcal{Z})}$. 
\end{proof}

\begin{lemma} \label{lem:G_tildeG}
The function $\tilde{G}$ defined in \eqref{eq:tilde_G} is such that $\tilde{G}(\gamma) = \sup_{\psi \in C_{b}(\mathcal{Z})} \left\{\langle \gamma, \psi \rangle - G(\psi) \right\}$, where $G$ is defined in \eqref{eq:G_def_3}. 
\end{lemma}

\begin{proof}
By definition $\tilde{G}(\gamma)$ is
\begin{align}
\begin{split} \label{eq:G_star_3}
    &\sup_{\psi \in C_{b}(\mathcal{Z})} \left\{ \int_{\mathcal{Z}} \psi(z) \ d\gamma(z) - \sup_{\substack{\gamma' \in \mathcal{M}(\mathcal{Z}) \\ \|\gamma'\|_{\text{TV}} \leq 1}} \int_{\mathcal{Z}} (\psi(z) - g(z)) \ d\gamma'(z) \right\} \\ &= \sup_{\psi \in C_{b}(\mathcal{Z})} \inf_{\substack{\gamma' \in \mathcal{M}(\mathcal{Z}) \\ \|\gamma'\|_{\text{TV}} \leq 1}} \left\{ \int_{\mathcal{Z}} \psi(z) \ d(\gamma-\gamma')(z) + \int_{\mathcal{Z}} g(z) \ d\gamma'(z) \right\}  
\end{split}
\end{align}
We want to apply \autoref{thm:kneser} to flip the supremum and the infimum. We set $B$ as in the proof of \autoref{lem:G2_conj}. The set $A$ requires a careful construction. $\mathcal{M}(\mathcal{Z})$, which is the space finite countably additive regular Borel measures, is included in the Banach space of finite finitely additive regular Borel measures $\text{rba}(\mathcal{Z})$ endowed with the total variation norm, which by
\autoref{thm:rieszmarkov_cb}
is the continuous dual of $C_b(\mathcal{Z})$. $\text{rba}(\mathcal{Z})$ can be endowed with the weak-* topology of $\text{rba}(\mathcal{Z})$, which is the weakest one that makes maps of the form $\gamma \mapsto \int_{\mathcal{Z}} f(z) \ d\gamma(z)$ continuous for any $f \in C_b(\mathcal{Z})$. We set $A = \mathcal{B}_{\mathcal{M}(\mathcal{Z})}$ to be the TV norm ball of $\mathcal{M}(\mathcal{Z})$, as a subset of $\text{rba}(\mathcal{Z})$ endowed with its the weak-* topology of $\text{rba}(\mathcal{Z})$. 
Notice that $\gamma' \mapsto \int_{\mathcal{Z}} (g(z) - \psi(z)) \ d\gamma'(z)$ is continuous in the weak-* topology of $\text{rba}(\mathcal{Z})$ because $g - \psi \in C_b(\mathcal{Z})$.

It only remains to show that $\mathcal{B}_{\mathcal{M}(\mathcal{Z})}$ is compact in the weak-* topology of $\text{rba}(\mathcal{Z})$. Thus, \autoref{thm:kneser} can be applied, which means that the right-hand side of \eqref{eq:G_star_3} is equal to
\begin{align}
\begin{split}
    \\&\inf_{\substack{\gamma' \in \mathcal{M}(\mathcal{Z}) \\ \|\gamma'\|_{\text{TV}} \leq 1}} \sup_{\psi \in C_{b}(\mathcal{Z})} \left\{ \int_{\mathcal{Z}} \psi(z) \ d(\gamma-\gamma')(z) + \int_{\mathcal{Z}} g(z) \ d\gamma'(z) \right\}
\\ &= 
\begin{cases}
\int_{\mathcal{Z}} g(z) \ d\gamma(z) &\text{if } \|\gamma\|_{\text{TV}} \leq 1 \\
+\infty &\text{otherwise}
\end{cases}.
\end{split}    
\end{align}
\end{proof}

\textbf{\textit{Proof of \autoref{lem:primal}.}}
The Euler-Lagrange condition for \eqref{eq:general_primal_problem} is 
\begin{align}
    0 = \beta^{-1} \log \left(\frac{d\nu_1^{\star}}{d\tau_{\mathcal{Y}}}(y) \right) + \frac{\int_{\mathcal{Z}} \left( \int_{\mathcal{Y}} \phi(y', z) \ d\nu_1^{\star}(y') - g(z) \right) \phi(y, z) \ d\tau_{\mathcal{Z}}(z)}{\left( \int_{\mathcal{Z}} \left( \int_{\mathcal{Y}} \phi(y', z) \ d\nu_1^{\star}(y') - g(z) \right)^2 \ d\tau_{\mathcal{Z}}(z) \right)^{1/2}} + K, \quad \forall y \in \mathcal{Y}
\end{align}
for some $K$. Thus,
\begin{align} \label{eq:gamma_1_density}
    \frac{d\nu_1^{\star}}{d\tau_{\mathcal{Y}}}(y) = \frac{1}{Z}\exp\left( - \beta \frac{\int_{\mathcal{Z}} \left( \int_{\mathcal{Y}} \phi(y', z) \ d\nu_1^{\star}(y') - g(z) \right) \phi(y, z) \ d\tau_{\mathcal{Z}}(z)}{\left( \int_{\mathcal{Z}} \left( \int_{\mathcal{Y}} \phi(y', z) \ d\nu_1^{\star}(y') - g(z) \right)^2 \ d\tau_{\mathcal{Z}}(z) \right)^{1/2}} \right), \quad \forall y \in \mathcal{Y}.
\end{align}
The Euler-Lagrange condition for \eqref{eq:general_primal_problem2} is
\begin{align} \label{eq:euler_problem2}
    0 = \tilde{\beta}^{-1} \log \left(\frac{d\nu_2^{\star}}{d\tau_{\mathcal{Y}}}(y) \right) + 2\int_{\mathcal{Z}} \left( \int_{\mathcal{Y}} \phi(y', z) \ d\nu_2^{\star}(y') - g(z) \right) \phi(y, z) \ d\tau_{\mathcal{Z}}(z) + K, \quad \forall y \in \mathcal{Y}
\end{align}
for some $K$. Hence,
\begin{align} \label{eq:gamma_2_density}
    \frac{d\nu_2^{\star}}{d\tau_{\mathcal{Y}}}(y) = \frac{1}{Z}\exp\left( - 2\tilde{\beta} \int_{\mathcal{Z}} \left( \int_{\mathcal{Y}} \phi(y', z) \ d\nu_2^{\star}(y') - g(z) \right) \phi(y, z) \ d\tau_{\mathcal{Z}}(z) \right).
\end{align}
Comparing \eqref{eq:gamma_1_density} with \eqref{eq:gamma_1_density}, we see that $\nu_1^{\star}$ is equal to $\nu_2^{\star}$ when $\tilde{\beta}$ is set such that
\begin{align}
    2\tilde{\beta} = \frac{\beta}{\left( \int_{\mathcal{Z}} \left( \int_{\mathcal{Y}} \phi(y', z) \ d\nu_1^{\star}(y') - g(z) \right)^2 \ d\tau_{\mathcal{Z}}(z) \right)^{1/2}}
\end{align}
Conversely, the solution $\nu_2^{\star}$ for a certain $\tilde{\beta}$ is equal to $\nu_1^{\star}$ when $\beta$ is set such that
\begin{align}
    2\tilde{\beta} = \frac{\beta}{\left( \int_{\mathcal{Z}} \left( \int_{\mathcal{Y}} \phi(y', z) \ d\nu_2^{\star}(y') - g(z) \right)^2 \ d\tau_{\mathcal{Z}}(z) \right)^{1/2}}.
\end{align}
\qed

\textit{\textbf{Proof of \autoref{prop:planted}.}}
    To prove (a), we use duality. Strong duality holds between \eqref{eq:general_primal_problem} and \eqref{eq:general_dual_problem} and moreover by \autoref{thm:duality_general} the respective solutions $\nu^{\star}_1$ and $h^{\star}$ of the two problems are linked by:
    \begin{align} \label{eq:link_duality_population}
    \frac{d\nu_1^{\star}}{d\tau_{\mathcal{Y}}}(y) = \frac{1}{Z_{\nu_1^{\star}}}\exp \left( - \beta \int_{\mathcal{Z}} \phi(y, z) \ h^{\star}(z) d\tau_{\mathcal{Z}}(z) \right).
    \end{align}
    Remark that an arbitrary element $f$ of the RKHS $\mathcal{F}_2$ admits a representation as
    \begin{align} \label{eq:f_h_change2}
    f(y) = \int_{\Theta} \phi(y, z) h(z) \ d\tau_{\mathcal{Z}}(z), \quad \text{where } h \in L^2(\mathcal{Z}), \quad \text{and } \|f\|_{\mathcal{F}_2} = \|h\|_{L^2(\mathcal{Z})}.
    \end{align}
    For an arbitrary $f$, denote by $\nu_f$ the probability measure with density $\frac{d\nu_{f}}{d\tau_{\mathcal{Y}}}(y) = \exp(-f(y))/\int_{\mathcal{Y}} \exp(-f(y')) \ d\tau_{\mathcal{Y}}(y')$. Using \eqref{eq:f_h_change2} and $g(z) = \int_{\mathcal{Y}} \phi(y,z) \ d\nu_p(y)$, we rewrite the problem \eqref{eq:general_dual_problem} as 
    \begin{align}
    \begin{split}
        &\argmin_{\substack{h \in L^2(\mathcal{Z}) \\ \|h\|_{L^2} \leq 1}} \int_{\mathcal{Z}} \int_{\mathcal{Y}} \phi(y,z) \ d\nu_p(y) \ h(z)\ d\tau_{\mathcal{Z}}(z) + \frac{1}{\beta} \log\left(\int_{\mathcal{Y}} \exp \left(-\beta \int_{\mathcal{Z}} \phi(y, z) h(z)\ d\tau_{\mathcal{Z}}(z) \right) d\tau_{\mathcal{Y}}(y) \right) \\ &= \argmin_{f \in \mathcal{B}_{\mathcal{F}_2}(\beta)} \int_{\mathcal{Y}} f(y) \ d\nu_p(y) + \log \left( \int_{\mathcal{Y}} e^{-f(y)} d\tau_{\mathcal{Y}}(y) \right) = \argmin_{f \in \mathcal{B}_{\mathcal{F}_2}(\beta)} - \int_{\mathcal{Y}} \log \left(\frac{d\nu_{f}}{d\tau_{\mathcal{Y}}}(y) \right) \ d\nu_p(y) \\ &= \argmin_{f \in \mathcal{B}_{\mathcal{F}_2}(\beta)} H(\nu_p, \nu_{f})
        = \argmin_{f \in \mathcal{B}_{\mathcal{F}_2}(\beta)} H(\nu_p, \nu_{f}) - H(\nu_p, \nu_p) = \argmin_{f \in \mathcal{B}_{\mathcal{F}_2}(\beta)} D_{KL}(\nu_p || \nu_{f}).
    \end{split}
    \end{align}
    In the first equality we have used Fubini's theorem to exchange the integrals in the first term, using the same reasoning as in \eqref{eq:adjoint_exp}-\eqref{eq:fubini_1}. In the second equality we use the definition of $\nu_{f}$. In the third equality, $H$ denotes the cross-entropy, and in the fourth one, we use that $H(\nu_p, \nu_p)$ is finite because $\nu_p$ is absolutely continuous w.r.t. $\tau_{\mathcal{Y}}$. The fifth equality is by the definition of the KL divergence.
    
    From this viewpoint, we have that the solution $f^{\star} = \argmin_{f \in \mathcal{B}_{\mathcal{F}_2}(\beta)} D_{KL}(\nu_p || \nu_{f})$ is linked to the solution $h^{\star}$ of \eqref{eq:general_dual_problem}: 
    $f^{\star}(\cdot) = \beta \int_{\mathcal{Z}} \phi(\cdot, z) \ h^{\star}(z) d\tau_{\mathcal{Z}}(z)$.
    Plugging this into \eqref{eq:link_duality_population}, we obtain that 
    \begin{align} \label{eq:nu_star_1}
    \frac{d\nu^{\star}_1}{d\tau}(y) = \frac{1}{Z_{\nu^{\star}_1}}\exp \left( - f^{*}(y) \right).
    \end{align}
    Since we have assumed that $E = -\log(\frac{d\nu_p}{d\tau_{\mathcal{Y}}}) \in \mathcal{B}_{\mathcal{F}_2}(\beta_0)$ with $\beta > \beta_0$, the unique solution of $\argmin_{f \in \mathcal{B}_{\mathcal{F}_2}(\beta)} D_{KL}(\nu_p || \nu_{f})$ is $f^{\star} = E$, which through \eqref{eq:nu_star_1} implies that $\nu^{\star}_1 = \nu_{E} = \nu_p$.
    
    To show (b), we use the Euler-Lagrange equation of \eqref{eq:general_primal_problem2}, which is stated in \eqref{eq:euler_problem2}.
    Since $\beta^{-1} \log\left( \frac{d\nu_2^{\star}}{d\tau_{\mathcal{Y}}}(y) \right) \neq 0$ for all $y \in \mathcal{Y}$, we must have that \begin{align} \label{eq:EL_not_zero}
    \int_{\mathcal{Z}} \left( \int_{\mathcal{Y}} \phi(y', z) \ d\nu_2^{\star}(y') - g(z) \right) \phi(y, z) \ d\tau_{\mathcal{Z}}(z) \neq K,
    \end{align}
    does not hold uniformly over $y \in \mathcal{Y}$ for any constant $K$.
    
    If we had $\int_{\mathcal{Z}} \left( \int_{\mathcal{Y}} \phi(y', z) \ d\nu_2^{\star}(y') - g(z) \right)^2 \ d\tau_{\mathcal{Z}}(z) = 0$, that would mean that for all $z \in \mathcal{Z}$, $\int_{\mathcal{Y}} \phi(y', z) \ d\nu_2^{\star}(y') - g(z) = 0$. This would imply that \eqref{eq:EL_not_zero} is equal to zero for all $y \in \mathcal{Y}$, yielding a contradiction.
\qed

\section{Proofs of \autoref{sec:implicit_mle_f1} and additional results}
\label{sec:proofs_implicit_f1}

\textbf{\textit{Proof of \autoref{cor:EBM_f1}.}}
The proof follows from applying \autoref{thm:duality_general_f1} with $\mathcal{Y} = \mathcal{X}$ and $\mathcal{Z} = \Theta$. \autoref{ass:phi} holds because it is implied by \autoref{ass:phi_new} when one sets $\xi_1 = \xi$. \autoref{ass:phi_new} also implies that $\phi$ satisfies the assumption (i) in \autoref{thm:duality_general_f1}. Assumption (ii) in \autoref{thm:duality_general_f1} is also fulfilled because $g(\theta) = \frac{1}{n} \sum_{i=1}^n \phi(x_i,\theta) \leq \frac{1}{n} \sum_{i=1}^n \xi(x_i)$, which means that $g \in C_b(\Theta)$. By \autoref{thm:duality_general_f1}, we see that problem \eqref{eq:EBM_dual_f1} is the Fenchel dual of
\begin{align}
    \begin{split} \label{eq:EBM_primal_2}
    \min_{\nu \in \mathcal{P}(\mathcal{X})} &\beta^{-1} D_{\text{KL}}(\nu||\tau_{\mathcal{X}})
    + \left\| \int_{\mathcal{X}} \phi(x, \cdot) \ d(\nu - \nu_n)(x) \right\|_{L^{\infty}}.
\end{split}
\end{align}
and we also obtain the characterization for the measure $\nu^{\star}$. Since 
$\left\| \int_{\mathcal{X}} \phi(x, \cdot) \ d(\nu - \nu_n)(x) \right\|_{L^{\infty}} = \sup_{\gamma \in \mathcal{M}(\Theta), \|\gamma\|_{\text{TV}} \leq 1} \int_{\Theta} \int_{\mathcal{X}} \phi(x, \theta) \ d(\nu - \nu_n)(x) \ d\gamma(\theta) = \sup_{f \in \mathcal{B}_{\mathcal{F}_1}} \int_{\mathcal{X}} f(x) \ d(\nu - \nu_n)(x)$ using Fubini's theorem, we can rewrite \eqref{eq:EBM_primal_2} as
\begin{align} \label{eq:EBM_primal_3_old}
    \min_{\nu \in \mathcal{P}(\mathcal{X})} \max_{\substack{\gamma \in \mathcal{M}(\Theta), \\ \|\gamma\|_{\text{TV}} \leq 1}} &\beta^{-1} D_{\text{KL}}(\nu||\tau_{\mathcal{X}})
    + \int_{\Theta} \int_{\mathcal{X}} \phi(x, \theta) \ d(\nu - \nu_n)(x) \ d\gamma(\theta).
\end{align}
\qed

\textbf{\textit{Proof of \autoref{prop:particle_dynamics_f1}.}}
For $\sigma = \pm 1$, define the empirical measures $\hat{\gamma}_t^{\sigma} = \frac{1}{m} \sum_{j=1}^m \mathds{1}_{\sigma_j = \sigma} w^{(j)}_t \delta_{\theta_t^{(j)}}$,
Given a test function $\chi$ on $\Theta$, we have that
\begin{align}
\begin{split}
    &\frac{d}{dt}\int_{\Theta} \chi(\theta) \ d\hat{\gamma}_t^{\sigma}(\theta) = \frac{d}{dt}\left( \frac{1}{m} \sum_{j=1}^{m} \mathds{1}_{\sigma_j = \sigma} w_t^{(j)} f(\theta_t^{(j)}) \right) \\ &= \frac{1}{m} \sum_{j=1}^{m} \mathds{1}_{\sigma_j = \sigma} \frac{dw_t^{(j)}}{dt} \chi(\theta_t^{(j)}) + \mathds{1}_{\sigma_j = \sigma} w_t^{(j)} \frac{d}{dt}\chi(\theta_t^{(j)})
    \\ &= \frac{\alpha}{m} \sum_{j=1}^{m} \mathds{1}_{\sigma_j = \sigma}  w_t^{(j)} (\sigma_j \tilde F_t(\theta_t^{(j)}) - \tilde{K}_t) \chi(\theta_t^{(j)}) + \mathds{1}_{\sigma_j = \sigma} w_t^{(j)} \nabla \chi(\theta_t^{(j)}) \cdot \sigma_j \nabla \tilde F_t(\theta_t^{(j)})
    \\ &= \alpha \int_{\Theta} \left((\sigma \tilde F_t(\theta) - \tilde{K}_t) \chi(\theta) + \sigma \nabla \tilde F_t(\theta) \cdot \nabla \chi(\theta) \right) \ d\hat{\gamma}_t^{\sigma}(\theta)
\end{split}    
\end{align}
This is the weak formulation of the first equation in \eqref{eq:coupled_dynamics_f1}. We also observe that the forward Kolmogorov equation of the third equation in \eqref{eq:coupled_particles_MLE} is the Fokker-Planck equation in the second line of \eqref{eq:coupled_dynamics_f1}. The propagation of chaos argument that allows us to establish convergence $\hat{\gamma}_t \to \gamma_t$ and $\hat{\nu}_t \to \hat{\nu}$ is classical \citep{sznitmantopics1991} and can be found for a very similar coupled setting in \cite{domingoenrich2020amean}.
\qed

\subsection{Link of dual $\mathcal{F}_1$-EBMs training with learned MMD training}
We show that training dual $\mathcal{F}_1$-EBMs is equivalent to learning a certain form of MMD with feature learning. This observation provides a clearer link between dual $\mathcal{F}_1$-EBMs and dual $\mathcal{F}_2$-EBMs, in which the kernel is fixed (equation \eqref{eq:kernel_random}). Feature-learning MMD has been the subject of several works and has been shown to outperform fixed-kernel MMD \citep{li2017mmd}. In particular, we have the following:

\begin{proposition} \label{prop:learned_mmd}
The solutions or saddle points of \eqref{eq:EBM_primal_3} are the saddle points of
\begin{align} \label{eq:learned_mmd}
    \min_{\nu \in \mathcal{P}(\mathcal{X})} \max_{\gamma \in \mathcal{P}(\Theta)} &\beta^{-1} D_{KL}(\nu||\tau_{\mathcal{X}})
    + MMD_{k_\gamma}(\nu, \nu_n),
\end{align}
where $MMD_{k_\gamma}(\nu, \nu_n) = ( \int_{\mathcal{X} \times \mathcal{X}} k_{\gamma}(x,x') \ d(\nu - \nu_n)(x) \ d(\nu - \nu_n)(x') )^{1/2}$ and the kernel $k_{\gamma} = \int_{\Theta} \phi(x, \theta) \phi(x', \theta) \ d\gamma(\theta)$ is well defined for any $\gamma \in \mathcal{P}(\Theta)$.
\end{proposition}
\begin{proof}
The second term in the objective of \eqref{eq:EBM_primal_3} is $\int_{\Theta} \int_{\mathcal{X}} \phi(x, \theta) \ d(\nu - \nu_n)(x) \ d\gamma(\theta)$. For any $\gamma \in \mathcal{M}(\Theta), \|\gamma\|_{\text{TV}} \leq 1$, we apply the Cauchy-Schwarz inequality and then Fubini's theorem:
\begin{align} \label{eq:cs_fubini}
&\int_{\Theta} \int_{\mathcal{X}} \phi(x, \theta) \ d(\nu - \nu_n)(x) \ d\gamma(\theta) \leq \left(\int_{\Theta} \left(\int_{\mathcal{X}} \phi(x, \theta) \ d(\nu - \nu_n)(x) \right)^2 \ d|\gamma|(\theta) \right)^{1/2} \\ &= \left( \int_{\mathcal{X} \times \mathcal{X}} \int_{\Theta} \phi(x, \theta) \phi(x', \theta) \ d|\gamma|(\theta) \ d(\nu - \nu_n)(x) \ d(\nu - \nu_n)(x') \right)^{1/2} \\ &= \left( \int_{\mathcal{X} \times \mathcal{X}} k_{\gamma}(x,x') \ d(\nu - \nu_n)(x) \ d(\nu - \nu_n)(x') \right)^{1/2} = MMD_{k_\gamma}(\nu - \nu_n).
\end{align}
For any $\nu \in \mathcal{P}(\Theta)$, notice that for all measures 
\begin{align} \label{eq:gamma_star_argmax_1}
\gamma^{\star} \in \argmax_{\gamma \in \mathcal{M}(\Theta), \|\gamma\|_{\text{TV}} \leq 1} \int_{\Theta} \int_{\mathcal{X}} \phi(x, \theta) \ d(\nu - \nu_n)(x) \ d\gamma(\theta)
\end{align}
and all measures 
\begin{align}
\gamma^{\star} \in \argmax_{\gamma \in \mathcal{P}(\mathcal{X})} \left(\int_{\Theta} \left(\int_{\mathcal{X}} \phi(x, \theta) \ d(\nu - \nu_n)(x) \right)^2 \ d|\gamma|(\theta) \right)^{1/2}, 
\end{align}
we must have 
\begin{align}
\text{supp}(\gamma^{\star}) \subseteq \left\{\theta' \in \Theta \ \bigg| \ \left|\int_{\mathcal{X}} \phi(x, \theta') \ d(\nu - \nu_n)(x) \right| = \max_{\theta \in \Theta} \left|\int_{\mathcal{X}} \phi(x, \theta) \ d(\nu - \nu_n)(x) \right| \right\}.
\end{align}
Hence, for any measure $\gamma^{\star}$ fulfilling \eqref{eq:gamma_star_argmax_1}, 
\begin{align}
    &\int_{\Theta} \int_{\mathcal{X}} \phi(x, \theta) \ d(\nu - \nu_n)(x) \ d\gamma^{\star}(\theta) = \max_{\theta \in \Theta} \left|\int_{\mathcal{X}} \phi(x, \theta) \ d(\nu - \nu_n)(x) \right| \\ &= \left(\int_{\Theta} \left(\int_{\mathcal{X}} \phi(x, \theta) \ d(\nu - \nu_n)(x) \right)^2 \ d|\gamma^{\star}|(\theta) \right)^{1/2},
\end{align}
which shows that at maximizers, all the terms of \eqref{eq:cs_fubini} are equal, concluding the proof.
\end{proof}

\section{Links with Score Matching }
\label{sec:proofs_SM}


\begin{proposition} \label{prop:SM_expression}
Suppose that $\mathcal{X} \subseteq \R^{d_1}$ is a manifold without boundaries. 
Assume that $\int_{\mathcal{X}} |\nabla_x \phi(x, \theta) \cdot \nabla \frac{d\nu_p}{d\tau_{\mathcal{X}}}(x)| \ d\tau_{\mathcal{X}}(x)$ is upper-bounded by some constant $K$ for all $\theta \in \Theta$. Assume also that $\sup_{\theta \in \Theta} \|\nabla_x \phi(x, \theta)\| < \eta(x)$ and that $\int_{\mathcal{X}} |\eta(x)|^2 \ d\nu_p(x) < \infty$. The optimization problem to train EBMs under the score matching loss over the ball $\mathcal{B}_{\mathcal{F}_1}(1)$ gives $f_{\rm SM} = \int_\Omega \varphi(\cdot,\theta) 
d\gamma_{\rm SM} (\theta)$ where
\begin{align} \label{eq:SM_objective_22}
    \gamma_{\rm SM}= \argmin_{\substack{\gamma \in \mathcal{M}(\Theta) \\ \|\gamma\|_{\text{TV}} \leq 1}}
    \int_{\Theta} \int_{\mathcal{X}}\left(\frac{1}{2} \nabla_x \phi(x, \theta) \cdot \nabla_x \int_{\Theta} \phi(x, \theta') \ d\gamma(\theta')  - \beta^{-1} \Delta_x \phi(x, \theta) \right) d\nu_n(x) d\gamma(\theta).
\end{align}
\end{proposition}
\begin{proof}
The score matching metric between two absolutely continuous measures $\nu$ and $\nu_p$ is
\begin{align} \label{eq:SM_equality}
    \text{SM}(\nu_p, \nu) = \int_{\mathcal{X}} \left\|\nabla \log \frac{d\nu}{d\tau_{\mathcal{X}}}(x) - \nabla \log \frac{d\nu_p}{d\tau_{\mathcal{X}}}(x) \right\|^2 \ d\nu_p(x)
\end{align}
If constrain the density of $\nu$ to belong to the $\mathcal{F}_1$ ball of radius $\beta$, we can write $\log \frac{d\nu}{d\tau_{\mathcal{X}}}(x) = - \int_{\Theta} \phi(x, \theta) \ d\gamma(\theta)$ for some $\gamma \in \mathcal{M}(\Theta)$ such that $\|\gamma \|_{\text{TV}} \leq \beta$.
Thus, the minimization problem of $\text{SM}(\nu, \nu_p)$ over this class of energies can be written as
\begin{align}
    \min_{\substack{\gamma \in \mathcal{M}(\Theta) \\ \|\gamma\|_{\text{TV}} \leq \beta}} \int_{\mathcal{X}} \left\|- \int_{\Theta} \nabla_x \phi(x, \theta) \ d\gamma(\theta) - \nabla \log \frac{d\nu_p}{d\tau_{\mathcal{X}}}(x) \right\|^2 \ d\nu_p(x) 
\end{align}
Following \cite{hyvarinen05estimation}, the objective functional can be expressed as
\begin{align}
    \int_{\mathcal{X}} \left(\left\|\int_{\Theta} \nabla_x \phi(x, \theta) \ d\gamma(\theta)\right\|^2 + \left\|\nabla \log \frac{d\nu_p}{d\tau_{\mathcal{X}}}(x) \right\|^2 + 2 \int_{\Theta} \nabla_x \phi(x, \theta) \cdot \nabla \log \frac{d\nu_p}{d\tau_{\mathcal{X}}}(x) \ d\gamma(\theta) \right) \ d\nu_p(x).
\end{align}
The middle term is constant w.r.t. $\gamma$, hence it is irrelevant. We use Fubini's theorem in the third term
\begin{align}
\begin{split} \label{eq:fubini_application}
    &\int_{\mathcal{X}} \int_{\Theta} \nabla_x \phi(x, \theta) \cdot \nabla \log \frac{d\nu_p}{d\tau_{\mathcal{X}}}(x) \ d\gamma(\theta) \ d\nu_p(x) = \int_{\Theta} \int_{\mathcal{X}} \nabla_x \phi(x, \theta) \cdot \nabla \log \frac{d\nu_p}{d\tau_{\mathcal{X}}}(x) \ d\nu_p(x) \ d\gamma(\theta) \\ &= \int_{\Theta} \int_{\mathcal{X}} \nabla_x \phi(x, \theta) \cdot \nabla \frac{d\nu_p}{d\tau_{\mathcal{X}}}(x) \ d\tau_{\mathcal{X}}(x) \ d\gamma(\theta) = - \int_{\Theta} \int_{\mathcal{X}} \Delta \phi(x, \theta) \frac{d\nu_p}{d\tau_{\mathcal{X}}}(x) \ d\tau_{\mathcal{X}}(x) \ d\gamma(\theta) \\ &= - \int_{\Theta} \int_{\mathcal{X}} \Delta \phi(x, \theta) \ d\nu_p(x) \ d\gamma(\theta).
\end{split}
\end{align}
In the fourth equality of \eqref{eq:fubini_application} we applied integration by parts. Fubini's theorem can be applied in the first equality because 
\begin{align}
\begin{split}
    \int_{\Theta} \int_{\mathcal{X}} |\nabla_x \phi(x, \theta) \cdot \nabla \log \frac{d\nu_p}{d\tau_{\mathcal{X}}}(x)| \ d\nu_p(x) \ d|\gamma|(\theta) &= \int_{\Theta} \int_{\mathcal{X}} |\nabla_x \phi(x, \theta) \cdot \nabla \frac{d\nu_p}{d\tau_{\mathcal{X}}}(x)| \ d\tau_{\mathcal{X}}(x) d|\gamma|(\theta) \\ &\leq K \|\gamma \|_{\text{TV}} < +\infty
\end{split}
\end{align}
We also use similar arguments for the first term:
\begin{align}
\begin{split}
    &\int_{\mathcal{X}} \left\|\int_{\Theta} \nabla_x \phi(x, \theta) \ d\gamma(\theta)\right\|^2 \ d\nu_p(x) = \int_{\mathcal{X}} \int_{\Theta} \int_{\Theta} \nabla_x \phi(x, \theta) \cdot \nabla_x \phi(x, \theta') \ d\gamma(\theta) \ d\gamma(\theta') \ d\nu_p(x) \\ &= \int_{\Theta} \int_{\mathcal{X}} \nabla_x \phi(x, \theta) \cdot \int_{\Theta} \nabla_x \phi(x, \theta') \ d\gamma(\theta') \ d\nu_p(x) \ d\gamma(\theta) 
\end{split}
\end{align}
In this equation we can apply Fubini's theorem because 
\begin{align}
    &\int_{\Theta} \int_{\mathcal{X}} \left|\int_{\Theta} \nabla_x \phi(x, \theta) \cdot \nabla_x \phi(x, \theta') \ d\gamma(\theta) \right| \ d\nu_p(x) \ d|\gamma|(\theta') \\ &\leq \int_{\Theta} \int_{\mathcal{X}} \int_{\Theta} \|\nabla_x \phi(x, \theta)\| \|\nabla_x \phi(x, \theta')\| \ d|\gamma|(\theta) \ d\nu_p(x) \ d|\gamma|(\theta') \\ &\leq \|\gamma\|_{\text{TV}} \int_{\mathcal{X}} \int_{\Theta} \eta(x)^2 \ d\nu_p(x) \ d|\gamma|(\theta') < +\infty,
\end{align}
by the assumption that $\int_{\mathcal{X}} \eta(x)^2 \ d\nu_p(x) < +\infty$.
The proof is concluded by exchanging $\nu_p$ by its empirical version $\nu_n$.
\end{proof}

\textbf{\textit{Proof of \autoref{prop:limitSM}.}}
Let us start from the dynamics \eqref{eq:coupled_dynamics_main}. For a domain $\mathcal{X}$ without boundary, Duhamel's principle states that the solution $u(x,t)$ of
\begin{align}
    \begin{cases}
    \partial_t u(x,t) - Lu(x,t) = f(x,t) \\
    u(x,0) = 0
    \end{cases}
\end{align}
is equal to $u(x,t) = \int_{0}^{t} P_{s} f(x,t) \ ds$, where $P_s f$ is the solution of
\begin{align}
    \begin{cases}
    \partial_t u(x,t) - Lu(x,t) = 0 \\
    u(x,s) = f(x,s)
    \end{cases}
\end{align}
While it is typically stated for classical PDEs, in our case we consider Duhamel's principle in the weak sense, i.e. the equalities hold when integrated with respect to test functions.

We can apply Duhamel's principle for the second equation of \eqref{eq:coupled_dynamics_main}, with $u(\cdot,t) = \nu_t - \nu_0$, $L u = - \alpha u$ and $f(x,t) = \nabla_x \cdot \left( \nu_t \nabla_x \int_{\Omega} \tilde{\phi}(x, \omega) \ d\mu_t(\omega) \right) + \beta^{-1} \Delta_x \nu_t + \alpha (\nu_n - \nu_0)$. Notice that the solution $P_s f$ of 
\begin{align}
    \begin{cases}
    \partial_t u(x,t) + \alpha u(x,t) = 0 \\
    u(x,s) = f(x,s)
    \end{cases}
\end{align}
is 
$P_s f(x,t) = f(x,s) e^{-\alpha (t-s)}$. 
By Duhamel's principle we obtain that
\begin{align}
\begin{split}
    \nu_t - \nu_0 &= \int_{0}^{t} P_s f(x,t) \ ds \\ &= \int_{0}^{t} \left( \nabla_x \cdot \left( \nu_s \nabla_x \int_{\Theta} \phi(x, \theta) \ d\gamma_s(\theta) \right) + \beta^{-1} \Delta_x \nu_s + \alpha (\nu_n - \nu_0) \right) e^{-\alpha (t-s)} \ ds.
\end{split}
\end{align}
Since $\alpha \int_{0}^t  e^{-\alpha (t-s)} \ ds = 1-e^{-\alpha t}$, this is equivalent to
\begin{align} \label{eq:duhamel}
    \nu_t = \nu_0 e^{-\alpha t} + \nu_{n} (1-e^{-\alpha t}) + \int_0^t e^{-\alpha (t-s)} \left( \nabla_x \cdot \left( \nu_s \nabla_x \int_{\Omega} \phi(x, \omega) \ d\gamma_s(\omega) \right) + \beta^{-1} \Delta \nu_s \right) \ ds
\end{align}
From \eqref{eq:duhamel}, we see that as $\alpha \rightarrow +\infty$,
\begin{align} \label{eq:measure_limit}
    \alpha(\nu_t - \nu_{n}) \rightarrow \nabla_x \cdot \left( \nu_t \nabla_x \int_{\Omega} \phi(x, \omega) \ d\gamma_t(\omega) \right) + \beta^{-1} \Delta \nu_t,
\end{align}
or alternatively, for any test function $f$,
\begin{align} \label{eq:weak_limit}
    \alpha \int_{\mathcal{X}} f(x) d(\nu_t - \nu_{n})(x) \rightarrow - \int_{\mathcal{X}} \nabla f(x) \cdot \nabla_x \int_{\Theta} \phi(x, \theta) \ d\gamma_t(\theta) \ d\nu_t(x) + \beta^{-1} \int_{\mathcal{X}} \Delta f(x) \ d\nu_t(x).
\end{align}
Moreover, \eqref{eq:measure_limit} implies that $\alpha \rightarrow +\infty$, $\nu_t \rightarrow \nu_{n}$. Applying this into \eqref{eq:weak_limit}, we obtain that
\begin{align} \label{eq:weak_limit_2}
    \alpha \int_{\mathcal{X}} f(x) d(\nu_t - \nu_{n})(x) \rightarrow - \int_{\mathcal{X}} \nabla f(x) \cdot \nabla_x \int_{\Theta} \phi(x, \theta) \ d\mu_t(\theta) \ d\nu_n(x) + \beta^{-1} \int_{\mathcal{X}} \Delta f(x) \ d\nu_n(x).
\end{align}
Plugging this into the definition of $F_t$ in \eqref{eq:def_F_f_mu}, we get that 
\begin{align}
\alpha F_t(\theta) \rightarrow  - \int_{\mathcal{X}} \nabla_x \phi(x, \theta) \cdot \nabla_x \int_{\Theta} \phi(x, \theta) \ d\mu_t(\theta) \ d\nu_n(x) + \beta^{-1} \int_{\mathcal{X}} \Delta f(x) \ d\nu_n(x).
\end{align}
Using this in the first equation of \eqref{eq:coupled_dynamics_main}, we get that in the limit $\alpha \rightarrow +\infty$,
\begin{align}
\begin{split}
    &\partial_t \gamma_t^{\sigma} 
    \\ &= \sigma \nabla_\theta \cdot \left( \gamma_t^{\sigma} \left( \nabla_{\theta} \int_{\mathcal{X}} \nabla_x \phi(x, \theta) \cdot \int_{\Theta} \nabla_x \phi(x, \theta') \ d\mu_t(\theta') \ d\nu_n(x) + \beta^{-1} \nabla_{\theta} \int_{\mathcal{X}} \Delta_x \phi(x, \theta) \ d\nu_t(x) \right) \right)
    \\ &+ \mu_t \left(- \sigma \int_{\mathcal{X}} \nabla_x \phi(x, \theta) \cdot \int_{\Theta} \nabla_x \phi(x, \theta) \ d\gamma_t(\theta) \ d\nu_n(x) + \sigma \beta^{-1} \int_{\mathcal{X}} \Delta_x \phi(x, \theta) \ d\nu_n(x) - \tilde{K}_t \right) 
    \\ &= 
    \frac{1}{2\beta^2} 
    \left( \sigma \nabla_\theta \cdot \left( \gamma^\sigma_t \nabla_{\theta} V(\gamma_t)(\theta) \right) - \gamma^\sigma_t \left(\sigma V(\gamma_t)(\theta) - \bar V(\gamma_t) \right) \right)
\end{split}
\end{align}
which is \eqref{eq:WFR_SM} up to a time reparametrization.
\qed



\subsection{Direct optimization of the score matching loss} \label{subsec:direct_optim}

Let $L$ be defined in \autoref{prop:limitSM}. The first variation $\frac{\delta L}{\delta \mu}(\mu)(\omega)$ of $L$ at $\mu$ is
\begin{align} \label{eq:first_variation_L}
    \frac{\delta L}{\delta \gamma}(\gamma) (\theta) &=  \int_{\mathcal{X}}\left(2 \beta^2 \nabla_x \phi(x, \theta) \cdot \nabla_x \int_{\Theta}   \phi(x, \theta') \ d\gamma(\theta')  - 2 \beta \Delta_x \phi(x, \theta) \right) d\nu_n(x).
\end{align}
We optimize \eqref{eq:SM_objective_2} via the Wasserstein-Fisher-Rao (WFR) gradient flow \eqref{eq:WFR_SM}.
This measure PDE can be approximated via a particle system ODE (equation \eqref{eq:discretization_SM}), and the corresponding
particle system may be discretized into Algorithm \autoref{alg:SM_ebm_f1}.

\begin{lemma} \label{eq:discretization_SM}
Let $\{x_i\}_{i=1}^n$ be samples from a target distribution $\nu_p$. Let $\{\theta_{0}^{(j)}\}_{j=1}^{m}$ be features sampled uniformly over $\Theta$, let $\{\sigma_j\}_{j=1}^{m}$ be uniform samples over $\{\pm 1\}$ and let $\{w_{0}^{(j)} = 1\}_{j=1}^{m}$ be the initial weight values, which are set to 1.
Equation \eqref{eq:WFR_SM} can be simulated by evolving the features $\{\theta^{(j)}\}_{j=1}^{m}$ and the weights $\{w^{(j)}\}_{j=1}^{m}$ via the following ODE:
\begin{align}
    \frac{d\theta_t^{(j)}}{dt} &= - \sigma_j \nabla_{\theta} \left( \frac{1}{n} \sum_{i=1}^n \nabla_x \phi(x_i, \theta_t^{(j)})  \frac{1}{m} \sum_{j'=1}^{m} \sigma_{(j')} w_t^{(j')} \nabla_x \phi(x_i, \theta_t^{(j')}) - \frac{\beta^{-1}}{n} \sum_{i=1}^n \Delta_x \phi(x_i, \theta_t^{(j)}) \right), \\
    \frac{d\log w_t^{(j)}}{dt} &= -\left( \frac{ \sigma_j}{n} \sum_{i=1}^n \nabla_x \phi(x_i, \theta_t^{(j)})  \frac{1}{m} \sum_{j'=1}^{m} \sigma_{j'} w_t^{(j')} \nabla_x \phi(x_i, \theta_{j'}) - \frac{ \sigma_j \beta^{-1}}{n} \sum_{i=1}^n \Delta_x \phi(x_i, \theta_t^{(j)}) - K(t)\right),
\end{align}
where 
\begin{align}
    K(t) &= \mathds{1}_{ \|\gamma^+_t\|_{\text{TV}}+\|\gamma^-_t\|_{\text{TV}} \geq 1}  \\ &\times\frac{1}{m} \sum_{j=1}^m \sigma_j w_{j} \left( \frac{1}{n} \sum_{i=1}^n \nabla_x \phi(x_i, \theta_j) \cdot \frac{1}{m} \sum_{j'=1}^{m} \sigma_{j'} w_{j'} \nabla_x \phi(x_i, \theta_{j'}) - \frac{\beta^{-1}}{n} \sum_{i=1}^n \Delta_x \phi(x_i, \theta_j) \right).
\end{align}
Namely, up to a time reparametrization with factor $2 \beta^2$, the time-dependent measure $\hat{\gamma}_t = \frac{1}{m} \sum_{j=1}^m \sigma_j w_t^{(j)} \delta_{\theta_t^{(j)}}$ converges weakly to the solution $\gamma_t = \gamma_t^+ - \gamma_t^-$ of \eqref{eq:WFR_SM} with uniform initialization, for any finite time interval $[0,T]$, as $m \rightarrow \infty$.
\end{lemma}

\begin{proof}
We check that $\hat{\gamma}_t$ is a weak solution of \eqref{eq:WFR_SM} as in \autoref{prop:particle_dynamics_f1}, and use propagation of chaos.
\end{proof}

\begin{algorithm}
\caption{$\mathcal{F}_1$-EBM training via score matching}
\begin{algorithmic}
\STATE \textbf{Input:} $n$ samples $\{x_i\}_{i=1}^n$ of the target distribution, stepsize $s$.
\STATE Initialize features $(\theta_0^{(j)})_{j=1}^m$ unif. over $\Theta$, weights $(w_0^{(j)} = 1)_{j=1}^m$, signs $(\sigma_j)_{j=1}^m$ unif. over $\{\pm 1\}$.
\STATE Initialize generated samples $\{X_0^{(i)}\}_{i=1}^N$ uniformly i.i.d. from $\{x_i\}_{i=1}^n$.
\FOR {$t=0,\dots,T-1$}
    \FOR{$j=1,\dots,m$}
        \STATE
        Make the update $\theta_{t+1}^{(j)} = \theta_{t}^{(j)} - s \sigma_j \nabla_{\theta} \left( \frac{1}{n} \sum_{i=1}^n \nabla_x \phi(x_i, \theta_{t}^{(j)}) \cdot \frac{1}{m} \sum_{j'=1}^{m} \sigma_{j'} w_t^{(j')} \nabla_x \phi(x_i, \theta_{j'}) \right) + s \beta^{-1} \sigma_j \nabla_{\theta} \left( \frac{1}{n} \sum_{i=1}^n \Delta_x \phi(x_i, \theta_j) \right)$.
        \STATE Set $\tilde{w}_{t+1}^{(j)} = w_{t+1}^{(j)} \exp( -\frac{s \sigma_j}{n} \sum_{i=1}^n \nabla_x \phi(x_i, \theta_t^{(j)}) \cdot \frac{1}{m} \sum_{j'=1}^{m} \sigma_{j'} w_t^{(j')} \nabla_x \phi(x_i, \theta_t^{(j')}) + \frac{s \beta^{-1} \sigma_j}{n} \sum_{i=1}^n \Delta_x \phi(x_i, \theta_t^{(j)}))$
        \STATE Normalize if needed $w_{t+1}^{(j)} = \tilde{w}_{t+1}^{(j)}/\max (\frac{1}{m}\sum_{j'=1}^{m} \tilde{w}_{t+1}^{(j')}, 1 )$.
    \ENDFOR
\ENDFOR
\STATE Energy $E_T(x) := \frac{\beta}{m} \sum_{j=1}^{m} w_j \sigma_j \phi(x, \theta_j)$. 
\end{algorithmic}
\label{alg:SM_ebm_f1}
\end{algorithm}

\begin{proposition} \label{prop:algequivalence}
The Algorithm \autoref{alg:SM_ebm_f1} is equivalent to Algorithm \autoref{alg:implicit_ebm_f1} with (i) base probability measure proportional to Lebesgue, i.e. $\nabla \log \frac{d\tau_{\mathcal{X}}}{d\lambda} = 0$, (ii) replacement probability $p_r = 1$ and (iii) noisy updates.
\end{proposition}

\begin{proof}
For any iteration $t$ and particle $i$, let $k_{t+1,i}$ be a uniform independent integer random variable over $\{1, \dots, n\}$, i.e. $x_{k_{t+1,i}}$ is a uniform random sample from $\{x_{i'}\}_{i'=1}^n$.
We may rewrite the updates on $\{\theta_{t+1}^{(j)}\}$, $\{w_{t+1}^{(j)}\}$ for Algorithm \autoref{alg:implicit_ebm_f1} with $p_r = 1$, $\beta^{-1} = 0$ as 
\begin{align}
\begin{split} \label{eq:updates_SM_rewritten}
    X_{t+1}^{(i)} &= x_{k_{t+1,i}} -\frac{s}{m} \sum_{j=1}^{m} w_t^{(j)} \sigma_j \nabla_x \phi(x_{k_{t+1,i}}, \theta^{(j)}) + \sqrt{2 \beta^{-1} s}\, \zeta_{t}^{(i)}, \\
    \theta_{t+1}^{(j)} &= \theta_{t}^{(j)} + s \alpha \sigma_j w_t^{(j)} \left(\frac{1}{N} \sum_{i=1}^{N} \nabla_{\theta} \phi(X_{t+1}^{(i)}, \theta_{t}^{(j)}) 
    - \frac{1}{n} \sum_{i=1}^{n} \nabla_{\theta} \phi(x_i, \theta_{t}^{(j)}) \right), \\
    \tilde{w}_{t+1}^{(j)} &= w_{t+1}^{(j)} \exp \left( \frac{s\alpha}{N} \sum_{i=1}^{N} \phi(X_{t+1}^{(i)}, \theta_{t}^{(j)}) 
    - \frac{s\alpha}{n} \sum_{i=1}^{n} \phi(x_i, \theta_{t}^{(j)}) \right), \\ w_{t+1}^{(j)} &= \tilde{w}_{t+1}^{(j)}/\max \left(\frac{1}{m}\sum_{j'=1}^{m} \tilde{w}_{t+1}^{(j')}, 1 \right).
\end{split}
\end{align}
Notice that in the regime of small stepsize $s \ll 1$, we can use a second order Taylor approximation:
\begin{align}
\begin{split}
    \nabla_{\theta} \phi(X_{t+1}^{(i)}, \theta_{t}^{(j)}) &\approx \nabla_{\theta} \phi(x_{k_{t+1,i}}, \theta_{t}^{(j)}) + \langle \nabla_{x, \theta} \phi(x_{k_{t+1,i}}), X_{t+1}^{(i)} - x_{k_{t+1,i}}\rangle \\ &+\beta^{-1} s \langle \zeta_{t}^{(i)}, \nabla_{x,x,\theta} \phi(x_{k_{t+1,i}}, \theta_{t}^{(j)}) \zeta_{t}^{(i)} \rangle + o(s) \\ &= \nabla_{\theta} \phi(x_{k_{t+1,i}}, \theta_{t}^{(j)}) - 2s \nabla_{x, \theta} \phi(x_{k_{t+1,i}}, \theta_{t}^{(j)}) \cdot \frac{1}{m} \sum_{j'=1}^{m} w^{(j')} \sigma_{j'} \nabla_x \phi(x_{k_{t+1,i}}, \theta_{t}^{(j')}) \\ &+\sqrt{2\beta^{-1} s} \langle \nabla_{x, \theta} \phi(x_{k_{t+1,i}}), \zeta_{t}^{(i)} \rangle +\beta^{-1} s \langle \zeta_{t}^{(i)}, \nabla_{x,x,\theta} \phi(x_{k_{t+1,i}}, \theta_{t}^{(j)}) \zeta_{t}^{(i)} \rangle + o(s).
\end{split}
\end{align}
Notice that that $\mathbb{E}[\langle \zeta_{t}^{(i)}, \nabla_{x,x,\theta} \phi(x_{k_{t+1,i}}, \theta_{t}^{(j)}) \zeta_{t}^{(i)} \rangle | \theta_{t}^{(j)}] = \frac{1}{n} \sum_{i=1}^{n} \nabla_\theta \nabla_{x,x} \phi(x_{i}, \theta_{t}^{(j)})$. Moreover, 
\begin{align}
\begin{split}
    &\mathbb{E}\left[\frac{1}{N} \sum_{i=1}^{N} \nabla_{\theta} \phi(x_{k_{t+1,i}}, \theta_{t}^{(j)}) - 2s \nabla_{x, \theta} \phi(x_{k_{t+1,i}}, \theta_{t}^{(j)}) \cdot \frac{1}{m} \sum_{j'=1}^{m} w^{(j')} \sigma_{j'} \nabla_x \phi(x_{k_{t+1,i}}, \theta_{t}^{(j')}) \bigg| \theta_t \right] \\ &= \frac{1}{n} \sum_{i=1}^{n} \nabla_{\theta} \phi(x_{i}, \theta_{t}^{(j)}) - 2s \nabla_{x, \theta} \phi(x_{i}, \theta_{t}^{(j)}) \cdot \frac{1}{m} \sum_{j'=1}^{m} w^{(j')} \sigma_{j'} \nabla_x \phi(x_{i}, \theta_{t}^{(j')})
\end{split}
\end{align}
Making use of these observations and the expression of the update on $\theta_{t+1}^{(j)}$ in \eqref{eq:updates_SM_rewritten}, we get that
\begin{align}
\begin{split}
\mathbb{E}\left[\theta_{t+1}^{(j)} - \theta_{t}^{(j)} | \theta_t \right] &= - s^2 \alpha \sigma_j w_t^{(j)} \left(\frac{1}{n} \sum_{i=1}^{n} \nabla_{x, \theta} \phi(x_{k_{t+1,i}}, \theta_{t}^{(j)}) \cdot \frac{1}{m} \sum_{j=1}^{m} w_t^{(j')} \sigma_j \nabla_x \phi(x_{k_{t+1,i}}, \theta_{t}^{(j)}) \right) \\ &+ s^2 \alpha \sigma_j w_t^{(j)} \frac{\beta^{-1}}{n} \sum_{i=1}^{n} \nabla_\theta \nabla_{x,x} \phi(x_{i}, \theta_{t}^{(j)})
\end{split}
\end{align}
And this is equal to the update in Algorithm \autoref{alg:SM_ebm_f1} after renaming the stepsize $s^2 \alpha \to s$. The analogous argument holds for the update on $\log \tilde{w}_{t+1}^{(j)}$.
\end{proof}

\subsection{Comparison with Score-based Generative Models (SGMs)}

A recent series of works \cite{song2019generative,song2020improved,song2021scorebased,song2021train, kadkhodaie2020solving,jolicoeur2020adversarial,dhariwal2021diffusion} have leveraged the link between score matching and reversing a diffusion process (ie, \emph{denoising}) to propose flexible and powerful generative models (SGMs). While our work shows connections with score matching, our approach is somewhat far from SGMs.
Indeed, SGMs proceed by estimating various score functions of noisy versions of the data distribution, rather than the original data distribution, and later use these estimates for obtaining new samples using a Langevin diffusion.
In contrast, the score matching loss that we consider is directly given by the score matching metric through the classical trick introduced by \cite{hyvarinen05estimation}, and our Langevin sampling process is built into the training dynamics.
Also, while our work makes use of SDEs to evolve the generated samples, we do not use a forward-backward framework in the style of certain SGMs~\cite{song2021scorebased}.

\section{Proofs of \autoref{sec:implicit_EBM}}
\label{sec:proofs_implicit_f2}

\begin{lemma} \label{lem:link}
If \autoref{ass:phi} holds, the Fenchel dual of the problem $\min_{\nu \in \mathcal{P}(\mathcal{X})} \beta^{-1} D_{KL}(\nu||\tau_{\mathcal{X}}) + \text{MMD}_{k}(\nu, \nu_n)$ is the problem 
$\max_{f \in \mathcal{B}_{\mathcal{F}_2}(\beta)} -\frac{1}{n} \sum_{i=1}^n f(x_i) - \log \left( \int_{\mathcal{X}} e^{-f(x)} d\tau_{\mathcal{X}}(x) \right)$.
\end{lemma}
\begin{proof}
We apply \autoref{thm:duality_general} to show that the problem
\begin{align} \label{eq:EBM_dual_long}
    \min_{\nu \in \mathcal{P}(\mathcal{X})} \beta^{-1} D_{KL}(\nu||\tau_{\mathcal{X}}) + \left( \int_{\Theta} \left( \int_{\mathcal{X}} \phi(x, \theta) \ d(\nu - \nu_n)(x) \right)^2 d\tau_{\Theta}(\theta) \right)^{1/2}
\end{align}
has dual problem \eqref{eq:EBM_dual}. It remains only to show that the second term of \eqref{eq:EBM_dual_long} is equal to $\text{MMD}_{k}(\nu, \nu_n)$.
To obtain this, observe that 
\begin{align}
\begin{split}
&\int_{\Theta} \left( \int_{\mathcal{X}} \phi(x, \theta) \ d(\nu - \nu_n)(x) \right)^2 d\tau_{\Theta}(\theta) \\ &= \int_{\mathcal{X} \times \mathcal{X}} \int_{\Theta} \phi(x, \theta) \phi(x', \theta) d\tau_{\Theta}(\theta) \ d(\nu - \nu_n)(x) \ d(\nu - \nu_n)(x') \\ &= \int_{\mathcal{X} \times \mathcal{X}} k(x, x') \ d(\nu - \nu_n)(x) \ d(\nu - \nu_n)(x'),
\end{split}
\end{align} 
The first equality holds by Fubini's theorem following an argument similar to equations \eqref{eq:adjoint_exp}-\eqref{eq:fubini_1}. The second equality follows from the characterization \eqref{eq:kernel_random} of the kernel $k$.
\end{proof}

\begin{lemma} \label{lem:wasserstein_EBM1}
The Wasserstein gradient flow for the objective functional of \eqref{eq:EBM_primal} is given by \eqref{eq:wasserstein_EBM}.
\end{lemma}
\begin{proof}
The proof is standard. If we denote $L(\nu) = \beta^{-1} D_{KL}(\nu||\tau_{\mathcal{X}}) + \text{MMD}_{k}(\nu, \nu_n)$, the first variation of $L$ at any $\nu \in \mathcal{P}(\mathcal{X})$ is $\frac{\delta L}{\delta \nu}(\nu) : \mathcal{X} \to \R$ such that for all $\nu' \in \mathcal{P}(\mathcal{X})$, $\lim_{\epsilon \to 0} \ (L(\nu + \epsilon(\nu'-\nu)) - L(\nu))/\epsilon = \int_{\mathcal{X}} \ d(\nu'-\nu)(x)$. In this case, for any absolutely continuous $\nu \in \mathcal{P}(\mathcal{X})$,
\begin{align} \label{eq:first_variation_L2}
    \frac{\delta L}{\delta \nu}(\nu)(x) = \beta^{-1} \log \frac{d\nu}{d\lambda}(x) + \beta^{-1} - \beta^{-1} \log \frac{d\tau_{\mathcal{X}}}{d\lambda} (x) + \frac{\int_{\mathcal{X}} k(x,x') \ d(\nu_t - \nu_n)(x')}{\text{MMD}_{k}(\nu_t, \nu_n)}
\end{align}
and its gradient is
\begin{align} 
    \nabla \frac{\delta L}{\delta \nu}(\nu)(x) = \beta^{-1} \frac{\nabla \frac{d\nu}{d\lambda}(x)}{\frac{d\nu}{d\lambda}(x)} - \beta^{-1} \log \frac{d\tau_{\mathcal{X}}}{d\lambda} (x) + \frac{\int_{\mathcal{X}} k(x,x') \ d(\nu_t - \nu_n)(x')}{\text{MMD}_{k}(\nu_t, \nu_n)}
\end{align}
It is well known \citep{santambrogio2017euclidean} that the Wasserstein gradient flow of a functional $L$ is a solution of the measure PDE
\begin{align}
    \partial_t \nu_t = \nabla \cdot \left(\nu_t \nabla \frac{\delta L}{\delta \nu}(\nu_t)(x) \right).
\end{align}
\end{proof}

\begin{lemma}  \label{lem:stationary_minimizer}
If $\mathcal{X}$ is arc-connected, the unique stationary solution $\nu^{\star}$ of \eqref{eq:wasserstein_EBM} is the unique minimizer of \eqref{eq:EBM_primal}. The stationary solution must satisfy \eqref{eq:stationary_sol_mmd}.
\end{lemma}
\begin{proof}
We follow the same reasoning as \cite{rotskoff2018neural, mei2018mean}, skipping some techical details. Denoting $L(\nu) = \beta^{-1} D_{KL}(\nu||\tau_{\mathcal{X}}) + \text{MMD}_{k}(\nu, \nu_n)$, all stationary solutions $\nu^{\star}$ of the Wasserstein gradient flow of $L$ must satisfy
\begin{align} \label{eq:stationary_cond_mmd}
    \nabla \frac{\delta L}{\delta \nu}(\nu^{\star})(x) = 0, \quad \forall x \in \text{supp}(\nu^{\star}) 
\end{align}
Because of the KL term, $\text{supp}(\nu^{\star}) = \mathcal{X}$. Since $L$ is strictly convex because MMD is convex and $D_{KL}$ is strictly convex, $L$ has at most one minimizer, which is uniquely specified by the Euler-Lagrange condition
\begin{align} \label{eq:euler_cond_mmd}
    \frac{\delta L}{\delta \nu}(\nu^{\star})(x) = K, \quad \forall x \in \mathcal{X}, \ \text{for some } K.
\end{align}
When $\mathcal{X}$ is arc-connected, \eqref{eq:stationary_cond_mmd} implies \eqref{eq:euler_cond_mmd}.

To show that the solution must satisfy \eqref{eq:stationary_sol_mmd}, we just develop \eqref{eq:euler_cond_mmd} as in \eqref{eq:first_variation_L2} and isolate.
\end{proof}

\section{Additional experiments}
\label{sec:additional_experiments}

\textbf{Experiments on teacher-student models in $d=2$.}
We analyze the case of a teacher with two neurons, both with the same negative weight~$w^*_j = -10$, in $d=2$ (i.e. on the sphere) and we train the student setting $\beta = 20$ such that the approximation errors. This low-dimensional example allows for a visual representation of the training dynamics (see videos \url{KLdual_1e3_points.mp4} and \url{KLdual_1e4_points.mp4} in the supplementary material). \autoref{fig:stud:teach} shows that the densities of the Gibbs distribution associated to the teacher and the student at the end of training are very concentrated in two separated regions on the sphere. This means that sampling this distribution by Langevin dynamics, which is required in the late stages of training in the primal formulation, would be challenging due to strong metastability. Our aim is to illustrate that our dual formulation avoids this metastability issue in the sampling.

For different values of $p_R$, and for $n=10^3, 10^4$ training data points, \autoref{fig:klsm3dadditional} shows the evolution of the KL-divergence and the score matching between the teacher and student models, and the TV-norm of the student measure, i.e. the $\mathcal{F}_1$ norm of the student energy. We use $N= 2 \cdot 10^3, 2 \cdot 10^4$ particles (resp.), $m=64$ student neurons, and a testing set of $n^{*} = 10^4$ to compute the KL-divergence and score matching metric.
In this setting we observe that $p_R = 0$ and $p_R = 1/60$ perform similarly, while score matching ($p_R = 1$) has a slower convergence and has larger terminal values for both the KL divergence and the score matching metric.
As expected, the test metrics improve with more training data $n$, and we observe that the relative gap between the methods becomes smaller; score matching becomes more competitive.

\begin{figure*}[h]
    \begin{minipage}[c]{8cm}
    \includegraphics[width=0.75\textwidth]{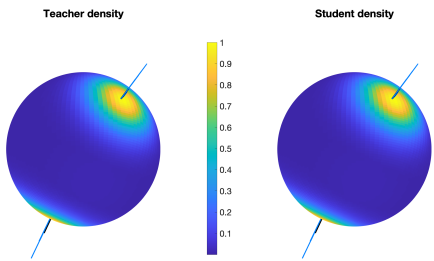}
    \end{minipage}
    \begin{minipage}[c]{6.5cm}
    \caption{Comparison between the teacher and student density in $d=2$, after training ($n = 10^4, \ N = 2 \cdot 10^4, m = 62, \alpha = 10, p_R = 0$). The location of the 2 teachers neurons are shown in black, and that of the 64 students neurons in blue.}
    \label{fig:stud:teach}
    \end{minipage}
\end{figure*}

\begin{figure*}[h]
    \centering
    \includegraphics[width=.40\textwidth]{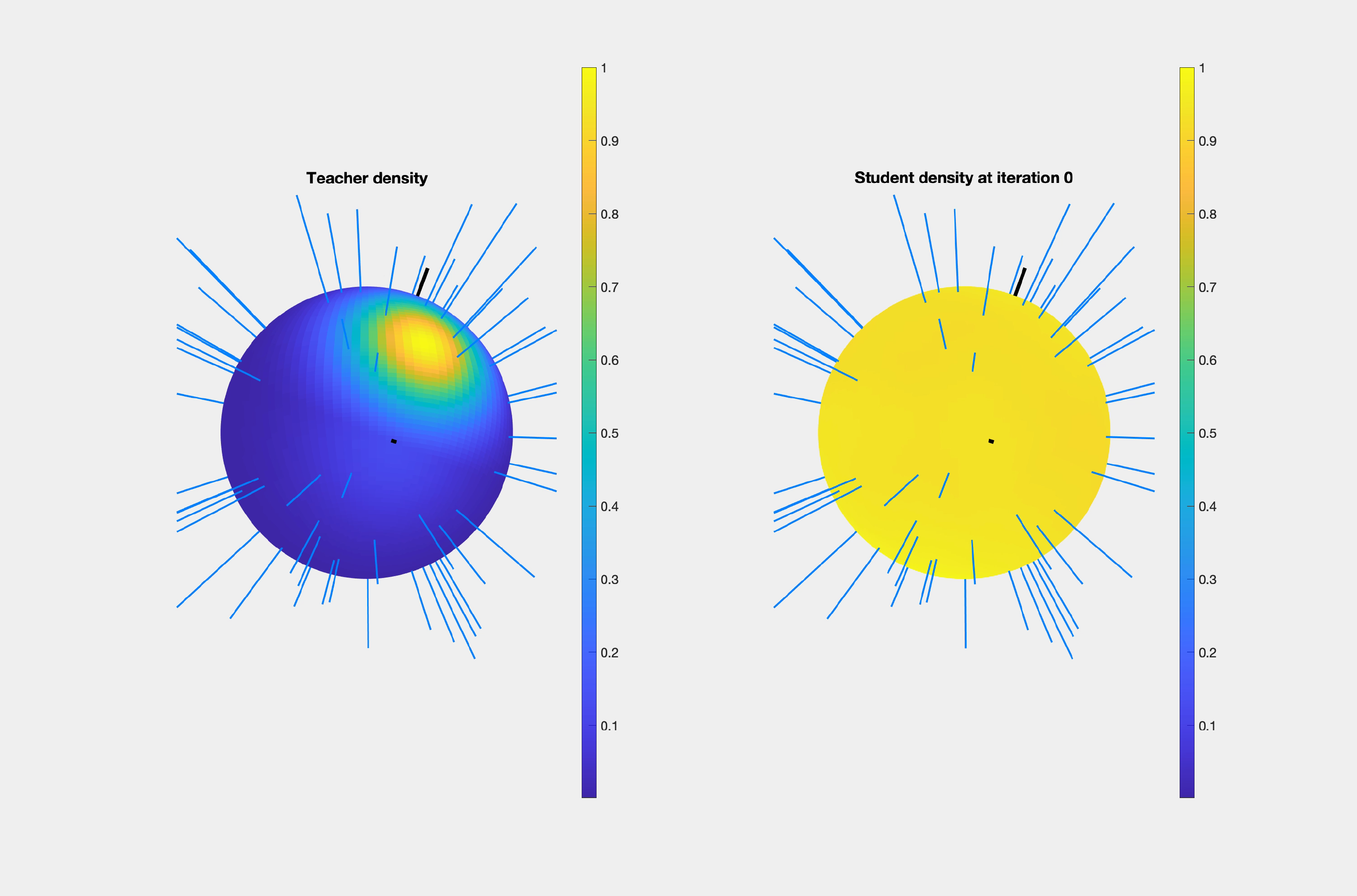}
    \includegraphics[width=.40\textwidth]{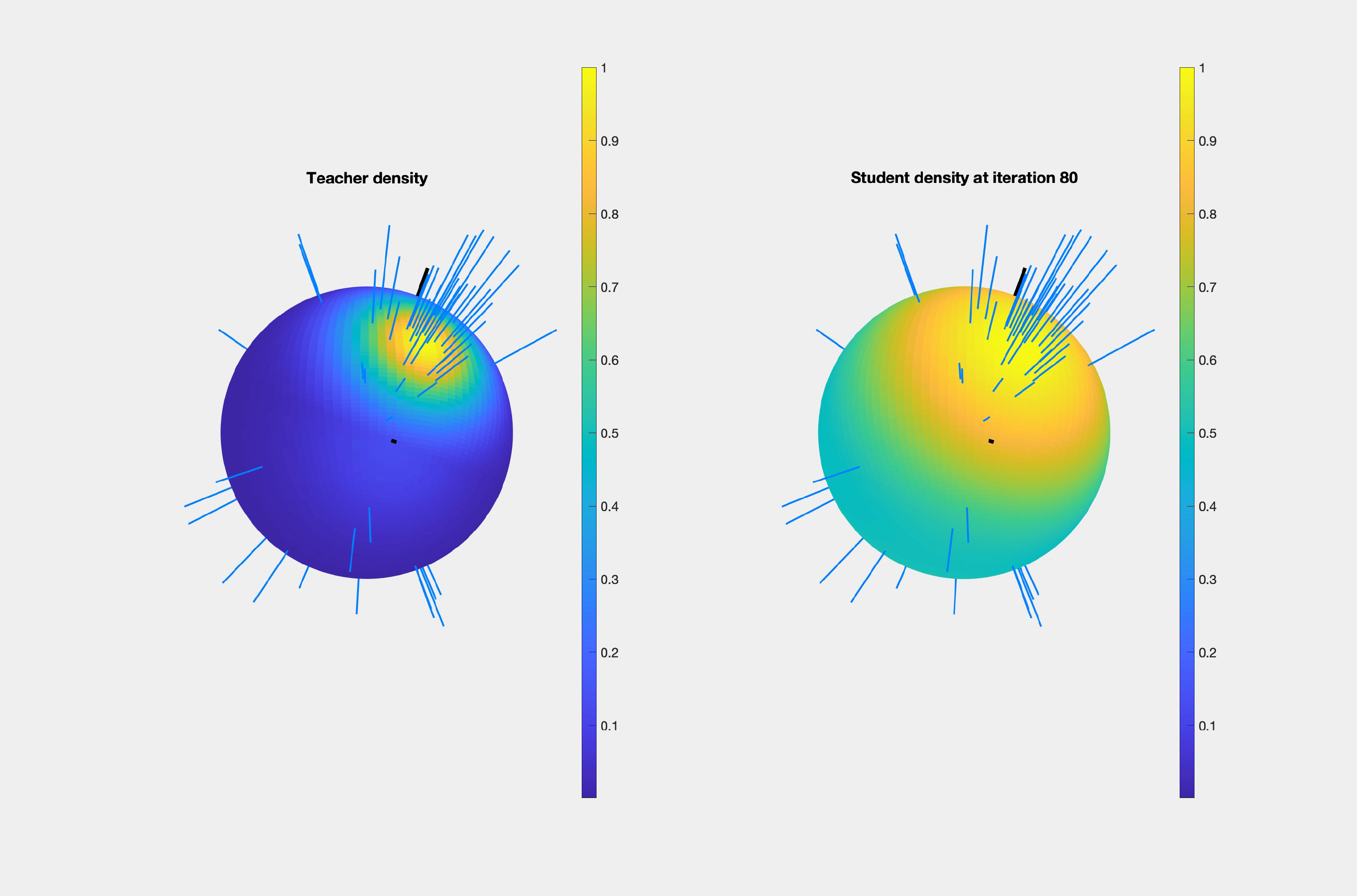} \\
    \includegraphics[width=.40\textwidth]{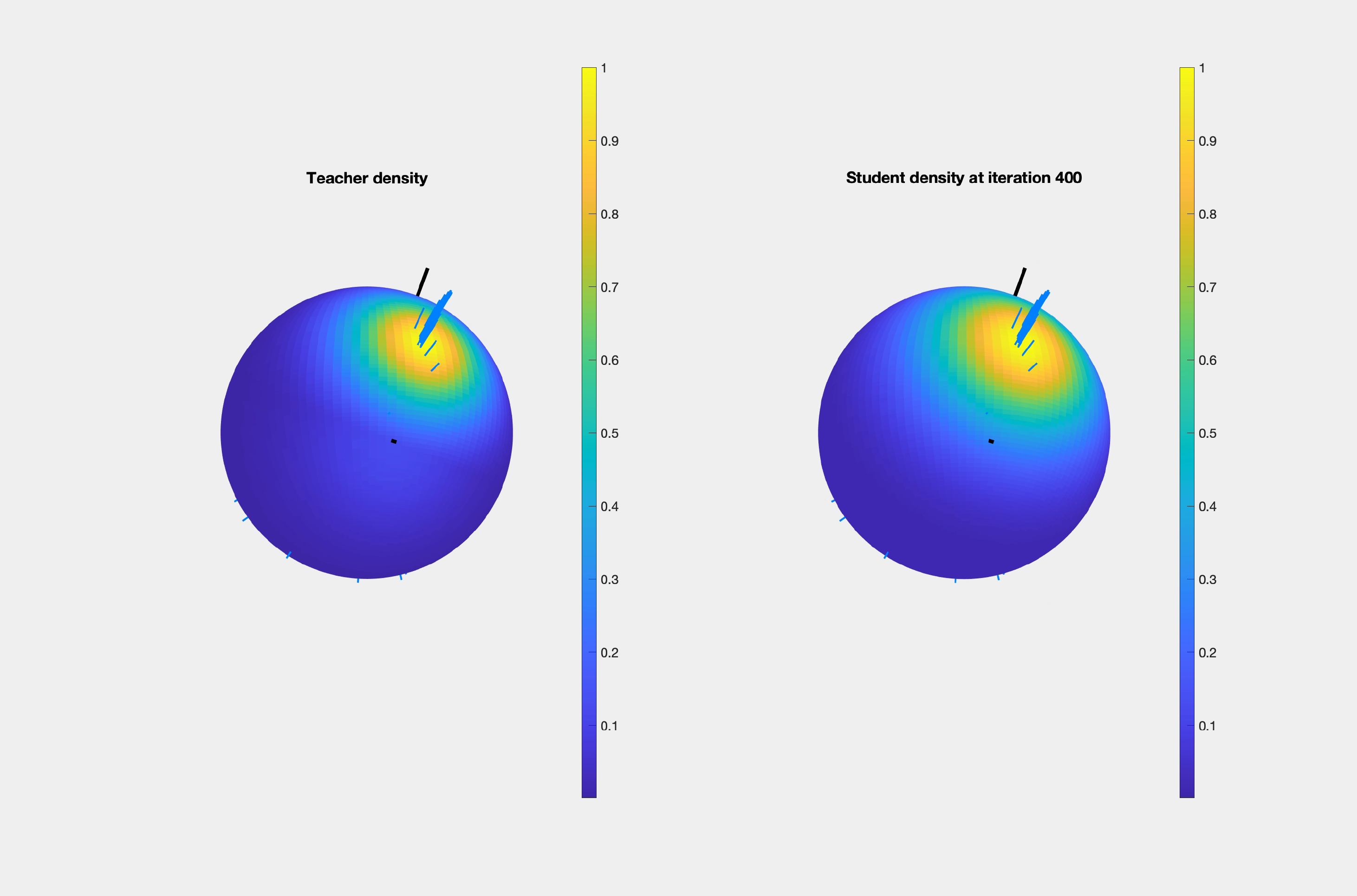}
    \includegraphics[width=.40\textwidth]{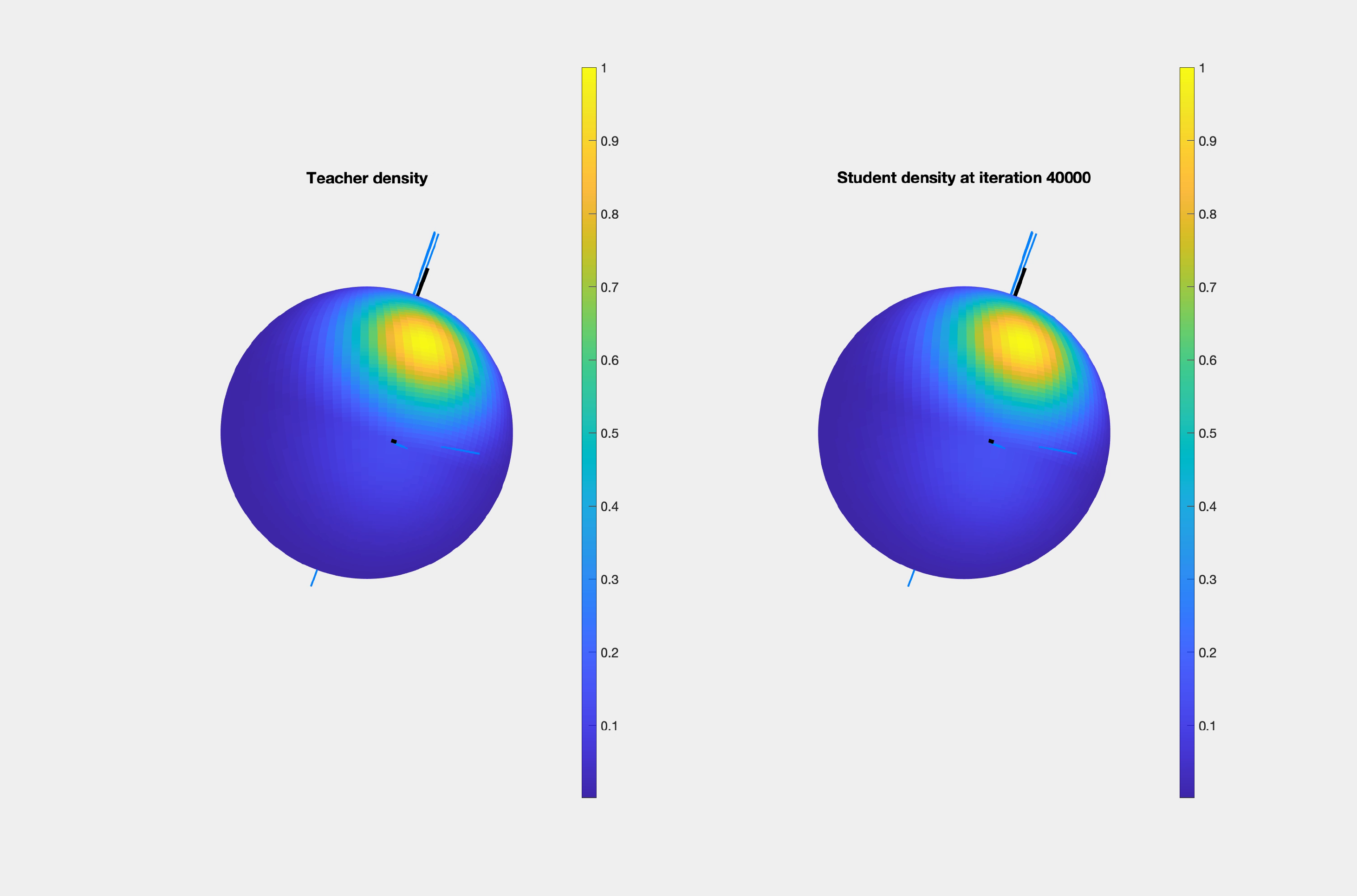}
    \caption{\textit{Experiments in $d=2$:} Selected frames of the video \url{KLdual_1e4_datapoints_monomodal.mp4} at iterations 0 (top left), 80 (top right), 400 (bottom left) and 40000 (bottom right). The parameters are $d=2, \ m=64, \ p_R = 0, \ n = 10^4, \ N = 2 \cdot 10^4, \ w_1^* = w_2^* = -10$. The teacher neurons, shown as black sticks, are almost perpendicular, and hence the teacher distribution is monomodal. The 64 student neurons are shown in blue. The two stages of training mentioned in text are clearly visible.
    }
    \label{fig:iterations_monomodal}
\end{figure*}
\begin{figure*}[h]
    \centering
    \includegraphics[width=.85\textwidth]{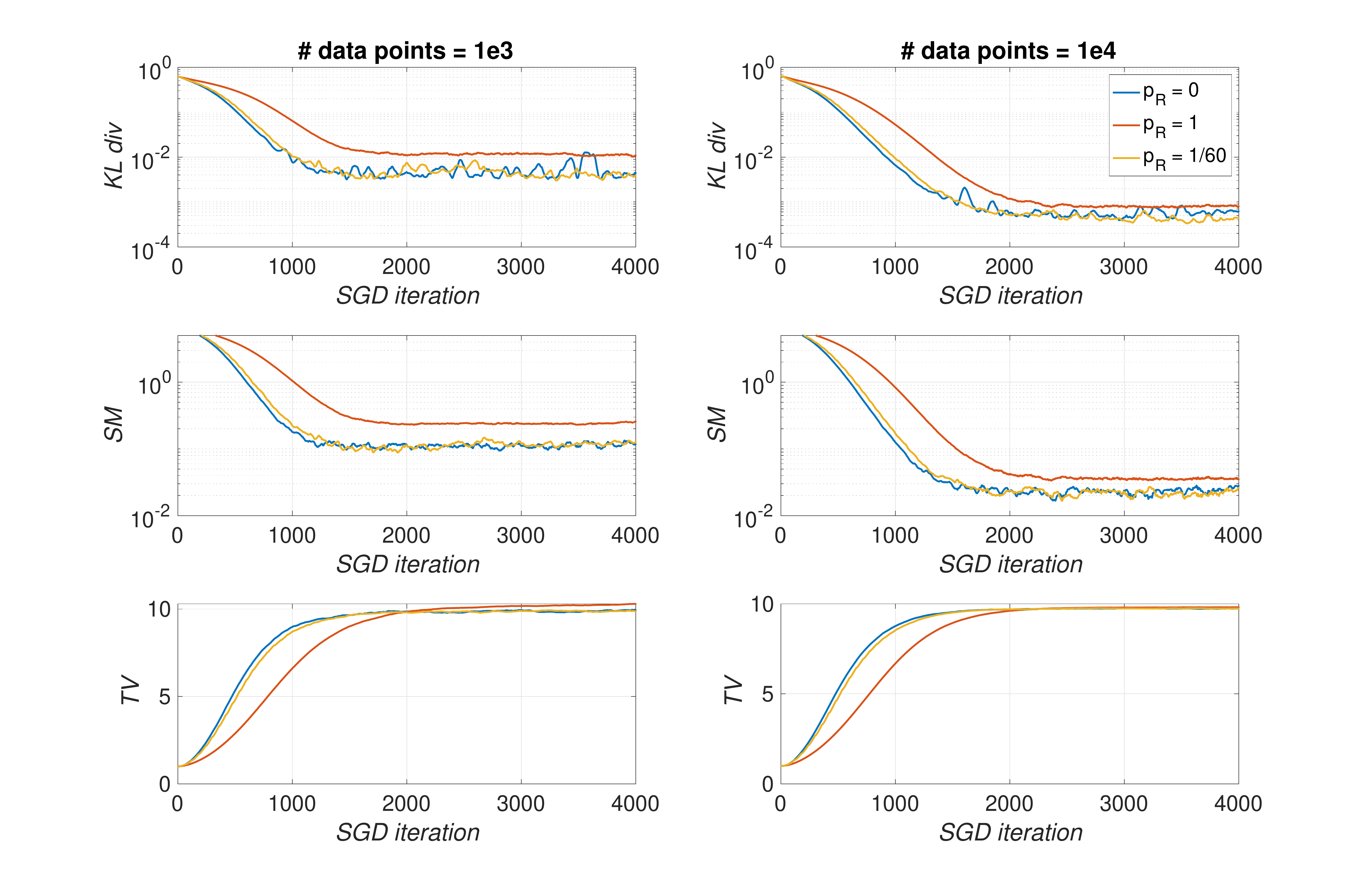}
    \caption{\textit{Experiments in d=2:} The evolution of the KL divergence, the score matching metric and the TV norm of the trained measure (i.e., the $\mathcal{F}_1$ norm) during training for Algorithm~\ref{alg:implicit_ebm_f1} with $\mathcal{X}=\mathbb{S}^2$, $m = 64$, $p_R \in \{0, 1, 1/60\}$, $s = 0.02$, $\alpha = 2 + 10 p_R$, and (left) $n = 10^3$, $N = 2 \cdot 10^3$, (right) $n = 10^4$, $N = 2 \cdot 10^4$. In comparison, the non-parametric kernel density estimator reaches a KL error of $2 \cdot 10^{-2}$ for $n=10^3$ and $6 \cdot 10^{-3}$ for $n=10^4$. 
    } 
    \label{fig:klsm3dadditional}
\end{figure*}
\textbf{Bimodality vs. monomodality in $d=14$.} \autoref{fig:projections_14d} shows the histograms for the cosines of the angles between the samples and each teacher neuron. We see that when the two teacher neurons are at an angle of 2.87 rad (almost opposite), the distribution is bimodal. When they are at an angle of 1.37 rad, the distribution is monomodal. 
\begin{figure*}
    \begin{minipage}[c]{9cm}
    \centering
    \includegraphics[width=.9\textwidth]{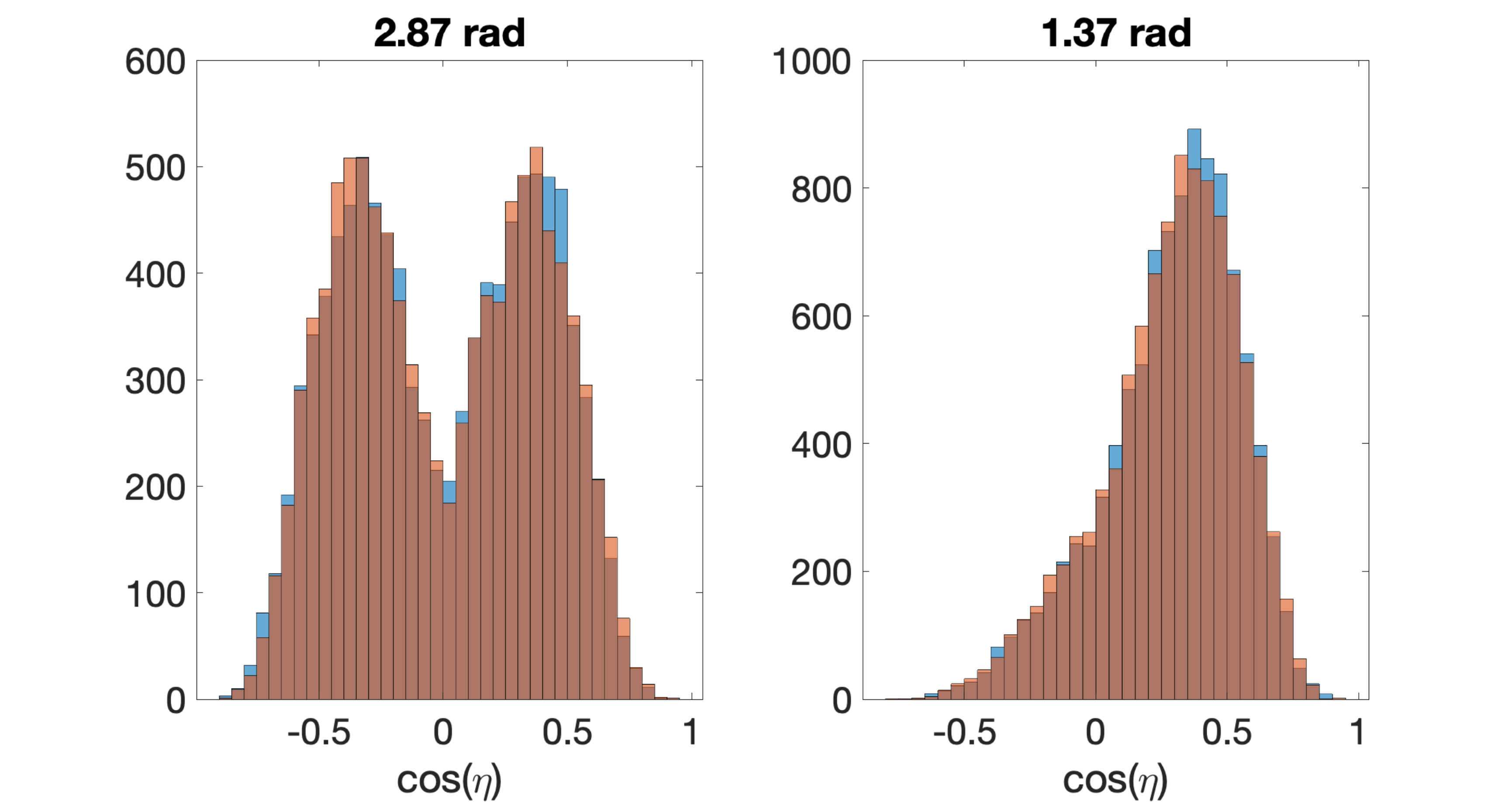}
    \end{minipage}
    \begin{minipage}[c]{6cm}
    \caption{\textit{Experiments in d=14:} Histograms for the cosines of the angles between each teacher neuron and samples from the target distribution, when the angle between teacher neurons is 2.87 and 1.37 rad.}
    \label{fig:projections_14d}
    \end{minipage}
\end{figure*}

\textbf{Comparing different values of $p_R$ for $d=14$}. In \autoref{fig:klsmd15newteach1000} (top, middle), which correspond to the bimodal case with angle 2.87 rad, we observe that the three variants have similar performance but $p_R = 1/40$ achieves the best metrics, followed very closely by $p_R = 0$ and $p_R = 1$ (score matching) a bit behind. Remark that early stopping might be beneficial in terms of the test error; the best test metrics are achieved roughly at the iteration at which the $\mathcal{F}_1$ norm of the trained energy reaches the $\mathcal{F}_1$ norm of the teacher energy. Interestingly, in the monomodal case with angle 1.37 rad (bottom of \autoref{fig:klsmd15newteach1000}), the best value for the KL divergence is achieved by $p_R = 1$ with early stopping, which beats the other two alternatives by a narrow margin. Unlike in the bimodal case, in the monomodal setting the training curves for the three methods display a change of behavior (a bump) slightly after initialization, and before the metrics reach values close to the final ones. This observation seems at odds with the common intuition that monomodal distributions are ``easier'' to deal with. To assess that the planted model defines a challenging high-dimensional density estimation problem, we consider a kernel density estimator baseline using an RBF kernel projected in the unit sphere. We report the KL divergences obtained in \ref{fig:klsmd15newteach1000}, and they are much higher than the EBM ones.

\begin{figure*}
    \centering
    \includegraphics[width=.65\textwidth]{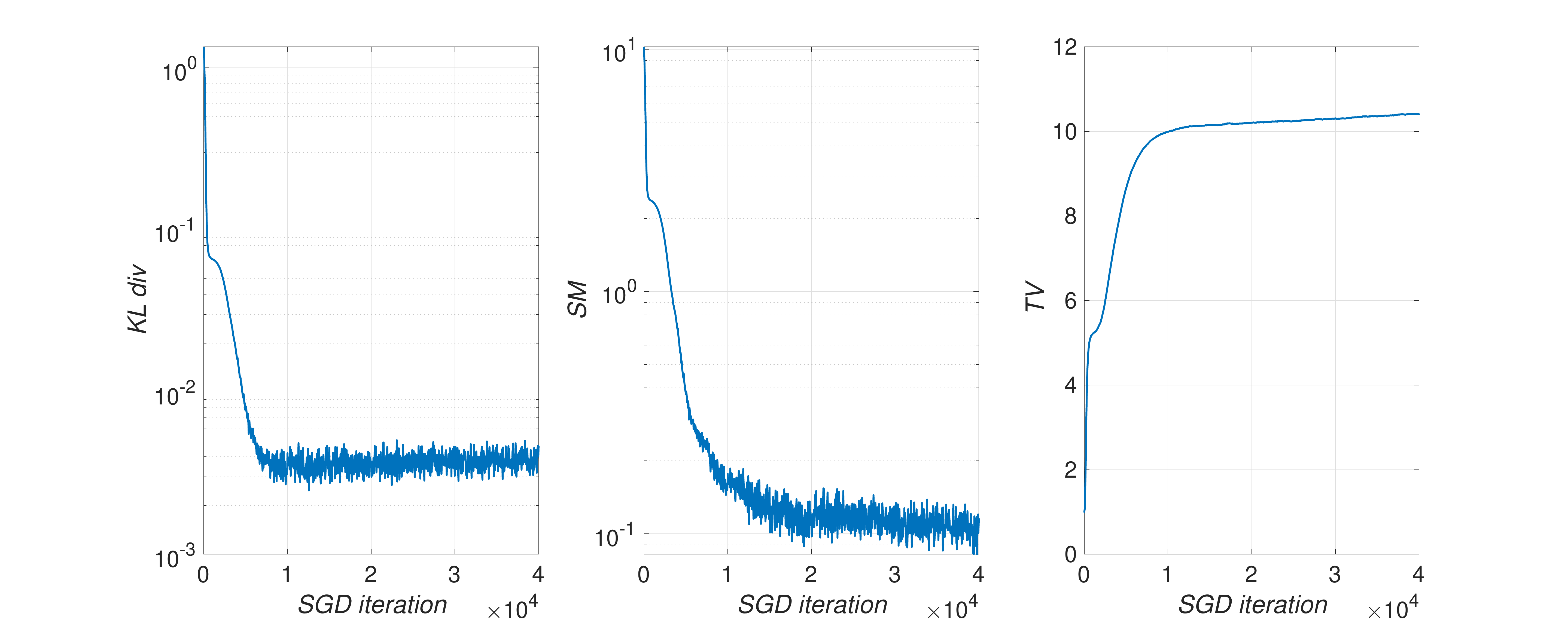}
    \caption{\textit{Experiments in $d=2$:} Evolution of the KL divergence, score matching and TV norm for the training dynamics of \url{KLdual_1e4_points_monomodal.mp4} and \autoref{fig:iterations_monomodal}.
    }
    \label{fig:klsmd3_monomodal}
\end{figure*}

\begin{figure*}
    \centering
    \includegraphics[width=.65\textwidth]{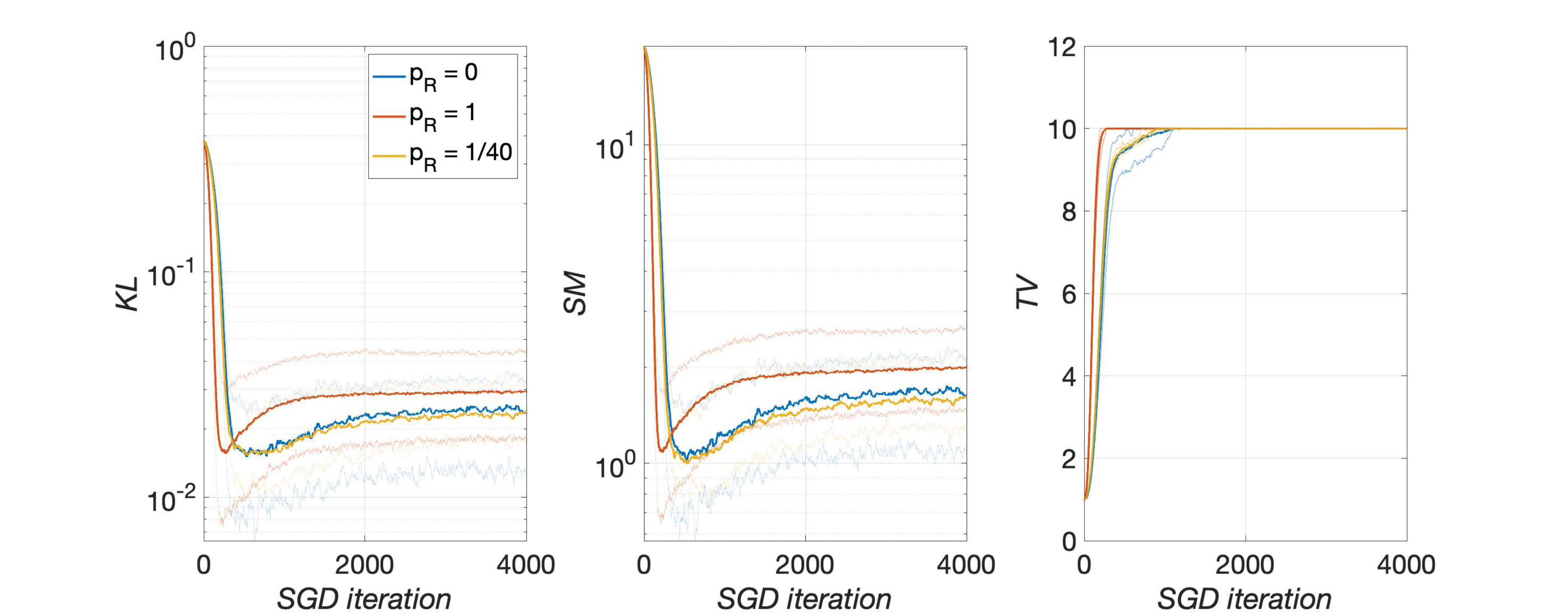}
    \includegraphics[width=.65\textwidth]{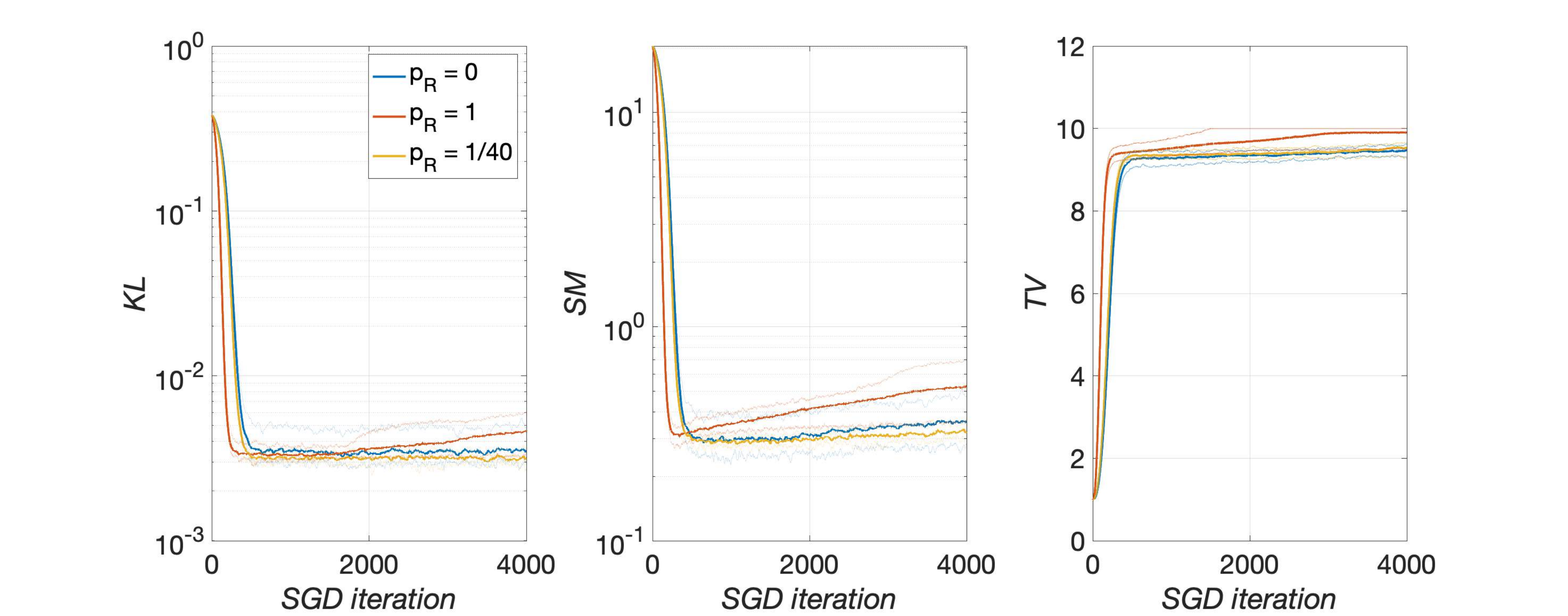} \\
    \includegraphics[width=.65\textwidth]{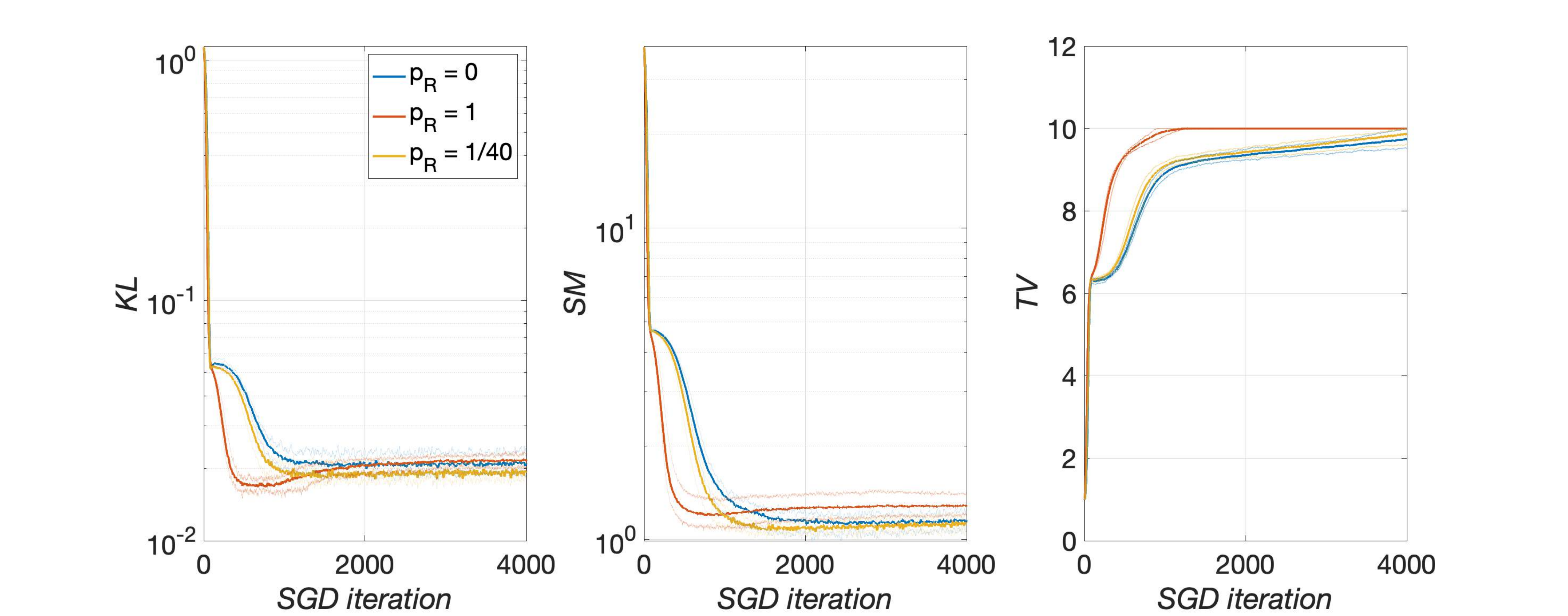}
    \caption{\textit{Experiments in d=14:} (Top) The evolution of the KL divergence, the score matching metric and the TV norm of the trained measure (i.e., the $\mathcal{F}_1$ norm) during training for Algorithm~\ref{alg:implicit_ebm_f1} with $\mathcal{X}=\mathbb{S}^{14}$, $m = 64$, $p_R=0, 1, 1/40$, $s = 0.02$, $\alpha = 10 + 50 p_R$, $n = 10^3$, $N = 2 \cdot 10^3$. The plots show the average, maxima and minima over six runs with different training and test samples, initializations and noise realizations, but with the same teacher network with an angle of 2.87 rad between neurons.
    In comparison, the non-parametric kernel density estimator reaches a KL divergence of $0.18$. (Middle) Same experiments with $n = 10^4$ and $N = 2 \cdot 10^4$. The non-parametric kernel density estimator reaches a KL divergence of $0.11$. (Bottom) Same experiments with $n = 10^4$ and $N = 2 \cdot 10^4$, and angle of 1.37 rad between teacher neurons. In comparison, the non-parametric kernel density estimator reaches a KL divergence of $0.15$.
    }
    \label{fig:klsmd15newteach1000}
\end{figure*}
\begin{figure*}
    \centering
    \includegraphics[width=.7\textwidth]{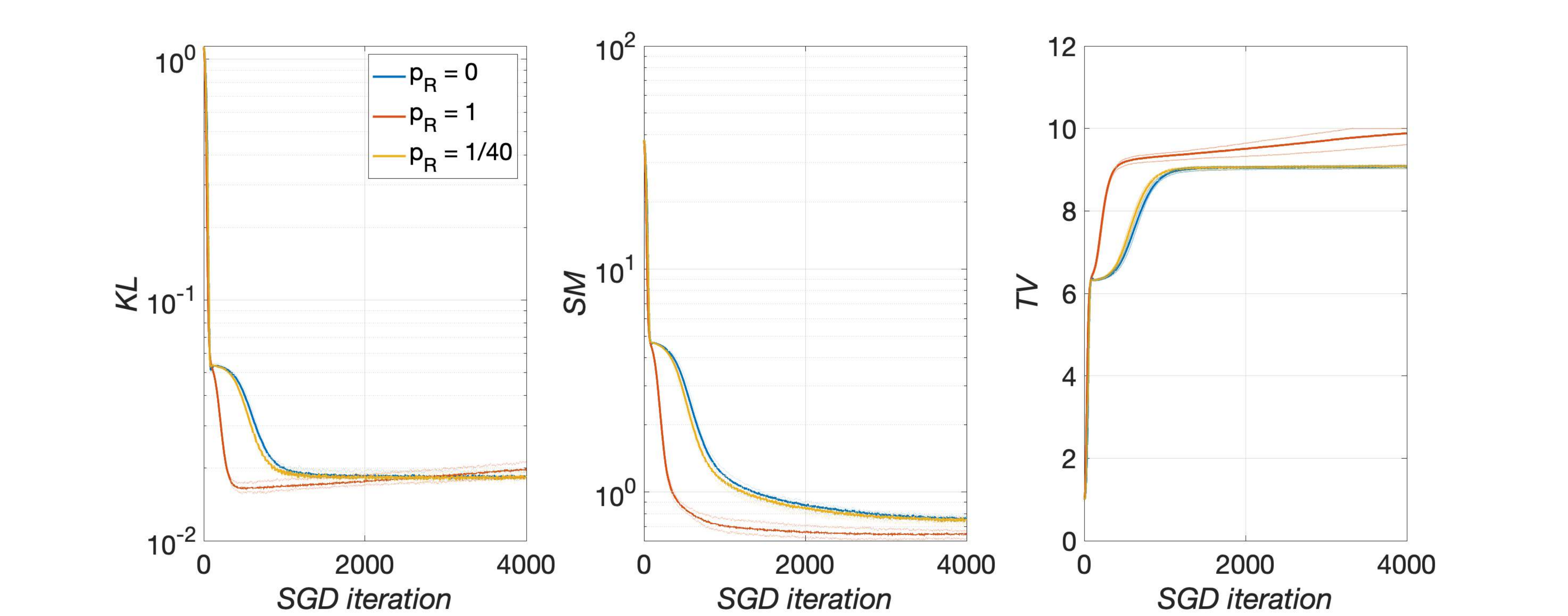}
    \caption{\textit{Experiments in $d=14$:} Same setting as bottom row of \autoref{fig:klsmd15newteach1000} (i.e., angle 1.37 rad), but with $n = 10^5, N = 2 \cdot 10^5$.
    }
    \label{fig:klsmd15oldteach100000}
\end{figure*}
In the bottom row of \autoref{fig:klsmd15newteach1000}, we observe that when the teacher distribution is monomodal, which happens when the teacher neurons are close to perpendicular, the training curves present a ``bumpy'' shape unlike in the bimodal case. \autoref{fig:klsmd15oldteach100000} shows plots in the same setting, but with 10 times more training data points. As already observed in \autoref{fig:klsm3dadditional} and \autoref{fig:klsmd15newteach1000}, taking larger $n$ improves the relative performance of score matching ($p_R = 1$) against the other two choices. It is also remarkable that the value of the KL divergence at the end of training in \autoref{fig:klsmd15oldteach100000} is about $2 \cdot 10^{-2}$, which is very similar to the value obtained in the bottom row of \autoref{fig:klsmd15newteach1000} despite the increase in $n$. This is at odds with the statistical analysis of \cite{domingoenrich2021energybased}, which predicts a decrease of the KL test error as $O(1/\sqrt{n})$ in the case where the approximation error is null. Hence, even though the KL values achieved are low, there is some effect at play which hinders optimization in the monomodal case.

To further understand the ``bumpy'' curves observed in the monomodal case, we return to experiments in $d=2$, this time with almost perpendicular teacher neurons. The results are shown in \autoref{fig:iterations_monomodal}. 
We observe similar trends in the curves of \autoref{fig:klsmd3_monomodal}. In \autoref{fig:iterations_monomodal}, we see that  the training occurs in  two stages:  first the student neurons first concentrate rather quickly near the mode of the teacher distribution: second, they slowly converge toward the teacher neurons. The bump in the KL and SM curves occurs when the first training stage ends and the second one sets in.

These findings seem to suggest an interesting dichotomy: when the two teacher neurons are far away and the distribution is bimodal, sampling is hard but training is easier; when the teacher neurons are closer and the distribution is monomodal, the opposite is true. In a generic situation, both issues may be present. More experiments are required to formulate concrete statements.

\end{document}